\documentclass{article}
\usepackage{natbib}

\newif\ifdraft
\drafttrue

\usepackage[utf8]{inputenc}
\usepackage[T1]{fontenc}  
\usepackage{hyperref}
\usepackage{url}
\usepackage{booktabs}       
\usepackage{amsfonts}       
\usepackage{nicefrac}       
\usepackage{microtype}      
\usepackage{xcolor}         

\usepackage[short]{optidef}
 
\usepackage{comment}
\usepackage[english]{babel}
\usepackage[ruled]{algorithm2e} 

\usepackage{fullpage}
\usepackage{graphicx}
\usepackage{amsmath,amssymb,dsfont,amssymb}
\usepackage{amsthm}
\usepackage{bm}
\usepackage{mathrsfs}
\usepackage{appendix}
\usepackage{graphicx}
\usepackage{pifont}
\usepackage{xcolor}
\usepackage{xcolor}
\usepackage{algpseudocode}
\usepackage{bbm}
\usepackage{authblk}
\usepackage{thm-restate}
\usepackage{listings}
\DeclareMathOperator{\sign}{sign}

\usepackage{caption}
\usepackage{subcaption}

\newtheorem{thm}{Theorem}
\newtheorem{lemma}{Lemma}
\newtheorem{definition}{Definition}

\newtheorem{rmk}{Remark}
\newtheorem{fact}{Fact}
\newtheorem{clm}{Claim}

\DeclareMathOperator*{\argmax}{argmax}
\DeclareMathOperator*{\argmin}{argmin}
\newcommand{\Hs}{\mathcal{H}}
\newcommand{\X}{\mathcal{X}}
\newcommand{\Y}{\mathcal{Y}}
\newcommand{\Z}{\mathcal{Z}}
\newcommand{\R}{\mathbb{R}}
\newcommand{\E}{\mathbb{E}}
\newcommand{\F}{\mathcal{F}}

\newcommand{\1}{\mathds{1}}
\newcommand{\Ps}{\mathcal{P}}
\newcommand{\Qs}{\mathcal{Q}}
\newcommand{\Gs}{\mathcal{G}}

\newcommand{\emily}[1]{}
\newcommand{\ar}[1]{}
\newcommand{\mk}[1]{}
\newcommand{\wes}[1]{}

\title{Multiaccurate Proxies for Downstream Fairness}

\author[1,2]{Emily Diana}
\author[1,2]{Wesley Gill}
\author[1,2]{Michael Kearns}
\author[3]{\authorcr Krishnaram Kenthapadi}
\author[1,2]{Aaron Roth}
\author[1]{Saeed Sharifi-Malvajerdi}
\affil[1]{University of Pennsylvania}
\affil[2]{Amazon AWS AI}
\affil[3]{Fiddler AI}

\begin{document}

\maketitle

\begin{abstract}
    We study the problem of training a model that must obey demographic fairness 
    conditions when the sensitive features are not  available at training time --- in other words, 
    \textit{how can we train a model to be fair by race when we don't have data about race?} 
    We adopt a \textit{fairness pipeline} perspective, in which an ``upstream'' 
    learner that does have access to the sensitive features will learn a \textit{proxy model} for these features from the other attributes.
    The goal of the proxy is to allow a general ``downstream'' learner --- with minimal assumptions on their prediction task --- to be able to use the
    proxy to train a model that is fair with respect to the true sensitive features. We show that obeying \textit{multiaccuracy} constraints with respect to the
    downstream model class suffices for this purpose, provide sample- and oracle efficient-algorithms and generalization bounds for learning
    such proxies, and conduct an experimental evaluation. In general, multiaccuracy is much easier to satisfy than classification accuracy, and can be satisfied even when the sensitive features are hard to predict.
\end{abstract}

\section{Introduction}

There are various settings in which there is a desire to train a model that is fair with respect to some sensitive
features (e.g. race and gender), but in which the values for these features are unavailable in the training data.
This might be for legal reasons (e.g. in the United States it is against the law to use race as an input to 
consumer lending models), or for policy reasons (e.g. many large consumer-facing organizations choose
not to ask their customers for such information). This leads to an apparent technical conundrum: \textit{How can we be fair by race if we don't have data about race?}

Standard practice when attempting to enforce statistical fairness constraints in the absence of sensitive data is to attempt to predict individual sensitive features $z$---like race---using  \textit{proxies} $\hat z$ that predict $z$ from other available features $x$. For example, the popular ``Bayesian Improved Surname Geocoding'' method \cite{surname} attempts to predict race from an individual's location and surname. But what properties should $\hat z$ have, and how can we achieve them algorithmically? The answer is not obvious. Suppose for a moment that $z \in \{0,1\}$ is binary valued. It has been observed in prior work that training a binary valued $\hat z$ to minimize classification error can yield a proxy that results in substantial bias even when used to solve the easier \textit{auditing} problem (i.e. bias detection only, not bias mitigation) on a downstream predictive model $h$ \cite{Kallus1,awasthi2}. 

\subsection{Our Model and Results}
We envision a pipeline model in which an upstream Proxy Learner (PL) has access to a data set with sensitive features, but without knowledge of what learning problems a variety of Downstream Learners (DLs) might want to solve.  We consider two cases: either the PL has access to \textit{samples} of labels from the distribution over problems that the DLs are interested in, or the PL hypothesizes that the labels the DLs wish to learn can be predicted using models from some binary-valued function class. The DLs do not have access to the sensitive features $z$. The goal of the PL is to train a \textit{proxy model} $\hat z$ that tries to predict the conditional expectation of  $z$, conditional on the other observable features $x$. A good proxy will have the property that for most DLs, training subject to demographic fairness constraints imposed via the proxy $\hat z$ will result in the same model that would have been obtained from imposing constraints directly on $z$. In the body of the paper we focus on equal error rate constraints, but in the appendix we show how our techniques extend to other standard measures of demographic fairness, including statistical parity and  false positive/negative rate equality. 

We make a connection between the PL problem and \textit{multiaccuracy} as defined by \cite{multiaccuracy}, which informally asks that a model $\hat z$ be statistically unbiased on a large collection of sets $G$. We note that if our proxy $\hat z$ is appropriately multiaccurate over groups defined by a function class $\Hs$, it serves as a good proxy for downstream fair learning problems over the hypothesis class $\Hs$. For statistical parity, we simply need multiaccuracy over the collection of sets $G = \Hs$ --- i.e. the collection of sets corresponding to points that each $h \in \Hs$ labels as positive. Past work gives algorithms for learning multiaccurate functions $\hat z$ with respect to any collection of groups such that membership in those groups can be determined at test time --- which it can, in the case of groups defined by a known function class $\Hs$ \cite{multiaccuracy}. Because statistical parity is a fairness constraint that is defined independently of the labels, a proxy $\hat z$ that is multiaccurate with respect to $\Hs$ can be used to solve \textit{any} downstream learning problem over $\Hs$ subject to a statistical parity constraint. For an equal error fairness constraint, we require that $\hat z$ satisfy multiaccuracy over a collection of sets $G$ corresponding to \textit{error regions} of each $h \in \Hs$. Whether an individual falls into such an error region is not observable by a deployed classifier (since this depends on the unobserved label $y$), so we must develop new algorithms for this case. These error regions depend on the labelling function (i.e. the learning problem), so to serve as a good proxy for many different downstream learning problems, $\hat z$ must be multiaccurate with respect to the error regions defined using many different labelling functions. We provide two ways to do this:

\begin{enumerate}
    \item If the labelling functions come from a \textit{distribution over problems}, then we show how to take a polynomially sized sample from this distribution and use it to train a proxy $\hat z$ that is a good proxy for \textit{most} learning problems in the distribution (i.e. we give a PAC-style bound with respect to the problem distribution). 
    \item If the labelling functions come from a bounded VC-class we can optimize over, then we show how to train a proxy $\hat z$ that is a good proxy for \textit{every} labelling function generated from that class.
\end{enumerate}

All of our algorithms are \textit{oracle efficient}, meaning they are efficient reductions to standard empirical risk minimization problems. Finally, we perform an empirical evaluation to demonstrate the utility of our proxy training algorithms.

\subsection{Related Work}
Using proxies for race or ethnicity is standard practice in finance and other settings in which sensitive features are often unavailable but fairness is a concern. Common features used for prediction include surname, first name, and geographic location \cite{surname,firstname,lending}.

Several papers, beginning with Chen et al. \cite{Kallus1}, have considered the problem of \textit{evaluating} measures of statistical fairness on a fixed classifier using a proxy for the specified sensitive attribute. Chen et al. \cite{Kallus1} characterize the bias that is introduced in estimating the degree to which a fixed classifier violates the statistical parity (also known as demographic disparity) condition, when a proxy representing a thresholding of the conditional probability of a binary  sensitive attribute is used. They also show that when the proxy is computed using the same features as the downstream classification, the true conditional expectation of the protected attribute (conditional on the non-sensitive features) can be used to give an unbiased estimate of the demographic disparity. Awasthi et al. \cite{awasthi2} embark on a similar study for evaluating disparities in false positive or negative rates, and characterize the \textit{distortion factor} of a proxy as a function of properties of the underlying distribution and propose estimating this distortion factor and then trying to correct for it. 

Several papers also aim at postprocessing or training fair models without sensitive features. Awasthi, Kleindessner, and Morgenstern \cite{awasthi1} consider perturbations of sensitive features (e.g. as they might be if labelled using crowdsourced workers) and give conditions (such as conditional independence of the noisy sensitive features and the non-sensitive features) under which post-processing a fixed classifier to equalize false positive or negative rates as measured under the proxy reduces the true disparity between false positive or negative rates subject to the true sensitive features. In similar noise models, Wang et al. \cite{wang2020robust} propose robust optimization based approaches to fairness constrained training with noisy sensitive features and Mehrotra and Celis consider the problem of fair subset selection \cite{subset}. Lahoti et al. \cite{lahoti} propose to solve a minimax optimization problem over an enormous set of ``computationally identifiable'' subgroups, under the premise that if there exists a good proxy for a sensitive feature, it will be included as one of these computationally identifiable groups defined with respect to the other features. This is related to subgroup fairness studied by Kearns et al. \cite{gerrymandering} and Hebert-Johnson et al. \cite{multiaccuracy} --- but this approach generally leads to a degradation in accuracy. A related line of work considers cryptographic solutions in a setting in which the relevant sensitive features for individuals are available---held by a third party in \cite{crypto1} or by the individuals themselves in \cite{crypto2}---but can only be accessed via cryptographic means like secure multiparty computation. Similarly, \cite{privatefair} studies the case in which the sensitive features can only be used in a differentially private way. These papers are similarly motivated, but operate in a very different setting. Finally, several papers study fairness constraints in pipelines, in which an individual is subject to a sequence of classification decisions, and study how the effects of these constraints compound \citep{pipeline1,pipeline2,pipeline3}. Many results in this literature are negative. Our paper gives a positive result in this setting.

\section{Model and Preliminaries}
\label{sec:preliminaries}
Let $\Omega = \X \times \Z \times \Y$ be an arbitrary data domain. Each data point is a triplet $\omega = (x,z,y)$, where $x \in \X$ is the feature vector excluding the sensitive attributes, $z \in \Z$ is a vector of sensitive attributes, and $y \in \Y = \{0,1\}$ is the binary label. In this paper we take $\Z = \{0,1\}^K$, and every $z \in \Z$ is a $K$-dimensional binary vector representing which groups (out of $K$ groups) an individual is a member of. For instance, in a case with $K=4$ groups, an individual with $z = (0,1,0,1)$ is a member of the second and fourth groups. We will use $z_k$ to denote the $k$th entry of $z$.

We assume there exists a distribution $\Ps$ over the unlabeled data domain $\X \times \Z$. We assume the labels are generated by functions in some domain $\F \subseteq \{ f: \X \times \Z \to \Y \}$. In other words, for any data point $\omega = (x,z,y) \in \Omega$, there exists a function $f \in \F$ such that $y = f(x,z)$. This is without loss of generality if we make no assumptions on the complexity of $\F$ --- in this case, functions $f$ can be randomized and represent arbitrary conditional label distributions, and will be the setting we operate in when we assume there is a distribution over $\F$. Alternately, we can make assumptions about the capacity of $\F$, and then aim to form good proxies for \textit{every} labelling function in $\F$. The data generation process can be viewed as first drawing $(x,z)$ from $\Ps$, and then letting $y = f(x,z)$ for some $f \in \F$. We \textit{may} additionally assume there exists a probability distribution $\Qs$ over $\F$. More details are discussed later on.

Our primary goal in this paper is to learn a proxy for $z$ as a function of features $x$,  which we write as  $\hat{z}$, such that any downstream classifier satisfies a variety of fairness constraints with respect to the learned proxy $\hat{z}$ if and only if it satisfies the same fairness constraints with respect to the true underlying $z$, up to small approximation.  Let $\Gs \subseteq \{ g: \X \to [0,M] \}$ be a class of functions that map a feature vector $x \in \X$ to a real-valued number in $[0,M]$. Given $\Gs$, our goal will be to learn $\hat{z} = (\hat{z}_1, \ldots, \hat{z}_K)$ such that for all $k$, $\hat{z}_k \in \Gs$. The $k$th component of $\hat{z}$ can be interpreted as a real-valued predictor for $z_k$.

We assume the downstream learning task for which we want to guarantee fairness can be cast as learning over a hypothesis class $\Hs \subseteq \{ h: \X \to \Y \}$. Thus the goal of the DLs will be to learn $h \in \Hs$ such that $h$ satisfies some statistical notion of fairness. These fairness notions generally require that a statistic of the learned classifier be (approximately) equalized across different groups. While our methods will apply to a broad class of fairness notions including statistical parity and equalized false positive and negative rates (see the appendix for details), in the body we focus on  \textit{equalized error} fairness  which requires that the error rate of the learned classifier be (approximately) equalized across groups. In other words, $h \in \Hs$ satisfies equalized error fairness if:

\begin{equation}\label{eq:equalized-error}
\forall k_1, k_2 \in [K]: \quad
\Pr \left[ h(x) \neq y \, \vert \, z_{k_1} = 1 \right] \approx \Pr \left[ h(x) \neq y \, \vert \, z_{k_2} = 1 \right]
\end{equation}

We first make the following simple, yet important, observation that will allow us to write fairness constraints, usually defined with respect to binary valued group membership, using a real valued proxy.

\begin{clm}
For every $k \in [K]$, we have
\begin{equation}\label{eq:observation}
    \Pr \left[ h(x) \neq y \, \vert \, z_k =1 \right] = \frac{\E \left[ z_k \1 \left[ h(x) \neq y \right] \right]}{\E \left[  z_k \right]}
\end{equation}
\end{clm}

\ifdraft
\begin{proof}
We have
\begin{align*}
    \Pr \left[ h(x) \neq y \, \vert \, z_k =1 \right] &= \frac{\Pr \left[ z_k = 1, h(x) \neq y \right]}{\Pr \left[ z_k =1 \right]} \\
    & = \frac{\E \left[ \1 \left[ z_k = 1 \right] \1 \left[ h(x) \neq y \right]\right] }{\E \left[ \1 \left[ z_k = 1 \right] \right]} \\
    &= \frac{\E \left[ z_k \1 \left[ h(x) \neq y \right] \right]}{\E \left[  z_k \right]}
\end{align*}
\end{proof}
\fi

Observe that the expression on the right hand side of Equation~\eqref{eq:observation} could be evaluated even if the sensitive feature labels $z$ were real valued rather than binary. We exploit this to evaluate these equalized error fairness constraints with our real valued proxies $\hat z$. Observe that if we have a proxy $\hat{z} \in \Gs$, such that for a particular classifier $h \in \Hs$:
\begin{equation}\label{eq:condition0}
\forall k \in [K]: \quad
\frac{\E \left[ z_k \1 \left[ h(x) \neq y \right] \right]}{\E \left[  z_k \right]} \approx \frac{\E \left[ \hat{z}_k (x) \1 \left[ h(x) \neq y \right] \right]}{\E \left[  \hat{z}_k (x) \right]}
\end{equation}
then if $h$ satisfies proxy fairness constraints defined by the proxy $\hat{z}$, i.e., constraints of the form:
\begin{equation}\label{eq:proxy_constraints0}
\forall k_1, k_2 \in [K]: \quad
\frac{\E \left[ \hat{z}_{k_1} (x) \1 \left[ h(x) \neq y \right] \right]}{\E \left[  \hat{z}_{k_1} (x) \right]} \approx \frac{\E \left[ \hat{z}_{k_2} (x) \1 \left[ h(x) \neq y \right] \right]}{\E \left[  \hat{z}_{k_2} (x) \right]}
\end{equation}
it will also satisfy the original fairness constraints with respect to the real sensitive groups $z$ and vice versa (Equation~\eqref{eq:equalized-error}). If the condition in Equation~\eqref{eq:condition0} is satisfied for \textit{every} $h \in \Hs$, then the proxy fairness constraints (Equation~\eqref{eq:proxy_constraints0}) can without loss be used to \textit{optimize} over all fair classifiers in $\Hs$. With this idea in mind, we can formally define a (good) proxy. The constraints we ask for can be interpreted as so-called \textit{multiaccuracy} or \textit{mean consistency} constraints as studied by \cite{multiaccuracy,momentmulti}, defined over the \textit{error regions} of hypotheses in the class $\Hs$: $\{\{(x,y) : h(x) \neq y\} : h \in \Hs\}$.

We will consider two different settings for modelling a multiplicity of downstream learning problems: 1) when there exists a distribution over $\F$ and we want our guarantee to hold with high probability over a draw of $f$ from this distribution, and  2) when we want our guarantee to hold for every $f \in \F$.

\begin{definition}[Proxy]\label{def:proxy}
Fix a distribution $\Ps$ over  $(\X \times \Z)$  and a distribution $\Qs$ over $\F$. We say $\hat{z}$ is an $(\alpha, \beta)$-proxy for $z$ with respect to $(\Ps, \Qs)$, if with probability $1-\beta$ over the draw of $f \sim \Qs$: for all classifiers $h \in \Hs$, and all groups $k \in [K]$,
\[
\left\vert \frac{\E_{(x,z) \sim \Ps} \left[ z_k \1 \left[ h(x) \neq f(x,z) \right] \right]}{\E_{(x,z) \sim \Ps} \left[  z_k \right]} - \frac{\E_{(x,z) \sim \Ps} \left[ \hat{z}_k (x) \1 \left[ h(x) \neq f(x,z) \right] \right]}{\E_{(x,z) \sim \Ps} \left[  \hat{z}_k (x) \right]} \right\vert \le \alpha
\]
If the above condition holds for every $f \in \F$, we say $\hat{z}$ is an $\alpha$-proxy with respect to $\Ps$. When providing in sample guarantees, we take the distributions to be the uniform distributions over the  data set. When distributions are clear from context, we simply write that $\hat z$ is an $(\alpha, \beta)$-proxy.
\end{definition}

Do such proxies exist? We first show the existence of perfect proxies, under the assumption that the sensitive features and the labels are conditionally independent given the other features. Note that this conditional independence assumption can be satisfied in a number of ways --- and in particular is always satisfied if \textit{either} the sensitive features or the labels can be determined as a function of the non-sensitive features --- even if the relationship is arbitrarily complex. For example this will be the case for prediction tasks in which human beings are near perfect. The proxy that we exhibit below is the conditional expectation defined over the underlying joint distribution on $x$ and $z$ and hence will generally not be learnable from polynomially sized samples. Subsequently, we will demonstrate that we can obtain proxies learnable with modest sample complexity. We note that perfect proxies always exist (without requiring a conditional independence assumption) for statistical parity fairness -- see the appendix.

\begin{clm}[Existence of a Proxy]
\label{clm:existence}
For any distribution $\Ps$ over $\X \times \Z$, $\hat{z}(x) = \E \left[z \,| \, x \right]$ is an $\alpha$-proxy with respect to $\Ps$, for $\alpha = 0$, provided that $z$ and $y$ are independent conditioned on $x$.
\end{clm}

\ifdraft
\begin{proof}
Fix $f \in \F$, $h \in \Hs$, and $k \in [K]$. We have that
\[
\E_{(x,z) \sim \Ps} \left[  \hat{z}_k (x) \right] = \E_{x \sim \Ps_\X} \left[  \hat{z}_k (x) \right] = \E_{x \sim \Ps_\X} \left[  \E \left[z_k \, | \, x \right] \right] = \E_{(x,z)} [z_k ]
\]
Also, note that $y= f(x,z)$, and that
\begin{align*}
    \E_{(x,z) \sim \Ps} \left[ \hat{z}_k (x) \1 \left[ h(x) \neq f(x,z) \right] \right] &= \E_{(x,z)} \left[ \E \left[z_k \, | \, x \right] \1 \left[ h(x) \neq y \right] \right] \\
    &= \E_{(x,z)} \left[ \E \left[z_k \1 \left[ h(x) \neq y \right] \, | \, x \right] \right] \\
    &= \E_{(x,z)} \left[z_k \1 \left[ h(x) \neq y \right] \right] \\
    &= \E_{(x,z)} \left[z_k \1 \left[ h(x) \neq f(x,z) \right] \right]
\end{align*}
completing the proof. The second equality holds because of the assumption that conditional on $x$, $y$ and $z$ are independent.
\end{proof}
\fi

\paragraph{Modelling the Proxy Learner (PL)}
The PL wants to learn a proxy in $\hat{z} \in \Gs^K$ as defined in Definition~\ref{def:proxy}. 

Solving this problem requires the knowledge of distributions; however, typically we will only have samples. Therefore, we assume the PL has access to a data set, which consists of two components: 1) $S = \{ (x_i, z_i) \}_{i=1}^n$ which is a sample of $n$ individuals from $\X \times \Z$ represented by their non-sensitive features and sensitive attributes. Throughout we will take $S$ to be $n$ $i.i.d.$ draws from the underlying distribution $\Ps$. 2) $F = \{ f_j \}_{j=1}^m$ which is a sample of $m$ labeling functions (or \textit{learning tasks}) taken from $\F$. The PL does not observe the actual functions $f_j \in \F$ but instead observes the realized labels of functions in $F$ on our data set of individuals $S$: $Y = \{ y_{ij} = f_j (x_i, z_i) \}_{i,j}$. The \textit{empirical} problem of the PL is to find a proxy $\hat{z}$ with respect to the observed data sets.

In this paper we have the PL optimize  squared error subject to the constraints given by the definition of a proxy:
\begin{mini}|l|
{\hat{z}_k \in \Gs}{\frac{1}{n}\sum_{i=1}^{n} \left( z_{ik} - \hat{z}_k (x_i) \right)^2}{}{}
\addConstraint{\frac{\sum_{i=1}^n z_{ik} \1 \left[  h(x_i) \neq y_{ij} \right]}{\sum_{i=1}^n z_{ik}  } = \frac{\sum_{i=1}^n \hat{z}_{k} (x_i) \1 \left[  h(x_i) \neq y_{ij} \right]}{\sum_{i=1}^n \hat{z}_{k} (x_i)}}{}{, \ \forall j \in [m], h \in \Hs}
\label{eqn:program1}
\end{mini}
Note this formulation gives us a decomposition of learning $\hat{z} = (\hat{z}_1, \ldots, \hat{z}_K) \in \Gs^K$ into learning each component $\hat{z}_k \in \Gs$ separately. The squared error objective is not strictly necessary (the constraints encode our notion of a good proxy on their own), but encourages the optimization towards the conditional label distribution of $z$ given $x$, which we showed in Claim \ref{clm:existence} is a good proxy. In our experiments we find this to be helpful.

\paragraph{Modelling the Downstream Learner (DL)} The DLs want to solve fair learning problems using models in some class $\Hs \subseteq \{h: \X \to \Y\}$ subject to the equalized error fairness constraint given in Equation~\eqref{eq:equalized-error}. The DL does not have access to the sensitive attribute $z$ and instead is given the proxy $\hat z \in \Gs^K$  learned by the PL. Thus, for a given learning task represented by some $f \in \F$ (determining the label $y = f(x,z)$), the DL solves the following  learning task subject to proxy fairness constraints.
\begin{mini}|l|
{h \in \Hs}{\E \left[ \text{err} \left( h ; (x,y)\right)\right] }{}{}
\addConstraint{\frac{\E \left[ \hat{z}_{k_1} (x) \1 \left[ h(x) \neq y \right] \right]}{\E \left[  \hat{z}_{k_1} (x) \right]} \approx \frac{\E \left[ \hat{z}_{k_2} (x) \1 \left[ h(x) \neq y \right] \right]}{\E \left[  \hat{z}_{k_2} (x) \right]}}{}{, \ \forall k_1, k_2 \in [K]}
\label{eqn:program3}
\end{mini}
where $\text{err}$ is some arbitrary objective function, and all expectations here are taken with respect to a draw of an individual $(x,z)$ from $\Ps$. Observe that if $\hat z$ is an $\alpha$-proxy, then this is equivalent to solving the original fairness constrained learning problem (defined with respect to the true demographic features $z$) in which the fairness constraints have slack at most $2\alpha$. We remind the reader that our focus in this paper is to solve the problem of the PL, and hence, we avoid standard issues that the DL will face, such as relating empirical and distributional quantities (these issues are identical whether the DL uses the sensitive features $z$ directly or a proxy $\hat z$). 

\paragraph{Game Theory and Online Learning Basics.}
In our analysis, we rely on several key concepts in game theory and online learning which we summarize here. Consider a zero-sum game between two players, a Learner with strategies in $S_1$ and an Auditor with strategies in $S_2$. The payoff function of the game is $U: S_1 \times S_2 \rightarrow \R_{\ge 0}$.

\begin{definition}[Approximate Equilibrium]\label{def:nuapprox}
A pair of strategies $(s_1, s_2) \in S_1 \times S_2$ is said to be a $\nu$-approximate minimax equilibrium of the game if the following conditions hold:
\[
 U(s_1, s_2) - \min_{s'_1 \in S_1}  U(s'_1, s_2)  \le \nu,
\quad
\max_{s'_2 \in S_2}  U(s_1, s'_2) -  U(s_1, s_2)  \le \nu
\]
\end{definition}

Freund and Schapire \cite{Freund} show that if a sequence of actions for the two players jointly has low \textit{regret}, then the uniform distribution over each player's actions forms an approximate equilibrium:

\begin{thm}[No-Regret Dynamics \cite{Freund}]\label{thm:noregret}
    Let $S_1$ and $S_2$ be convex, and suppose  $U(\cdot, s_2): S_1 \to \R_{\ge 0}$ is convex for all $s_2 \in S_2$ and $U(s_1, \cdot): S_2 \to \R_{\ge 0}$ is concave for all $s_1 \in S_1$. Let $(s_1^1, s_1^2, \ldots, s_1^T)$ and $(s_2^1, s_2^2, \ldots, s_2^T)$ be  sequences of actions for each player. If for $\nu_1,\nu_2 \ge 0$, the regret of the players jointly satisfies
    \[
    \sum_{t=1}^T U(s_1^t, s_2^t) - \min_{s_1 \in S_1} \sum_{t=1}^T U(s_1, s_2^t) \le \nu_1 T,
    \quad
    \max_{s_2 \in S_2} \sum_{t=1}^T U(s_1^t, s_2) - \sum_{t=1}^T U(s_1^t, s_2^t) \le \nu_2 T
    \]
then  the pair $(\bar{s}_1, \bar{s}_2)$ is a $(\nu_1+\nu_2)$-approximate equilibrium, where      $\bar{s}_1 = \frac{1}{T}\sum_{t=1}^T s_1^t \in S_1$ and $\bar{s}_2 = \frac{1}{T}\sum_{t=1}^T s_2^t \in S_2$ are the uniform distributions over the action sequences.
\end{thm}

\section{Learning a Proxy from Data}
\label{sec:general}

We now give a general oracle  efficient algorithm that the Proxy Learner can use to learn a proxy, whenever the underlying proxy class $\Gs^K$ is expressive enough to contain one. Our algorithm is in fact a general method for obtaining a \textit{multiaccurate} regression function $\hat z$ with respect to an arbitrary collection of sets --- we instantiate it with sets defined by the error regions of classifiers $h \in \Hs$. In contrast to the algorithms for multiaccurate learning given by \cite{multiaccuracy,multiaccuracy2}, our algorithm has the advantage that it need not be able to evaluate which sets a new example is a member of at test time (but has the disadvantage that it must operate over a sufficiently expressive model class). This is crucial, because we will not know whether a new example $x$ falls into the error region of a classifier $h$ before learning its label.  

Our derivation proceeds as follows. First, we rewrite the constraints in Program~\eqref{eqn:program1} as a large linear program. We then appeal to strong duality to derive the Lagrangian of the linear program. We note that computing an approximately optimal solution to the linear program corresponds to finding approximate equilibrium strategies for both players in the game in which one player ``The Learner'' controls the primal variables and aims to minimize the Lagrangian value, and the other player ``The Auditor'' controls the dual variables and aims to maximize the Lagrangian value.

Finally, if we construct our algorithm in such a way that it simulates repeated play of the Lagrangian game such that both players have sufficiently small regret, we can apply  Theorem~\ref{thm:noregret} to conclude that our empirical play converges to an approximate equilibrium of the game. In our algorithm, the Learner approximately best responds to the mixed strategy of the Auditor -- who plays \textit{Follow the Perturbed Leader (FTPL)} \citep{kalai}, described in the appendix. Note that it is the functions played by the Learner that will eventually form the proxy output by the  \textit{Proxy Learner}. Furthermore, our algorithm will be \textit{oracle efficient}: it will make polynomially many calls to oracles that solve ERM problems over $\Hs$ and $\Gs$. The specific types of oracles that we need are defined as follows.

\begin{definition}[Cost Sensitive Classification Oracle for $\Hs$]
An instance of a Cost Sensitive Classification problem, or a $CSC$ problem, for the class $\Hs$ is given by a set of $n$ tuples $\{x_i,c_i^0,c_i^1 \}_{i=1}^{n}$ such that $c_i^l$ corresponds to the cost for predicting label $l$ on sample $x_i$. Given such an instance as input, a $CSC (\Hs)$ oracle finds a hypothesis $h \in \Hs$ that minimizes the total cost across all points:
$
h \in \argmin_{h' \in \Hs} \sum_{i=1}^{n}\left[h'(x_i)c_i^1 + (1-h'(x_i))c_i^0\right]
$.
\end{definition}

\begin{definition}[Empirical Risk Minimization Oracle for $\Gs$]
\label{def:erm}
An empirical risk minimization oracle for a class $\Gs$ (abbreviated $ERM (\Gs)$) takes as input a data set $S$ consisting of $n$ samples and a loss function $L$, and finds a function $g \in \Gs$ that minimizes the empirical loss, i.e., $g \in \text{argmin}_{g' \in \Gs} \sum_{i=1}^{n} L(g',S_i)$.
\end{definition}

\ifdraft
\paragraph{Follow the Perturbed Leader} Follow the Perturbed Leader (FTPL) is a no-regret learning algorithm that can sometimes be applied -- with access to an oracle -- to an appropriately convexified learning space that is too large to run gradient descent over. It is formulated as a two-player game over $T$ rounds. At each round $t \leq T$, a learner selects an action $a^t$ from its action space $A \subset \{0,1\}^d$, and an auditor responds with a loss vector $\ell^t \in \Re ^d$. The learner's loss is the inner product of $\ell^t$ and $a^t$. If the learner is using an algorithm $A$ to select its action each round, then the learner wants to pick $A$ so that the regret $R_A(T) := \sum_{t=1}^{T} \langle \ell^t, a^t\rangle - \min_{a \in A} \sum_{i=1}^{T} \langle \ell^t, a^t \rangle$ grows sublinearly in $T$. Algorithm~\ref{alg:ftpl_pseudo} accomplishes this goal by perturbing the cumulative loss vector with appropriately scaled noise, and then an action is chosen to minimize the perturbed loss. Pseudocode and guarantees are stated below. 

\begin{algorithm}[h]
\KwIn{learning rate $\eta$}
Initialize the learner $a^1 \in A$\;
\For{t = 1,2, \ldots}{
Learner plays action $a^{t}$\;
Adversary plays loss vector $\ell^t$\;
Learner incurs loss of $\langle \ell^t, a^t \rangle$. Learner updates its action: $a^{t+1} = \argmin_{a \in A} \left\{ \left\langle \sum_{s \le t} \ell^s , a \right\rangle + \frac{1}{\eta} \left\langle \xi^t , a \right\rangle \right\}$ \;
where $\xi^t \sim Uniform \left( [0,1]^d \right)$, independent of every other randomness.
}
\caption{Follow the Perturbed Leader (FTPL)}
\label{alg:ftpl_pseudo}
\end{algorithm}


\begin{thm}[Regret of FTPL \cite{kalai}]
\label{thm:ftpl}
    Suppose for all $t$, $\ell^t \in [-M,M]^d$. Let $\mathcal{A}$ be Algorithm~\ref{alg:ftpl_pseudo} run with learning rate $\eta = 1/(M \sqrt{dT})$. For every sequence of loss vectors $(\ell^1, \ell^2, \ldots, \ell^T)$ played by the adversary,  $\E \left[ R_\mathcal{A}(T) \right] \le 2 M d^{3/2} \sqrt{T}$, where expectation is taken with respect to the randomness in $\mathcal{A}$. 
\end{thm}
\else
\fi

\paragraph{Specifying the Linear Program.} To transform Program~\eqref{eqn:program1} into a linear program amenable to our two-player zero sum game formulation, we do the following: 1) We break the constraints of Program~\eqref{eqn:program1}, which are given as equality of fractions, into joint equality of their numerators and denominators. 2) We expand $\Gs$ to the set of distributions over $\Gs$: we will find a distribution $p \in \Delta (\Gs)$, where $\Delta (\Gs)$ is the set of probability distributions over $\Gs$, and further \textit{linearize} our objective function and constraints by taking expectations with respect to the variable $p \in \Delta (\Gs)$. 3) Finally, we ensure that we have \textit{finitely many} variables and constraints by assuming that $\Hs$ and $\Gs$ have bounded complexity. In particular, given a data set $S$, we can write constraints corresponding to every $h \in \Hs (S) \triangleq \{ (h(x_1), \ldots, h(x_n)): h \in \Hs\}$ where $\Hs(S)$ includes the set of all possible labelings induced by $\Hs$ on $S$. Note that as long as $\Hs$ has finite VC dimension $d_\Hs$, Sauer's Lemma implies $\vert \Hs (S) \vert = O(n^{d_\Hs})$, and therefore we will have only finitely many constraints. Second, instead of working with the entire class $\Gs$, for some appropriately chosen $\epsilon$, and given our data set $S$, we can optimize over (distributions over) an $\epsilon$-covering of $\Gs$ with respect to the data set $S$, which we call $\Gs (S)$. As long as the class $\Gs$ has finite pseudo-dimension $d_\Gs$, it is known that $|\Gs (S)| = O(\epsilon^{-d_\Gs})$, and therefore we reduce our primal variables from distributions over $\Gs$ to distributions over $\Gs (S)$ which will guarantee that we have finitely many variables in our optimization problem. We provide more details in the appendix. Given these considerations, we  formulate the constrained ERM problem of the PL as follows: for every group $k \in [K]$, the PL solves

\begin{mini}|l|
{p_k \in \Delta (\Gs (S))}{\frac{1}{n}\sum_{i=1}^{n} \E_{\hat{z}_k \sim p_k} \left[ \left( z_{ik} - \hat{z}_k (x_i) \right)^2 \right]}{}{}
\addConstraint{\frac{\sum_{i=1}^n \E_{\hat{z}_k \sim p_k} \left[ \hat{z}_{k} (x_i) \right]}{\sum_{i=1}^n z_{ik}} - 1 = 0 \;}{}{}
\addConstraint{
\sum_{i=1}^n \left( z_{ik} - \E_{\hat{z}_k \sim p_k} \left[ \hat{z}_{k} (x_i) \right]\right) \1 \left[ h(x_i) \neq y_{ij} \right] = 0, }{}{\; \forall j \in [m], \forall h \in \Hs (S)}
\label{eqn:programAlgos}
\end{mini}

We will solve this constrained optimization problem by simulating a zero sum two player game on the Lagrangian dual. Given dual variables $\lambda \in \R^d$ (where $d = 1 + m | H(S)|$) we have that the Lagrangian of Program~\eqref{eqn:programAlgos} is given by:
\begin{align}
L(\lambda,p_k) &= \frac{1}{n} \sum_{i=1}^{n} \E_{\hat{z}_k \sim p_k} \left[ \left( z_{ik} - \hat{z}_k (x_i) \right)^2 \right] + \E_{\hat{z}_k \sim p_k} \left[ \lambda_0\left(\frac{\sum_{i=1}^n \hat{z}_{k} (x_i)}{\sum_{i=1}^n z_{ik}} - 1\right) +\sum_{\substack{h \in \Hs (S) \\ j \in [m]}} \lambda_{h,j} \sum_{i=1}^n \left(z_{ik} - \hat{z}_{k} (x_i) \right) \1 \left[ h(x_i) \neq y_{ij} \right] \right]
\label{eqn:game}
\end{align}

Given the Lagrangian, solving linear program~\eqref{eqn:programAlgos} is equivalent to solving the following minimax problem:
\begin{equation}\label{eq:minimaxp}
\min_{p_k \in \Delta (\Gs (S))} \max_{\lambda \in \R^d} L(\lambda,p_k) = \max_{\lambda \in \R^d} \min_{p_k \in \Delta (\Gs (S))} L(\lambda,p_k)
\end{equation}
where the minimax theorem holds because the range of the primal variable, i.e. $\Delta (\Gs (S))$, is convex and compact, the range of the dual variable, i.e. $\R^d$, is convex, and that the Lagrangian function $L$ is linear in both primal and dual variables. Therefore we focus on solving the minimax problem \eqref{eq:minimaxp} which can be seen as a two player zero sum game between the primal player (the Learner) who is controlling $p_k$, and the dual player (the Auditor) who is controlling $\lambda$. Using  no-regret dynamics, we will have the Learner deploy its best response strategy in every round which will be reduced to a call to $ERM (\Gs)$ and let the Auditor with strategies in $\Lambda = \{ \lambda = (\lambda_0, \lambda') \in \R^d: | \lambda_0 | \le C_0, \Vert \lambda' = (\lambda_{h,j})_{h,j} \Vert_1 \le C \}$ play according to Follow the Perturbed Leader (FTPL). We place upper bounds ($C_0$ and $C$) on the components of the dual variable to guarantee convergence of our algorithm; note that the minimax theorem continues to hold in the presence of these upper bounds. We will set these upper bounds optimally in our algorithm to guarantee the desired convergence. 

Our algorithm is described in Algorithm~\ref{alg:ftpl} and its guarantee is given in Theorem~\ref{thm:ftplProxy}. The algorithm returns a distribution over $\Gs$, but we turn the distribution into a deterministic regression function that defines a proxy by taking the expectation with respect to that distribution. We note that the Auditor will employ FTPL for the constraints that depend on $h$, calling upon the cost sensitive classification oracle $CSC(\Hs)$. 

\ifdraft
Given an action $\hat{z}_k$ of the Learner, we write $LC(\hat{z}_k)$ for the $n \times m$ matrix of costs for labelling each data point as 1, where $LC_j(\cdot)$ indicates the column of costs corresponding to the choice of labels $y_{.j}$. Note that this formulation allows us to cast our seemingly nonlinear problem as an $n$-dimension linear optimization problem, which we do by viewing our costs as the inner product of the outputs of a classifier $h$ on the $n$ points and the corresponding cost vector. When we want to enforce equal group error rates, we can define the costs for labeling examples as positive ($h(x)=1$) as a function of their true labels $y_{i,j}$ as:

\begin{align}
\begin{split}
c^0(x_i, y_{ij})=0,
\end{split}
\begin{split}
c^1(x_i, y_{ij})=\left(z_{ik} - \hat{z}_k(x_i)\right) \left(\1 \left[y_{ij}=0\right] - \1\left[y_{ij}=1\right]\right)
\end{split}
\label{eqn:costs}
\end{align}
\fi

We denote the true distribution over $\lambda'$ maintained by the Auditor's FTPL algorithm by $Q_k^t$. Because $Q_k^t$ is a distribution over an exponentially large domain $(O(n^{d_\Hs}))$, we can only aim to represent a sparse version, which we do by efficiently sampling from $Q_k^t$; we call the empirical distribution $\hat{Q}_k^t$. We represent the Auditor's learned distribution over $\lambda_0$ by $P_k^t$, and we find that $P_k^t:=C_0\left(2Bern(p_k^t) - 1\right)$ is a scaled Bernoulli distribution with success probability $p_k^t$, where $p_k^t$ is given in Algorithm~\ref{alg:ftpl}. When we sample $\lambda= \left(\lambda_0,\lambda'\right)$ from the product distribution $P_k^t \times \hat{Q}_k^t$, this means that we are drawing $\lambda_0$ from $P_k^t$ and $\lambda'$ from $\hat{Q}_k^t$. The proof of Theorem~\ref{thm:ftplProxy}, below, is included in the appendix.

\begin{algorithm}
\SetAlgoLined
\KwIn{ Data set $\{x_i,y_{ij},z_i\}_{i=1}^{n} \forall j \in [m]$, target proxy parameter $\alpha$, target confidence parameter $\delta$, upper bound $M$ on proxy values, groups $k \in [K]$}
Set dual variable upper bounds: $C=C_0 = (M^2(1+nM)/2\alpha \sum_{i=1}^n z_{ik}) + 1$\;
Set iteration count: $T = \left\lceil \sqrt{2(1+nM)\left(n^{3/2}CM + C_0\frac{nM}{\sum_{i=1}^n z_{ik}}\right)/ \alpha\sum_{i=1}^n z_{ik}} \; \right\rceil$\;
Set sample count: $W = \left \lceil (1+nM)^2 n^2C^2M^2 \log(\frac{TK}{2\delta})/\left(\alpha \sum_{i=1}^n z_{ik} \right)^2 \; \right \rceil$\;
Set learning rates of FTPL: $\eta =\frac{1}{CM} \sqrt{\frac{1}{nT}}\;$, $\;\eta' = \frac{\sum_{i=1}^n z_{ik}}{C_0nM} \sqrt{\frac{1}{T}}$\;
\For{$k=1$ \textbf{to} $K$}{
Initialize $\hat{z}_k^0=\bar{0}$\;
\For{$t=1$ \textbf{to} $T$}{
\For{$w=1$ \textbf{to} $W$}{
Draw a random vector $\xi^w$ uniformly at random from $[0,1]^{n}$\;
Use oracle $CSC(\Hs)$ to compute:
$(h^{w,t}, j^{w,t}) = \argmin_{h \in H, j \in [m]} - \sum_{t' < t}|\langle LC_j(\hat{z}_k^{t'}),h \rangle| + \frac{1}{\eta} \langle\xi^w,h\rangle$\;
Let $\lambda'^{w,t}$ be defined as $\lambda^{w,t}_{h,j} = sign(2\1 \left[ \langle LC_{j^{w,t}}(\hat{z}_k^{t}),h^{w,t}\rangle > 0 \right] - 1) \times C \1 \left[h = h^{w,t},j = j^{w,t} \right]$\;
}

Let $\hat{Q}_k^t$ be the empirical distribution over $\lambda'^{w,t}$\;
Set distribution over $\lambda_0^t$ : $P_k^t = C_0\left(2Bern(p_k^t) - 1\right)$ where $p_k^t = \min(1,-\eta'(\frac{\sum_{i=1}^n \hat{z}_{k} (x_i)}{\sum_{i=1}^n z_{ik}} - 1)\1\left[\frac{\sum_{i=1}^n \hat{z}_{k} (x_i)}{\sum_{i=1}^n z_{ik}} -1 < 0 \right])$\;
The Learner best responds: $\hat{z}_k^t = \argmin_{\hat{z}_k \in \Gs} \mathbb{E}_{\lambda \sim P_k^t \times \hat{Q}_k^t} L(\hat{z}_k, \lambda)$ by calling $ERM(\Gs)$.
}
}
\KwOut{$\hat{p}$ = uniform distribution over $\{ \hat{z}^1$,...,$\hat{z}^T \}$}
\caption{Learning a Proxy}
\label{alg:ftpl}
\end{algorithm}

\begin{thm}[$\alpha$-Proxy for $m$ labeling functions taken from $\F$]
\label{thm:ftplProxy}
Fix any $\alpha$, and $\delta$. Suppose $\Hs$ has finite VC dimension, and $\Gs$ has finite pseudo-dimension. Suppose $\Delta (\Gs)^K$ contains a $0$-proxy. Then given access to oracles $CSC(\Hs)$ and $ERM(\Gs)$, we have that with probability at least $1-\delta$, Algorithm~\ref{alg:ftpl} returns a distribution $\hat{p} \in \Delta (\Gs)^K$ such that $\hat{z} (x) \triangleq \E_{g \sim \hat{p}} \left[ g(x) \right] = \frac{1}{T} \sum_{t=1}^T \hat{z}_t (x)$ is an $\alpha$-proxy. 

\end{thm}

\ifdraft
\section{Learning a Linear Proxy}
\label{sec:linear}
Our follow-the-perturbed-leader based algorithm can handle an arbitrary proxy class $\Gs$, so long as we have an oracle for optimizing over it. But the algorithm can simplify substantially when the primal optimization problem is convex in its parameters, as it is when we choose $\Gs$ to be the set of linear proxies.  In this section we consider the case in which $\hat{z}(x;\theta)$ is a linear regression of $x$ on the true sensitive features $z$, taking the form $\hat{z}(x;\theta) = \theta x$. Because both $\hat{z}$ and its negation are convex in $\theta$, we can find such a proxy by implementing a two-player game in which the Proxy Learner uses Online Projected Gradient Descent, and the Auditor best responds by appealing to an oracle over $\Hs$, as in Algorithm~\ref{alg:ftpl}. We summarize Online Projected Gradient Descent below.

\paragraph{Online Projected Gradient Descent}
Online Projected Gradient Descent is a no-regret online learning algorithm which we can formulate as a two-player game over $T$ rounds. At each round $t \leq T$, a Learner selects an action $\theta^t$ from its action space $\Theta \subset \Re^d$ (equipped with the $L_2$ norm) and an Auditor responds with a loss function $\ell^t: \Theta \to \R_{\ge 0}$. The learner's loss at round $t$ is $\ell^t(\theta^t)$. If the learner is using an algorithm $A$ to select its action each round, then the learner wants to pick $A$ so that the regret $R_A(T) := \sum_{t=1}^{T} \ell^t (\theta^t) - \min_{\theta \in \Theta} \sum_{i=1}^{T} \ell^t(\theta)$ grows sublinearly in $T$. When $\Theta$ and the loss function played by the Auditor are convex, the Learner may deploy Online Projected Gradient Descent (Algorithm~\ref{alg:descent}) to which the Auditor best responds. In this scenario, each round $t$, the Learner selects $\theta^{t+1}$ by taking a step in the opposite direction of the gradient of that round's loss function, and $\theta^{t+1}$ is projected into the feasible action space $\Theta$. Pseudocode and the regret bound are included below.

\begin{algorithm}[t]
\KwIn{learning rate $\eta$}
Initialize the learner $\theta^1 \in \Theta$\;
\For{$t=1, 2, \ldots$}{
Learner plays action $\theta^{t}$\;
Adversary plays loss function $\ell^t$\;
Learner incurs loss of $\ell^t (\theta^t)$\;
Learner updates its action:
$\theta^{t+1} = \text{Proj}_{\Theta} \left( \theta^{t} - \eta \nabla \ell^{t} (\theta^{t}) \right)$
}
\caption{Online Projected Gradient Descent}
\label{alg:descent}
\end{algorithm}

\begin{thm}[Regret for Online Projected Gradient Descent \cite{zinkevich}]\label{thm:gdregret} 
    Suppose $\Theta \subseteq \R^d$ is convex, compact and has bounded  diameter $D$: $\sup_{\theta, \theta' \in \Theta} \left\Vert \theta - \theta' \right\Vert_2 \le D$. Suppose for all $t$, the loss functions $\ell^t$ are convex and that there exists some $G$ such that $\left\Vert \nabla \ell^t (\cdot) \right\Vert_2 \le G$. Let $\mathcal{A}$ be Algorithm~\ref{alg:descent} run with learning rate $\eta = D/(G \sqrt{T})$. We have that for every sequence of loss functions $(\ell^1, \ell^2, \ldots, \ell^T)$ played by the adversary, $R_\mathcal{A}(T) \le GD \sqrt{T}$.
\end{thm}

\paragraph{Specifying and Solving the Linear Program}
Many aspects of our setup for Algorithm~\ref{alg:ftpl} are shared here. In particular, we once again focus on solving the minimax problem \eqref{eq:minimaxp} by viewing it as a two player zero sum game. The primal player (the Learner) controls $p_k$, but this time explicitly through the parameter $\theta$, and the dual player (the Auditor) controls $\lambda$. We continue to utilize no-regret dynamics; now, however, the Proxy Learner plays Online Projected Gradient Descent, and the Auditor with strategies in $\Lambda = \{ \lambda = (\lambda_0, \lambda') \in \R^d: | \lambda_0 | \le C_0, \Vert \lambda' = (\lambda_{h,j})_{h,j} \Vert_1 \le C \}$ best responds by appealing to a cost sensitive classification oracle over the class $H$ (CSC($H$)). Upper bounds ($C_0$ and $C$) on the dual variable components are again set to guarantee convergence.

Our algorithm is described in Algorithm~\ref{alg:EqualErrors} and its guarantee is given in Theorem~\ref{thm:linear}. The algorithm once again returns a distribution over $\Gs$, but because $\hat{z}$ is now convex in $\theta$, we can simply return an average over $\theta$ as our final model. This is the approach we take in our experiments. In contrast to the previous algorithm, the Auditor calls upon $CSC(\Hs)$ to maximize the exact costs without additional noise. Note that in practice, we need to use a heuristic to estimate $CSC(\Hs)$, as such an oracle is generally not available in practice; more information is given in Section~\ref{sec:experiments}.

\begin{algorithm}
\SetAlgoLined
\KwIn{ Data set $\{x_i,y_{ij},z_i\}_{i=1}^{n} \forall j \in [m]$, target proxy parameter $\alpha$, upper bound $M$ on proxy values, upper bound $B$ on magnitude of proxy gradient values, groups $k \in [K]$, dimension of parameter vector $d$, diameter of $\Theta$ space $D$;
}
Set dual variable upper bound: $C = \lceil \left( M^2(1+nM) + 2\alpha \sum_{i=1}^n z_{ik}(1+nM))^{-1} \right)/ \left( \alpha \sum_{i=1}^n z_{ik} \right) \rceil$\;
Set iteration count: $T = \lceil \left( d D (1+nM)\left( 2MB + nCB \left( \sum_{i=1}^{n} z_{ik}\right)^{-1}\right) / \left( \alpha \sum_{i=1}^n z_{ik} \right) \right)^2 \rceil$\;
\For{$k=1$ \textbf{to} $K$}{
\For{$t=1$ \textbf{to} $T$}{
    Set $\eta = t^{-1/2}$\;
    $h^*, j^* = \argmax_{h \in \Hs, j \in [m]} |\sum_{i=1}^n \left(z_{ik} - \hat{z}_{k} (x_i; \theta_k) \right) \1 \left[ h(x_i) \neq y_{ij} \right]| $\;
    \If{$|\sum_{i=1}^n \left(z_{ik} - \hat{z}_{k} (x_i; \theta_k)\right) \1 \left[ h^*(x_i) \neq y_{ij^*} \right]| \geq |\frac{\sum_{i=1}^n \hat{z}_{k} (x_i; \theta_k)}{\sum_{i=1}^n z_{ik}} - 1|$}{
    $\lambda_{h^*,j^*} = C \cdot \sign \left( \sum_{i=1}^n \left(z_{ik} - \hat{z}_{k} (x_i; \theta_k)\right) \1 \left[ h^*(x_i) \neq y_{ij^*} \right]\right) $\;
    }
    \Else{$\lambda_0 = C \cdot \sign \left( \frac{\sum_{i=1}^n \hat{z}_{k} (x_i; \theta_k)}{\sum_{i=1}^n z_{ik}} - 1\right)$}
    Set all other $\lambda$ to 0\;

    $\frac{\delta L}{\delta \theta_k}=\frac{2}{n}\sum_{i=1}^{n} \nabla_{\theta_k} \hat{z}_k(x_i;\theta_k)\cdot \left( \hat{z}_k(x_i;\theta_k)- z_{ik} \right) + \lambda_0 \frac{\sum_{i=1}^n \nabla_{\theta_k}\hat{z}_{k} (x_i; \theta_k)}{\sum_{i=1}^n z_{ik}} - \lambda_{h^*,j^*} \sum_{i=1}^n \nabla_{\theta_k} \hat{z}_{k} (x_i; \theta_k) \1 \left[ h^*(x_i) \neq y_{ij^*} \right]$\;
$\theta_k^{(t+1)} = \theta_k^{(t)}-\eta^{t} \frac{\delta L}{\delta \theta_k}$\;
}
}
\KwOut{$\hat{p}$ = uniform distribution over $\{ \hat{z}(\cdot; \theta^1)$,...,$\hat{z}(\cdot; \theta^T) \}$}
\caption{Learning a Linear Proxy}
\label{alg:EqualErrors}
\end{algorithm} 

\begin{thm}[Learning a Linear $\alpha$-Proxy]
\label{thm:linear}
Fix any $\alpha$. Suppose $\Hs$ has finite VC dimension, and $\Gs$ has finite pseudo-dimension. Suppose $\Delta (\Gs)^K$ contains a $0$-proxy. Then given access to oracle $CSC(\Hs)$, Algorithm~\ref{alg:EqualErrors} returns a distribution $\hat{p} \in \Delta (\Gs)^K$ such that $\hat{z} (x) \triangleq \E_{g \sim \hat{p}} \left[ g(x) \right] = (1/T) \sum_{t=1}^T \hat{z}_t (x)$ is an $\alpha$-proxy.
\end{thm}

\fi

\section{Proxies for the Entire Class of Labeling Functions}
\label{sec:functionclass}

Finally, we show that with minor adaptations, and given access to the right oracle, one can learn a proxy for an entire class of labeling functions $\F$, instead of only a finite sample of $m$ functions from $\F$. Accordingly, our data set here only consists of $n$ individual $S= \{ (x_i,z_i) \}_{i=1}^n$, with no observed labels, as our goal is to learn a proxy which is good for \textit{every} $f \in \F$. In this case, the second set of constraints in Program~\eqref{eqn:programAlgos} will be re-written as:

\begin{align}
\sum\nolimits_{i=1}^n \left( z_{ik} - \E_{\hat{z}_k \sim p_k} \left[ \hat{z}_{k} (x_i) \right]\right) \1 \left[ h(x_i) \neq f(x_i, z_i) \right] = 0, \ \forall f \in \F, \forall h \in \Hs
\label{eqn:aboveconstraint}
\end{align}

Note that $1[h(x_i) \neq f(x_i,z_i)] = 1[(h \oplus f) (x_i,z_i) = 1]$ where $h \oplus f$ denotes the XOR of $h$ and $f$ over the $\X \times \Z$ domain: $(h \oplus f) (x,z) = h(x) \oplus f(x,z)$. Define
$\Hs \oplus \F = \left\{ h \oplus f : h \in \Hs, f \in \F \right\}$. We rewrite Equation~\eqref{eqn:aboveconstraint} as:

\[
\sum\nolimits_{i=1}^n \left( z_{ik} - \E_{\hat{z}_k \sim p_k} \left[ \hat{z}_{k} (x_i) \right]\right) \1 \left[ g (x_i, z_i) = 1 \right] = 0, \ \forall g \in \Hs \oplus \F
\]

\ifdraft
Therefore, assuming access to a cost sensitive classification oracle for $\Hs \oplus \F$, we can solve the corresponding optimization problem in an oracle-efficient manner, with the following adjusted linear program and Lagrangian dual:

\begin{mini}|l|
{p_k \in \Delta (\Gs (S))}{\frac{1}{n}\sum_{i=1}^{n} \E_{\hat{z}_k \sim p_k} \left[ \left(z_{ik} - \hat{z}_k (x_i) \right)^2\right]}{}{}
\addConstraint{\frac{\sum_{i=1}^n \E_{\hat{z}_k \sim p_k} \left[\hat{z}_k (x_i)\right]}{\sum_{i=1}^n z_{ik}} - 1 = 0\;}{}{}
\addConstraint{
\sum_{i=1}^{n} \left( z_{ik} - \E_{\hat{z}_k \sim p_k} \left[ \hat{z}_k (x_i) \right] \right) \1 \left[ g (x_i, z_i) = 1 \right] = 0, }{}{\; \forall g \in (\Hs \oplus \F)(S)}
\label{eqn:programXOR}
\end{mini}

\begin{align}
\label{eqn:gameXOR}
L(\hat{z}_k, \lambda) &= \sum_{i=1}^{n} \E_{\hat{z}_k \sim p_k} \frac{\left( z_{ik} - \hat{z}_k (x_i) \right)^2}{n} + \E_{\hat{z}_k \sim p_k} \left[ \lambda_0\left(\frac{\sum_{i=1}^n \hat{z}_k (x_i)}{\sum_{i=1}^n z_{ik}} - 1\right) +\sum_{g \in \Hs \oplus \F} \lambda_{g} \sum_{i=1}^n \left(\hat{z}_k (x_i) - z_{ik}\right) \1 \left[ g(x_i) =1 \right] \right]
\end{align}

We need only minimally alter Algorithm~\ref{alg:ftpl} to find an approximate solution to this game. Rather than call on $CSC(\Hs)$ to find the most violated constraint of Program~\ref{eqn:programAlgos}, we call on $CSC(\Hs \oplus F)$, using the cost for labeling $g(x_i)=1$:
\begin{align}
\begin{split}
c^0(x_i)=0,
\end{split}
\begin{split}
c^1(x_i)=\left(z_{ik} - \hat{z}_k(x_i)\right)
\end{split}
\end{align}

\fi

\begin{thm}[$\alpha$-Proxy for the entire $\F$]
\label{thm:functionclass}
Fix any $\alpha$, and $\delta$. Suppose $\Hs$ and $\F$ have finite VC dimension, and $\Gs$ has finite pseudo-dimension. Suppose $\Delta (\Gs)^K$ contains a $0$-proxy. Then given access to oracles $CSC(\Hs \oplus \F)$ and $ERM(\Gs)$, we have that with probability at least $1-\delta$, Algorithm~\ref{alg:ftpl} returns a distribution $\hat{p} \in \Delta (\Gs)^K$ such that $\hat{z} (x) \triangleq \E_{g \sim \hat{p}} \left[ g(x) \right] = (1/T) \sum_{t=1}^T \hat{z}_t (x)$ is an $\alpha$-proxy.
\end{thm}

\ifdraft
\emily{Added proof}
\begin{proof}
It suffices to observe that Algorithm~\ref{alg:ftpl} reduces to Algorithm~\ref{alg:functionclass} if we replace $CSC(\Hs)$ with $CSC(\Hs \oplus \F)$ and consider one vector of dummy labels, $y_{ij}$, where $j=1$ and $y_{i1}=0$ for all samples $i$.
\end{proof}
\fi

\section{Generalization Theorems}
\label{sec:generalization}
In this section, we provide generalization guarantees for a proxy using a \textit{uniform convergence} approach.

First, in Theorem~\ref{thm:gen-x}, we consider the case where there is no distribution over the class of labeling functions $\F$, and we want to form a good proxy for \textit{every} labelling function in $\F$; in particular, we show how many samples in $S \sim \Ps^n$ are required to guarantee (with high probability over $S$) that \textit{every} $\hat z$ that is a good proxy with respect to the sample $S$ is also a good proxy with respect to the underlying distribution of the data points $\Ps$. Second, in Theorem~\ref{thm:gen-xf}, we consider the case where there is a distribution $\Qs$ over the class of labeling functions $\F$, in addition to the distribution $\Ps$ over individuals; in particular, we show how many samples in $S \sim \Ps^n$ and $F \sim \Qs^m$ are required to guarantee (with high probability over $S$ and $F$) that \textit{every} $\hat z$ that is a good proxy with respect to the sample $(S,F)$ is also a good proxy with respect to the distributions $(\Ps, \Qs)$. We point out that our uniform convergence bounds are taken over the entire $\Delta (\Gs)^K$, not only $\Gs^K$, because our algorithm outputs an object in $\Delta (\Gs)^K$.

It turns out that the sample complexity of learning such proxies can be characterized by the \textit{pseudo-dimension} ($Pdim$) of the proxy class $\Gs$, which is a standard notion used in the learning theory literature (see for e.g. \cite{haussler}) to measure the complexity of a \textit{real-valued} function class. While we provide the formal definition of pseudo-dimension in the appendix, we note a couple of facts regarding this notion. First, pseudo-dimension generalizes the notion of \textit{VC dimension} ($VCdim$) which is typically used to measure the complexity of \textit{binary} function classes.

\begin{fact}[\cite{haussler}]\label{fact:pseudo-vc}
If $\Hs \subseteq \{ h: \X \to \{0,1\} \}$, then $Pdim (\Hs) = VCdim (\Hs)$.
\end{fact}

Second, if $\Gs$ is a class of $d$-dimensional \textit{linear} proxies, then $Pdim(\Gs) = d$ 

\begin{fact}[\cite{haussler}]\label{lem:vectorspace}
If $\Gs \subseteq \{ g: \X \to \R \}$ forms a vector space of dimension $d$, then $Pdim (\Gs) = d$.
\end{fact}

 With this notion of pseudo-dimension in hand, we formally state our first generalization theorem below in Theorem~\ref{thm:gen-x}. We note that in addition to the pseudo-dimension of $\Gs$, our sample complexity bound depends on the VC dimension of $\Hs$ and $\F$ as well because we take a uniform convergence approach that requires bounding the difference of empirical and distributional expectations appearing in the definition of the proxy, for \textit{all} classifiers $h \in \Hs$, and \textit{all} learning tasks $f \in \F$. The sample complexity bound will further depend polynomially on $M$ (a uniform upper bound for functions in $\Gs$), $\mu$ (smallest probability measure of groups), and $\mu_\Gs$ (smallest probability measure of groups, determined by proxies in $\Gs$). The proof of this theorem is in the appendix.

\begin{thm}[Generalization over $\Ps$]\label{thm:gen-x}
Fix any $\epsilon$ and $\delta$. Fix a distribution $\Ps$ over $\X \times \Z$. Suppose $Pdim (\Gs) = d_\Gs$, $VCdim(\F) = d_\F$, and $VCdim(\Hs) = d_\Hs$. We have that with probability at least $1-\delta$ over $S \sim \Ps^n$, every $\hat z$ that is an $\alpha$-proxy with respect to the data set $S$ is also an $(\alpha + \epsilon)$-proxy with respect to the underlying distribution $\Ps$, provided that
\[
n = \tilde{\Omega} \left( \frac{M^2 \left( d_\Gs + \max \left\{ d_\Hs, d_\F \right\}  + \log \left( K / \delta \right) \right)}{\mu \mu_\Gs^2 \left( \min \left\{ \mu, \mu_\Gs \right\} \right)^2 \epsilon^2} \right)
\]
where
$
\mu = \min_{1 \le k \le K} \E_{(x,z) \sim \Ps} \left[ z_k \right], \ \mu_\Gs = \inf_{\hat z \in \Gs} \left\{  \min_{1 \le k \le K} \E_{(x,z) \sim \Ps} \left[ \hat{z}_k (x) \right] \right\}
$.
\end{thm}

\begin{rmk}
We remark that Theorem~\ref{thm:gen-x} subsumes the case when there is only one, or more generally, finitely many learning tasks, because it is known that if $| \F | < \infty$, then we have $d_\F \le \log \left( | \F | \right)$.
\end{rmk}

We note that the sample complexity bound of Theorem~\ref{thm:gen-x} grows with the VC dimension of $\F$, i.e., we are assuming $\F$ has bounded complexity ($d_\F < \infty$). In our next generalization theorem, we consider the setting where there is a distribution over $\F$ from which an $i.i.d.$ sample is collected. This distributional modeling allows us to make no assumption on the complexity of $\F$, i.e., $\F$ can have $d_\F = \infty$. This allows us to handle real labels, without the need to make any assumptions about their underlying complexity. The proof of Theorem~\ref{thm:gen-xf} is provided in the appendix.

\begin{thm}[Generalization over $\Ps$ and $\Qs$]\label{thm:gen-xf}
Fix any $\epsilon$, $\delta$, and $\beta$. Fix a distribution $\Ps$ over $\X \times \Z$, and a distribution $\Qs$ over $\F$. Suppose $Pdim (\Gs) = d_\Gs$, and $VCdim(\Hs) = d_\Hs$. We have that with probability at least $1-\delta$ over $S \sim \Ps^n$ and $F \sim \Qs^m$, every $\hat z$ that is an $(\alpha, 0)$-proxy with respect to the data set $(S, F)$ is also a $((\alpha + \epsilon) / \beta, \beta )$-proxy with respect to the underlying distributions $(\Ps, \Qs)$, provided that
\[
n = \tilde{\Omega} \left( \frac{ M^{2} \left( d_\Gs + \max \left\{ d_\Hs, \log \left( m \right) \right\}  + \log \left( K / \delta \right) \right)}{ \mu \mu_\Gs^2 \left( \min \left\{ \mu, \mu_\Gs \right\} \right)^2 \epsilon^2} \right),
 \ m = \tilde{\Omega} \left( \frac{ M^{2} \left( K d_\Gs \log \left( \left\vert \text{supp} (\Ps) \right\vert \right) + \log \left( 1 / \delta \right) \right) }{ \mu_\Gs^2 \left( \min \left\{ \mu, \mu_\Gs \right\} \right)^2 \epsilon^4} \right)
\]
where $\text{supp} (\Ps)$ is the support of $\Ps$, and $\mu$ and $\mu_\Gs$ are defined as in Theorem~\ref{thm:gen-x}.
\end{thm}

\begin{rmk}
Our bounds contain the term $\mu_{\Gs}$ because they are algorithm independent uniform convergence bounds. The algorithms we give in this paper however always produce a $\hat{z}$ that satisfies $\E_{S} \left[ \hat{z}_k (x) \right] \ge (1/2)\E_{S} \left[ z_k \right]$ in sample. Together with standard arguments, this allows us to give generalization guarantees for our algorithms that remove the dependence on $\mu_\Gs$. Technically, this follows from applying our uniform convergence theorems to the class $\Gs^K(\mu) \equiv \{ \hat{z} \in \Gs^K: \forall k, \, \E_{(x,z) \sim \Ps} \left[ \hat{z}_k (x) \right] = \Omega (\mu) \}$.
\end{rmk}

\section{Experiments}
\label{sec:experiments}
In this section, we perform an empirical evaluation of our proxy training algorithms. Observe that our theorems are predicated on two assumptions that are either difficult to verify or else do not hold exactly in practice:
\begin{enumerate}
    \item Our class $\Gs$ contains a good proxy. We cannot verify this without demonstrating one, which we have to do by training a good proxy. 
    \item We have a cost sensitive classification oracle for $\Hs$ (and in the non-linear case, also for $\Gs$). In practice we do not have ``oracles'' and most learning problems are NP-complete, but we have good heuristics that we can use in place of our oracles.
\end{enumerate}

Our experiments are aimed at verifying the utility of our algorithms, even in the simple case in which we take $\Gs$ to be the set of linear regression functions and $\Hs$ to be the set of linear threshold functions.

\subsection{Methodology}

\label{sec:methodology_preiminaries}

\subsubsection{Weighted Binary Sample Transformation}
\label{sec:swt}

To use a real-valued proxy with standard downstream algorithms for fair machine learning algorithms, many of which assume sensitive features are binary or categorical, we transform a dataset with real-valued sensitive features (see \S\ref{sec:data}) into a dataset with twice as many samples, each of which is paired with a \textit{binary} group membership and sample weight (see Definition~\ref{def:swt} and Claim~\ref{clm:swt}). We note that many learning algorithms are already equipped to handle sample weights, so our transformed dataset fits nicely into existing methods. 

\subsubsection{Paired Regression Classifier}
\label{sec:prc}

For the oracle $CSC(\Hs)$, we experiment with the paired regression classifier (PRC) used in \cite{minimax2} and \cite{gerrymandering}. The PRC produces a \emph{linear threshold function}, just as logistic regression does -- see Definition~\ref{def:PRC}.

\subsubsection{Reductions algorithm for error parity}
\label{sec:reductions_accuracy_parity}

The reductions algorithm for error parity was introduced in \cite{agarwal2018reductions} as a method of producing randomized ensembles of classifiers that achieve high population accuracy while satisfying accuracy parity between sensitive groups. The algorithm takes a relaxation parameter $\gamma \in [0, 1]$ that specifies the maximum allowable difference in error between any two sensitive groups. We implemented the algorithm and augmented it to support arbitrary sample weights. See Definition~\ref{def:reductions_error_parity} for a precise specification. 

\subsection{Experimental Process}
\label{sec:exp_process}

We performed a variety of experiments on real, fairness-sensitive datasets, taking $\Gs$ to be the model class of linear regression functions and $\Hs$ to be the class of linear threshold functions. We compared the performance of a proxy trained with our algorithm (``$\Hs$-proxy'') to the performance achieved using the true sensitive features (``true labels''), a binary proxy in the form of a logistic regression model (``baseline proxy''), and a real-valued proxy in the form of a linear regression model (``mean-squared error (MSE) proxy''). We evaluated the performance of the models produced by the downstream learner for each type of proxy on both group error disparity (with respect to the true sensitive features) and overall population accuracy.

For the three types of proxies and the true labels we performed the following:
\begin{enumerate}
    \item Train the proxy.
    \item If the proxy is real-valued ($\Hs$- or MSE proxy), apply the Weighted Binary Sample Transformation to the dataset.
    \item Train a downstream learner using the reductions algorithm for accuracy parity (described in Section~\ref{sec:reductions_accuracy_parity}) to produce a relaxation curve of models over 10 values of $\gamma \in \{0, 0.005, 0.001, \dots, 0.045\}$.
\end{enumerate}

We plot the performance of the proxies with respect to population error and error disparity between the sensitive groups. The error disparity is plotted with respect to the \textit{true sensitive features}, even for the downstream models that only have access to the proxy features during training.

\subsection{Implementation}
\label{sec:implementation}

We implemented a generic and slightly simplified version of our proxy training algorithm in PyTorch that (approximately) solves our constrained optimization problem via gradient descent. Our implementation leverages auto-differentiation to avoid solving for gradients explicitly in closed form and allows us to select $\Gs$ as an arbitrary architecture multi-layer perceptron (MLP). In particular, this permits $\Gs$ as the class of linear models--which is equivalent to the class of one-layer perceptrons--but also enables the use of our algorithm to train more complex proxies. We also support an arbitrary downstream learner class $\Hs$,  given an algorithm for training regression models in $\Hs$. \footnote{See the Section~\ref{sec:appendix_implementation} of the appendix for a full report of the implementation details and hyperparameter selection.}

\subsection{Data}
\label{sec:data}

We primarily relied on the recently published American Community Survey (ACS) datasets and tasks from~\cite{DBLP:journals/corr/abs-2108-04884}. Rather than looking at the entire United States, we focused our analysis on data from New York state, which we found was sufficiently large to admit excellent out-of-sample generalization. For each task, we examined three sensitive features: sex, age, and race. Because we focus on binary sensitive attributes, we transformed age and race into binary features; we thresholded age at 40 years, and we treated race as a white/non-white binary indicator. All downstream models were trained to make predictions without having access to \textit{any} of the three sensitive features, regardless of which feature the model enforced fairness with respect to. Categorical features were converted into one-hot encoded vectors and the dimensionality $d$ was computed before one-hot encoding. The table below summarizes the tasks used in our experiments. Full details about each prediction task are specified in the appendix in Section~\ref{sec:appendix_datasets}.
 
\begin{table}[ht]
\centering
\begin{tabular}{|l|c|c|l|l|l|}
\hline
Dataset   & \multicolumn{1}{l|}{Sample Count} & \multicolumn{1}{l|}{$d$} & Label                                                                       & Sensitive Feature(s)     \\  
\hline

ACSEmployment & 196104 & 12 & Employment & Race, sex, age \\
\hline
ACSIncome & 101270 & 4 & Income > \$50K & Race, sex, age \\
\hline
ACSIncomePovertyRatio & 196104 & 15 & Income-Poverty Ratio < 250\% & Race, sex, age \\
\hline
ACSMobility & 39828 & 17 & Same address one year ago & Race, sex, age \\
\hline
ACSPublicCoverage & 71379 & 15 & Health Insurance & Race, sex, age \\
\hline
ACSTravelTime & 89145 & 8 & Commute > 20 minutes& Race, sex, age \\
\hline
\end{tabular}
\label{tab:data}
\end{table}

\subsection{Results Overview}

\begin{itemize}

  \item A linear proxy trained with our algorithm often serves as an excellent substitute for the true sensitive features and enables us to train downstream models that attain high population accuracy while enforcing fairness constraints with respect to the true sensitive features. This performance is robust to relaxations in fairness constraints of the downstream learner.
  
  \item Models trained on $\Hs$-proxies almost never performed worse than those trained on a naive baseline--which for linear models is a logistic regression trained to predict the binary sensitive feature of each instance--and often performed far better. 
   
  \item Nearly all downstream models we experimented with generalized extremely well out-of-sample with respect to both fairness and accuracy. In fact, for most experiments the in-sample and out-of-sample plots appear nearly identical in terms of both the shape of the tradeoff curves and the values they span. This generalization performance can likely be attributed to the choice to use linear-complexity $\Gs$ and $\Hs$ rather than more complex classes. Further experimentation is necessary before assuming good out-of-sample generalization for non-linear proxies or downstream learners. 
  
  \item On some tasks, the $\Hs$-proxy failed to serve as a good substitute for the true sensitive features, resulting in a downstream model that violated the intended fairness constraints. However, in each of our experiments where this occurred, the failure of the proxy could be detected at training time. We address these failures in more detail in the appendix (Section~\ref{sec:addressing_failures}). 
    
   \item Often, the MSE proxy--which is a simple linear regression model without the additional multi-accuracy constraints--serves as a performant proxy. This empirical finding justifies the use of the MSE in the objective function of our constrained optimization problem as a heuristic for finding good solutions. We also find that the real-valued MSE proxy often out-performs the naive binary-valued baseline. This is predicted by our Claim \ref{clm:existence}, which proves that this will work whenever the conditional distribution on the protected feature can be well approximated by a linear function. 

\end{itemize}

\subsection{Plots}
\label{sec:plots}

In this section we will analyze the experimental results on the ACSIncome dataset on the three sensitive features race, age, and sex. The ACSIncome dataset is an improved version of the popular Adult dataset, making these experiments more easily contextualized with those in the existing literature compared to the other ACS tasks. Moreover, we found these three tasks had diverse results that demonstrated the capability of our proxy algorithm while also revealing some of its shortcomings in practice. The remaining experiments and plots, including those on which the proxy failed more dramatically, can be found in the appendix in Section~\ref{sec:appendix_plots}. \footnote{Since our downstream models are randomized ensembles, we report all statistics \textit{in expectation} over these ensembles. This means that we can create linear combinations of models to trace a Pareto frontier, which is plotted as a dotted line.}

\subsubsection{ACS-Income-Race}
Fig.~\ref{fig:acs_income_race} displays nearly ideal results supporting the theory. In sample, downstream models trained on true sensitive features exhibit the best tradeoff curve, followed by the $\Hs$-proxy, and then the MSE proxy, all three of which exhibit a sensible tradeoff between error disparity and population error. The $\Hs$-proxy induces a tradeoff curve with similar shape to that of the true labels, but with ${\sim}$0.003 greater error disparity and ${\sim}$0.01 greater error. The curve of the MSE proxy is similar in shape to that of the $\Hs$-proxy but with accuracy ${\sim}0.002$ less than the $\Hs$-proxy. The least disparate model induced by the MSE proxy is equally accurate to that of the $\Hs$-proxy, but marginally less disparate. The baseline proxy exhibits the worst downstream performance, with a clustered tradeoff curve that is Pareto dominated by models from all other proxies. For \textit{any} model trained with the baseline proxy, we can improve accuracy by more than 0.01 without increasing disparity by switching to some model trained with our $\Hs$-proxy. 

Out-of-sample behavior of all models is quite similar to in-sample, though the maximum disparity for all curves decreases from  ${\sim}0.025$ to ${\sim}0.02$ and the least disparate models trained on the true labels increase disparity from  ${\sim}0$ to ${\sim}0.0025$. 
 
\begin{figure*}[h]
    \centering
    \includegraphics[width=0.7\textwidth]{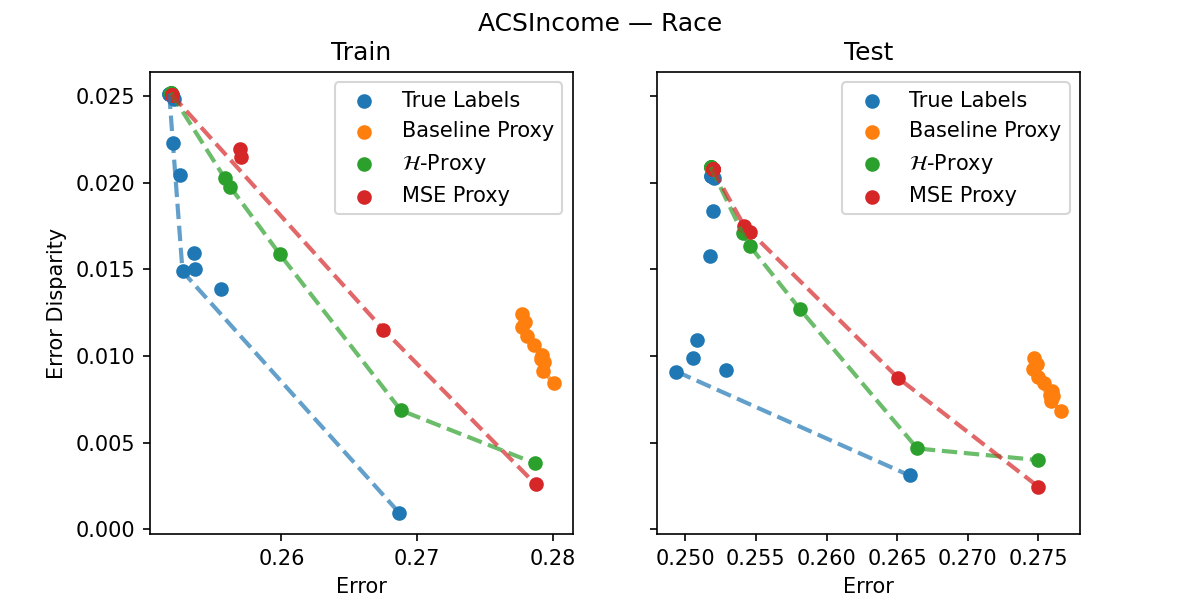}
    \captionsetup[subfigure]{labelformat=empty, belowskip=0pt}
    \caption{Proxy results on the ACSIncome dataset with race as sensitive feature}
    \label{fig:acs_income_race}
\end{figure*}

\subsubsection{ACS-Income-Age}
In Fig.~\ref{fig:acs_income_age} we observe that the downstream performance of models trained on the $\Hs$-proxy are nearly identical to that of the models trained with the true sensitive features, indicating success for our proxy algorithm. Models trained on both of these proxies exhibit a clean tradeoff between error and fairness; the least disparate models achieve error disparity near 0 and population error slightly under 0.29, and the most disparate models accept error disparity ${\sim}$0.05 to achieve population error near 0.25. The baseline and MSE proxies induce similar tradeoffs but are unable to induce downstream models with error disparity lower than  ${\sim}$0.01. This indicates success of the $\Hs$-proxy's multi-accuracy constraints in enforcing downstream accuracy parity -- its least disparate model achieves error disparity near 0. Out of sample, the performance of each model is nearly identical to its performance in sample. 

\begin{figure*}[h]
    \centering
    \includegraphics[width=0.7\textwidth]{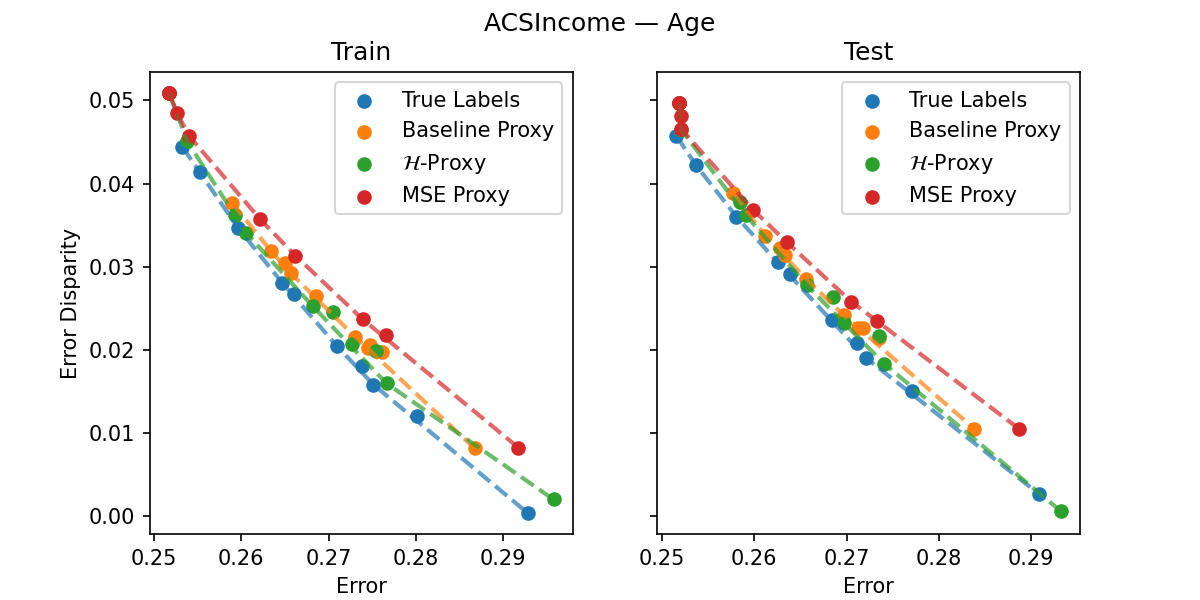}
        \captionsetup[subfigure]{labelformat=empty, belowskip=0pt}
    \caption{Proxy results on the ACSIncome dataset with age as sensitive feature}
    \label{fig:acs_income_age}
\end{figure*}

\subsubsection{ACS-Income-Sex}
In Fig.~\ref{fig:acs_income_sex} we observe that the downstream performance of the models trained on the $\Hs$-proxy is quite close to that of the the true labels, although the minimum disparity it achieves is greater by about 0.007. This indicates that the $\Hs$-proxy is not a perfect substitute for the true labels, although it is still quite good. As we relax fairness constraints, this disparity in accuracy gap between the models trained on the $\Hs$-proxy and those trained on the true labels shrinks, and the left endpoints of both curves are nearly identical in terms of error and disparity. Of the three proxies, the baseline proxy achieves the lowest error disparity of ${\sim}$0.0025 on the training data, but it comes at the cost of significantly lower accuracy. At the same levels of disparity, the model trained on the $\Hs$-proxy is more accurate than the model trained on the baseline by ${\sim}$0.01, and for error disparity values greater than ${\sim}$0.0075, the $\Hs$-proxy Pareto dominates the baseline in terms of both accuracy and  accuracy gap. Generalization is quite good for all models: the shape of each curve and the range of values spanned is consistent in and out of sample, although nearly all models are slightly more disparate out-of-sample. \footnote{The one exception to this is the rightmost point on the curve of models corresponding to the baseline proxy, which achieves slightly lower error disparity out-of-sample. However, given that the other models trained on the baseline proxy have worse out-of-sample performance than in-sample, we suspect this empirical improvement is simply noise that, by luck, worked in the proxy's favor.}

\begin{figure*}[h]
    \centering
    \includegraphics[width=0.7\textwidth]{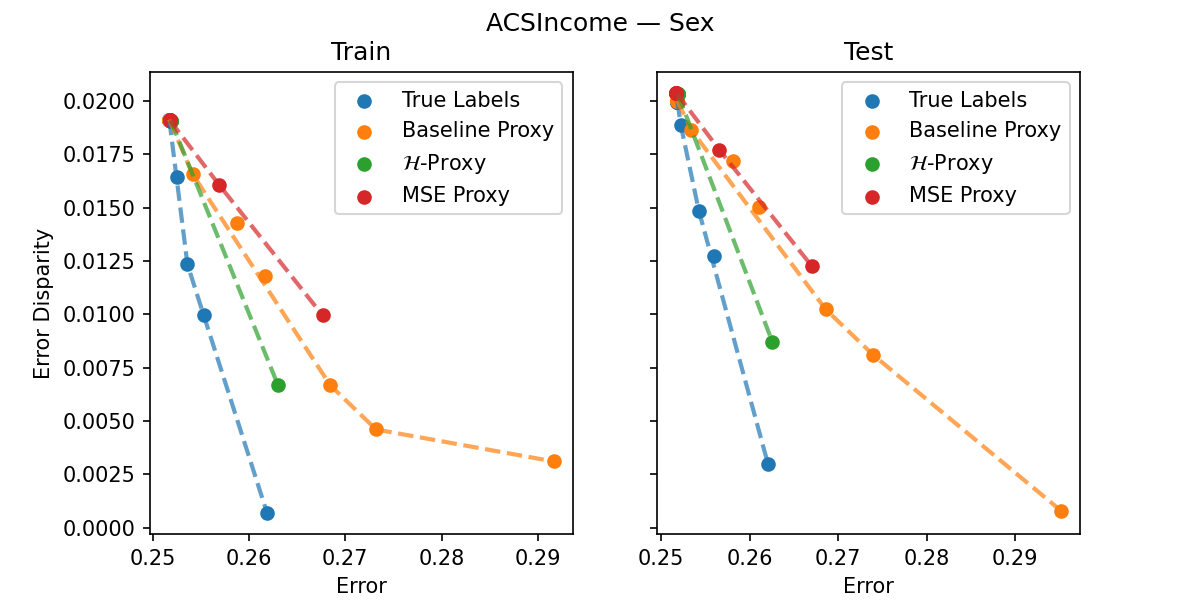}
        \captionsetup[subfigure]{labelformat=empty, belowskip=0pt}
    \caption{Proxy results on the ACSIncome dataset with sex as sensitive feature}
    \label{fig:acs_income_sex}
\end{figure*}

\section{Conclusion and Discussion}
We have shown that it is possible to efficiently train \textit{proxies} that can stand in for missing sensitive features to effectively train downstream classifiers subject to a variety of demographic fairness constraints. We caution however that proxies --- even when well trained --- should continue to be viewed as a second best solution, to be used only when sensitive features are impossible to collect. Our theoretical and empirical results demonstrate that proxies trained using our methods can stand in as near perfect substitutes for sensitive features in downstream training tasks, but these results crucially depend on the assumption that the data that the Proxy Learner uses to train its proxy is distributed identically to the data that the Downstream Learner uses, and has labels from the same problem distribution. In real applications, either of these assumptions can fail (or can become false due to distribution shift, even if they are true at the moment that the proxy is trained). A risk of relying on proxies is that the Learner might be blind to these failures. Without other guardrails, proxies could also be used to explicitly engage in discrimination, and so should be used only in the context of enforcing and auditing fairness constraints.

\bibliographystyle{plainnat}
\bibliography{main_arxiv_facct}


\newpage
\appendix
\section{A More General Framework}

In this section we first define a general family of fairness notions that captures equalized error fairness (discussed in the body of the paper), equalized false positive (negative) rate fairness, and statistical parity, to name a few. We will then give a definition of a proxy with respect to this general family of fairness notions. All our techniques used to derive algorithms and generalization guarantees for the specific case of equalized error fairness discussed in the body of paper can be extended to work for any other fairness notion that meets our definition in this section.

\begin{definition}[A General Family of Fairness Notions]\label{def:fair-notion}
We say a fairness notion is defined by $(E_1, E_2)$ where $E_1: \Hs \to 2^{\X \times \Y}$ and $E_2 \subseteq \X \times \Y$, if it has the following form: $h \in \Hs$ is fair if for all pairs of groups $k_1, k_2 \in [K]$,
\[
\Pr \left[ (x,y) \in E_1 (h) \, \vert \, (x,y) \in E_2, z_{k_1} = 1 \right] \approx \Pr \left[ (x,y) \in E_1 (h) \, \vert \, (x,y) \in E_2, z_{k_2} = 1 \right]
\]
\end{definition}

This is a very general class of fairness constraints that captures (for example) statistical parity (by taking $E_1(h) = \{ (x,y): h(x) = 1\}, E_2 = \X \times \Y$), equalized errors (by taking $E_1(h) = \{ (x,y): h(x) \neq y\}, E_2 = \X \times \Y$), and equalized false positive rates (by taking $E_1(h) = \{ (x,y): h(x) = 1\}, E_2 = \{ (x,y): y = 0 \}$), and many others. We first make the following simple, yet important, observation that will allow us to come up with a well-defined notion of a proxy.
\begin{clm}
Given a fairness notion defined by $(E_1, E_2)$, for every $k \in [K]$, we have
\begin{equation}\label{eq:observation0}
    \Pr \left[ (x,y) \in E_1 (h) \, \vert \,(x,y) \in E_2, z_k =1 \right] = \frac{\E \left[ z_k \1 \left[ (x,y) \in E_1 (h) \cap E_2 \right] \right]}{\E \left[  z_k \1 \left[ (x,y) \in E_2 \right] \right]}
\end{equation}
\end{clm}
\begin{proof}
We have
\begin{align*}
    \Pr \left[ (x,y) \in E_1 (h) \, \vert \,(x,y) \in E_2, z_k =1 \right] &= \frac{\Pr \left[ z_k = 1, (x,y) \in E_1 (h) \cap E_2 \right]}{\Pr \left[ z_k =1,  (x,y) \in E_2\right]} \\
    & = \frac{\E \left[ \1 \left[ z_k = 1 \right] \1 \left[ (x,y) \in E_1 (h) \cap E_2 \right]\right] }{\E \left[ \1 \left[ z_k = 1 \right] 1 \left[ (x,y) \in E_2 \right] \right]} \\
    &= \frac{\E \left[ z_k \1 \left[ (x,y) \in E_1 (h) \cap E_2 \right] \right]}{\E \left[  z_k \1 \left[ (x,y) \in E_2 \right] \right]}
\end{align*}
\end{proof}

Observe that the expression on the right hand side of Equation~\eqref{eq:observation0} could be evaluated even if the demographic labels $z$ were real valued rather than binary valued. We exploit this to be able to evaluate these fairness constraints with our real valued proxies $\hat z$. Observe that if we have a proxy $\hat{z}$, such that for a particular classifier $h \in \Hs$:
\begin{equation}\label{eq:condition}
\forall k \in [K]: \quad
\frac{\E \left[ z_k \1 \left[ (x,y) \in E_1 (h) \cap E_2 \right] \right]}{\E \left[  z_k \1 \left[ (x,y) \in E_2 \right] \right]} \approx \frac{\E \left[ \hat{z}_k (x) \1 \left[ (x,y) \in E_1 (h) \cap E_2 \right] \right]}{\E \left[  \hat{z}_k (x) \1 \left[ (x,y) \in E_2 \right] \right]}
\end{equation}
Then if $h$ satisfies proxy fairness constraints defined by the proxy $\hat{z}$, i.e., constraints of the form:
\begin{equation}\label{eq:proxy_constraints}
\forall k_1, k_2 \in [K]: \quad
\frac{\E \left[ \hat{z}_{k_1} (x) \1 \left[ (x,y) \in E_1 (h) \cap E_2 \right] \right]}{\E \left[  \hat{z}_{k_1} (x) \1 \left[ (x,y) \in E_2 \right] \right]} \approx \frac{\E \left[ \hat{z}_{k_2} (x) \1 \left[ (x,y) \in E_1 (h) \cap E_2 \right] \right]}{\E \left[  \hat{z}_{k_2} (x) \1 \left[ (x,y) \in E_2 \right] \right]}
\end{equation}
it will also satisfy the original fairness constraints with respect to the real demographic groups $z$, i.e., it will satisfy the constraints of Definition~\ref{def:fair-notion}. If the condition of Equation~\eqref{eq:condition} is satisfied for \emph{every} $h \in \Hs$, then the proxy fairness constraints (Equation~\eqref{eq:proxy_constraints}) can without loss be used to \emph{optimize} over all fair classifiers in $\Hs$. With this idea in mind, we can formally define a proxy. We will consider two different settings: 1) when there exists a distribution over $\F$ and we want our guarantee to hold with high probability over the draw of $f$ from this distribution 2) when we want our guarantee to hold for every $f \in \F$.

\begin{definition}[Proxy]\label{def:proxy}
Fix a distribution $\Ps$ over the $(\X \times \Z)$ domain, a distribution $\Qs$ over $\F$, and a fairness notion defined by $(E_1, E_2)$. We say $\hat{z}$ is an $(\alpha, \beta)$-proxy for $z$ with respect to $(E_1, E_2 , \Ps, \Qs)$, if with probability $1-\beta$ over the draw of $f \sim \Qs$: for all classifiers $h \in \Hs$, and all groups $k \in [K]$,
\[
\left\vert \frac{\E_\Ps \left[ z_k \1 \left[ (x,f(x,z)) \in E_1 (h) \cap E_2 \right] \right]}{\E_\Ps \left[  z_k \1 \left[ (x,f(x,z)) \in E_2 \right] \right]} - \frac{\E_\Ps \left[ \hat{z}_k (x) \1 \left[ (x,f(x,z)) \in E_1 (h) \cap E_2 \right] \right]}{\E_\Ps \left[  \hat{z}_k (x) \1 \left[ (x,f(x,z)) \in E_2 \right] \right]} \right\vert \le \alpha
\]
If the above condition holds for every $f \in \F$, we say $\hat{z}$ is an $\alpha$-proxy with respect to $(E_1, E_2, \Ps)$.
\end{definition}

In the body of the paper, we showed the existence of $0$-proxies for the case of equalized error fairness, under the assumption that $z$ and $y$ are independent conditioned on $x$. Similar derivations apply for other notions of fairness like equalized false positive and negative rates. Here we show the existence of a proxy for statistical parity fairness, which requires no assumptions at all on the data distribution.

\begin{clm}[Existence of a Proxy]
For statistical parity fairness: $E_1(h) = \{ (x,y) : h(x) = 1 \}$ and $E_2 = \X \times \Y$, and for any distribution $\Ps$ over $\X \times \Z$, $\hat{z}(x) = \E \left[z \,| \, x \right]$ is an $\alpha$-proxy with respect to $(E_1, E_2, \Ps)$, for $\alpha = 0$.
\end{clm}

\begin{proof}
Fix $f \in \F$, $h \in \Hs$, and $k \in [K]$. Note that
\[
\E_{(x,z) \sim \Ps} \left[  \hat{z}_k (x) \right] = \E_{x \sim \Ps_\X} \left[  \hat{z}_k (x) \right] = \E_{x \sim \Ps_\X} \left[  \E \left[z_k \, | \, x \right] \right] = \E_{(x,z)} [z_k ]
\]
Also,
\begin{align*}
    \E_{(x,z) \sim \Ps} \left[ \hat{z}_k (x) \1 \left[ (x,f(x,z)) \in E_1 (h) \right] \right] &= \E_{x \sim \Ps_\X} \left[ \E \left[z_k \, | \, x \right] \1 \left[ h(x) = 1 \right] \right] \\
    &= \E_{x \sim \Ps_\X} \left[ \E \left[z_k \1 \left[ h(x) = 1 \right] \, | \, x \right] \right] \\
    &= \E_{(x,z)} \left[z_k \1 \left[ h(x) = 1 \right] \right] \\
    &= \E_{(x,z)} \left[z_k \1 \left[ (x,f(x,z)) \in E_1 (h) \right] \right]
\end{align*}
completing the proof.
\end{proof}

\paragraph{Modelling the Proxy Learner (PL).}
In the body of the paper we derived the problem of the proxy learner for the special case of equalized error fairness. Similarly, and more generally, we can write down the corresponding optimization problem of the PL for a fairness notion defined by some $(E_1, E_2)$. Recall that we assume the PL has access to a data set, which consists of two components: 1) $S = \{ (x_i, z_i) \}_{i=1}^n$ which is a sample of $n$ individuals from $\X \times \Z$ represented by their non-sensitive features and sensitive attributes. 2) $F = \{ f_j \}_{j=1}^m$ which is a sample of $m$ labeling functions (or \emph{learning tasks}) taken from $\F$. The PL does not observe the actual functions $f_j \in \F$ but instead observes the realized labels of functions in $F$ on our data set of individuals $S$: $Y = \{ y_{ij} = f_j (x_i, z_i) \}_{i,j}$. Given these data sets, the empirical problem of the PL can be formulated as: for all $k$, solve
\begin{mini}
{\hat{z}_k \in \Gs}{\frac{1}{n}\sum_{i=1}^{n} \left( z_{ik} - \hat{z}_k (x_i) \right)^2}{}{}
\addConstraint{\forall j \in [m], \, \forall h \in \Hs}{}{}
\addConstraint{\frac{\sum_{i=1}^n z_{ik} \1 \left[ (x_i,y_{ij} ) \in E_1 (h) \cap E_2 \right]}{\sum_{i=1}^n z_{ik}  \1 \left[ (x_i,y_{ij} ) \in E_2 \right] } = \frac{\sum_{i=1}^n \hat{z}_{k} (x_i) \1 \left[ (x_i,y_{ij} ) \in E_1 (h) \cap E_2 \right]}{\sum_{i=1}^n \hat{z}_{k} (x_i) \1 \left[ (x_i,y_{ij} ) \in E_2 \right]}}{}{ \ }
\label{eqn:program1_general}
\end{mini}
which gives us a decomposition of learning $\hat{z} = (\hat{z}_1, \ldots, \hat{z}_K)$ into learning each component $\hat{z}_k$ separately.

\ifdraft
\section{Adjustments for Statistical Parity}
When we want to use statistical parity as our fairness notion, we can just create a dummy label $y_{i}=0$ for all $i$ and only have one label per sample and then run the Algorithms~\ref{alg:ftpl}, \ref{alg:functionclass}, and \ref{alg:EqualErrors} as normal.

Alternately, because the condition of being a good proxy with respect to $\Hs$ for statistical parity (essentially) maps on to a multi-accuracy constraint with respect to $\Hs$, we can obtain a good proxy by running the algorithms of \cite{multiaccuracy,multiaccuracy2}. This is because unlike multiaccuracy for \emph{error regions} as we study in the body of the paper (membership in which requires evaluating whether $h(x) \neq y$), whether or not $h(x) = 1$ can be evaluated at test time for any $h \in \Hs$, since it does not make reference to the unknown label $y$. The advantage of this approach is that the algorithms of \cite{multiaccuracy,multiaccuracy2} optimize over linear combinations of functions $h \in \Hs$, which they prove by construction always contain a feasible solution. This removes the need to assume that the class we optimize over contains a good proxy. 
\fi

\ifdraft
\section{Missing Material Section~\ref{sec:linear}}
\label{sec:missing_linear}
\begin{proof}
We begin by upper bounding the $L_2$ norm of the gradient 
\begin{align}
||\nabla_{\theta} L(\lambda,\theta_k)|| &= ||\frac{2}{n}\sum_{i=1}^{n} \nabla_{\theta_k} \hat{z}_k(x_i;\theta_k)\cdot \left( \hat{z}_k(x_i;\theta_k)- z_{ik} \right)\\
&+ \lambda_0 \frac{\sum_{i=1}^n \nabla_{\theta_k}\hat{z}_{k} (x_i; \theta_k)}{\sum_{i=1}^n z_{ik}} +\sum_{h \in \Hs, j \in [m]} \lambda_{h,j} \sum_{i=1}^n \nabla_{\theta_k} \hat{z}_{k} (x_i; \theta_k) \1 \left[ h(x_i) \neq y_{ij} \right]||\\
& \leq ||\frac{2}{n}\sum_{i=1}^{n} \nabla_{\theta_k} \hat{z}_k(x_i;\theta_k)\cdot \hat{z}_k(x_i;\theta_k)\\
&+ \lambda_0 \frac{\sum_{i=1}^n \nabla_{\theta_k}\hat{z}_{k} (x_i; \theta_k)}{\sum_{i=1}^n z_{ik}} +\sum_{h \in \Hs, j \in [m]} \lambda_{h,j} \sum_{i=1}^n \nabla_{\theta_k} \hat{z}_{k} (x_i; \theta_k) \1 \left[ h(x_i) \neq y_{ij} \right]||\\
& \leq d \left( 2MB + nCB \left( \frac{1}{\sum_{i=1}^{n} z_{ik}}\right)\right)
\end{align}

Applying Theorem~\ref{thm:gdregret}, with appropriate choice of $\eta$ (derived below), we bound the Proxy Learner's average regret over $T$ rounds:

\begin{align}
\frac{R_T}{T} &\leq \frac{\sup_{\theta, \theta' \in \Theta} \left\Vert \theta - \theta' \right\Vert_2 ||\nabla L||}{\sqrt{T}} \leq \frac{d \cdot D \left( 2MB + nCB \left( \frac{1}{\sum_{i=1}^{n} z_{ik}}\right)\right)}{\sqrt{T}}
\end{align}

Setting $T \geq \frac{d^2 D^2 \left( 2MB + nCB \left( \frac{1}{\sum_{i=1}^{n} z_{ik}}\right)\right)^2}{\epsilon^2}$ and $\eta = \frac{D}{d \left( 2MB + nCB \left( \frac{1}{\sum_{i=1}^{n} z_{ik}}\right)\right)\sqrt{T}}$, we have that $\frac{R_T}{T} \leq \epsilon$. Because the Auditor plays a no-regret strategy, we apply Theorem~\ref{thm:noregret} to assert that the mixed strategy $\hat{p}$ (and corresponding empirical distributions over $\lambda$ visited by the Auditor) form an $\epsilon$-approximate equilibrium.

Now we will show that an approximate solution to the game corresponds to an approximate solution to the Program~\eqref{eqn:programAlgos}. First, we consider some $\hat{z}(\cdot; \theta^*)$ that is a feasible solution to Program~\eqref{eqn:programAlgos}, and $\hat{\lambda}$ is an $\epsilon$-approximate minimax solution to the Lagrangian game specified in Equation~\eqref{eqn:game}. 

Now we will analyze the case in which we have a solution $\hat{z}(\cdot;\theta)$ that is an $\epsilon$-approximate solution to the Lagrangian game but is not a feasible solution for the constrained optimization problem~\eqref{eqn:programAlgos}. This must be because some constraint is violated. Let $\xi$ be the magnitude of the violated constraint, and let $\lambda$ be such that the dual variable for the violated constraint is set to $C$. By the definition of an $\epsilon$-approximate minimax solution, we know that $$L(\theta,\hat{\lambda}) \geq L(\theta, \lambda) \geq \frac{1}{n}\sum_{i=1}^{n} \left( z_{ik} - \hat{z}_k (x_i; \theta_k) \right)^2 + C\xi- \epsilon$$

Then,

\[
\frac{1}{n}\sum_{i=1}^{n} \left( z_{ik} - \hat{z}_k (x_i; \theta_k) \right)^2 + C\xi \leq L(\theta,\hat{\lambda}) + \epsilon \leq L(\theta^*,\hat{\lambda}) + 2\epsilon \leq \frac{1}{n}\sum_{i=1}^{n} \left( z_{ik} - \hat{z}_k (x_i; \theta^*_k) \right)^2 + 2 \epsilon
\]

Because 
\[
\frac{1}{n}\sum_{i=1}^{n} \left( z_{ik} - \hat{z}_k (x_i; \theta_k) \right)^2  \leq M^2
\]

$C\xi\geq M^2 + 2\epsilon$. Therefore, the maximum constraint violation is no more than $\frac{M^2 + 2\epsilon}{C}$. Setting $C = \frac{M^2 + 2\epsilon}{\epsilon}$,  we that $\hat{z}(\cdot; \hat{\theta})$ does not violate any constraint by more than $\epsilon$.

The last step is to transition from an $\epsilon$-approximate solution to the linear program to an $\alpha$-proxy. Setting $\epsilon = \frac{\alpha \sum_{i=1}^n z_{ik}}{1+nM}$, as outlined at the beginning of the section, gives us an $\alpha$-proxy. Therefore, the number of rounds $T$ that we require is 
\[
T \geq \frac{d^2 D^2 \left( 2MB + nCB \left( \frac{1}{\sum_{i=1}^{n} z_{ik}}\right)\right)^2}{\left( \frac{\alpha \sum_{i=1}^n z_{ik}}{1+nM} \right)^2}
\]
with dual variable upper bound
\[
C \geq \frac{M^2 + 2\frac{\alpha \sum_{i=1}^n z_{ik}}{1+nM}}{\frac{\alpha \sum_{i=1}^n z_{ik}}{1+nM}}
\]

\end{proof}
\fi

\section{Missing Material Section~\ref{sec:general}}

\subsection{Casting the PL's problem as a Linear Program}
In this section we provide more details as to how the problem of PL can be cast as a linear program with \emph{finitely many} variables and constraints. The high level idea is that on a given finite data set $S$, we can reduce the sets $\Hs$ and $\Gs$ to ones that have finitely many elements. Here we will use definitions and tools (such as covering sets and numbers) from Subsection~\ref{subsec:tools}.

In the first step, note that we can reduce the entire $\Hs$ to $\Hs (S) = \{ (h(x_1), h(x_2), \ldots, h(x_n)) : h \in \Hs\}$ which includes all possible labelings induced on $S$ by any function in $\Hs$. We then have that, as long as $\Hs$ has finite VC dimension, $|\Hs(S)|$ is finite. In particular
\begin{lemma}[Sauer's Lemma]
If $\Hs$ has finite VC dimension $d$, then for any data set $S$ of size $n$, we have $|\Hs(S)| = O (n^d)$.
\end{lemma}
We therefore get a linear program with finitely many constraints. Second, we can apply similar techniques to reduce the entire function class $\Gs$ to a set that has finitely many functions, implying we will get a linear program with finitely many variables. Since functions in $\Gs$ are real-valued, we take $\Gs_{\epsilon} (S)$ to be an $\epsilon$-cover of $\Gs$ with respect to $S$ and the $d_1$ metric (we use $\mathcal{N} \left( \epsilon, \Gs (S), d_1 \right)$ notation in Subsection \ref{subsec:tools}), for some $\epsilon$ appropriately chosen later on. What this implies is that for any $\hat{z}_k \in \Gs$, there exists $\tilde{z}_k \in \Gs_\epsilon (S)$, that satisfies
\begin{equation}\label{eq:bl}
\frac{1}{n} \sum_{i=1}^n \left\vert \hat{z}_k (x_i) -  \tilde{z}_k(x_i) \right\vert \le  \epsilon
\end{equation}
on the data set $S$, and furthermore, as long as $\Gs$ has finite pseudo-dimension, we know by Lemma~\ref{lem:N1} that $\Gs_\epsilon (S)$ is finite. In particular, if $\Gs$ has pseudo dimension $d$, then $|\Gs_\epsilon (S)| = O (\epsilon^{-d})$. Now given that our target proxy approximation parameter is $\alpha$, and using the guarantee of Equation~\eqref{eq:bl}, we can choose $\epsilon = O(\alpha)$ and optimize over (distributions over) $\Gs_\epsilon (S)$ which will guarantee that we have a linear program with finitely many variables (which are probability weights over functions in $\Gs_\epsilon (S)$).

\subsection{Derivation of Linear Program}
In this section, we derive a relationship between the slack of an approximate solution to Program~\ref{eqn:programAlgos} and the level of approximation ($\alpha$), that we seek in a proxy. First, we rewrite Program~\ref{eqn:programAlgos} in the following way:
For some group $k$ and labeling $j$:
\[
\sum_{i=1}^n \hat{z}_{k} (x_i) \1 \left[ h(x_i) \neq y_{ij} \right] - \sum_{i=1}^n z_{ik} \1 \left[ h(x_i) \neq y_{ij} \right] = 0
\]
\[
\frac{\sum_{i=1}^n \hat{z}_{k} (x_i)}{\sum_{i=1}^n z_{ik}} - 1 = 0
\]
Then for any $\epsilon > 0$, we will derive an algorithm to find $\hat{z}_k$ such that the constraints are satisfied up to at most $\epsilon$ slack:
\begin{equation}\label{eq:eq1}
\left\vert \sum_{i=1}^n \hat{z}_{k} (x_i) \1 \left[ h(x_i) \neq y_{ij} \right] - \sum_{i=1}^n z_{ik} \1 \left[ h(x_i) \neq y_{ij} \right] \right\vert \le \epsilon
\end{equation}
\begin{equation}\label{eq:eq2}
\left\vert \frac{\sum_{i=1}^n \hat{z}_{k} (x_i)}{\sum_{i=1}^n z_{ik}} - 1 \right\vert \le \epsilon
\end{equation}
Note that the second inequality gives us the following multiplicative  guarantee, which will be useful later on:
\begin{equation}\label{eq:eq3}
(1-\epsilon) \sum_{i=1}^n z_{ik} \le \sum_{i=1}^n \hat{z}_{k} (x_i) \le (1+\epsilon)
\sum_{i=1}^n z_{ik}
\end{equation}

We use these transformations to show how an approximate solution to Program~\ref{eqn:programAlgos} corresponds to an approximate proxy, formalized in the Lemma statement below.

\begin{lemma}
\label{lemma:epsilon}
For a fixed data set $S$, $\hat{z}$ is an $\alpha$ proxy for $z$ if it is an $\epsilon$-approximate solution to Program~\ref{eqn:programAlgos}, with 
\[
\epsilon = \frac{\alpha \sum_{i=1}^n z_{ik}}{1+nM}
\]

\end{lemma}

\begin{proof}
\[
\left\vert \frac{\sum_{i=1}^n z_{ik} \1 \left[ h(x_i) \neq y_{ij} \right]}{\sum_{i=1}^n z_{ik}} - \frac{\sum_{i=1}^n \hat{z}_{k} (x_i) \1 \left[ h(x_i) \neq y_{ij} \right]}{\sum_{i=1}^n \hat{z}_{k} (x_i)} \right\vert
\]
If the first term is larger than the second term, we can bound the difference as follows
\begin{align*}
    &\frac{\sum_{i=1}^n z_{ik} \1 \left[ h(x_i) \neq y_{ij} \right]}{\sum_{i=1}^n z_{ik}} - \frac{\sum_{i=1}^n \hat{z}_{k} (x_i) \1 \left[ h(x_i) \neq y_{ij} \right]}{\sum_{i=1}^n \hat{z}_{k} (x_i)} \\
    &\le \frac{\sum_{i=1}^n z_{ik} \1 \left[ h(x_i) \neq y_{ij} \right]}{\sum_{i=1}^n z_{ik}} - \frac{\sum_{i=1}^n \hat{z}_{k} (x_i) \1 \left[ h(x_i) \neq y_{ij} \right]}{(1+\epsilon)\sum_{i=1}^n z_{ik}} \\
    &= \frac{1}{\sum_{i=1}^n z_{ik}} \left( \sum_{i=1}^n z_{ik} \1 \left[ h(x_i) \neq y_{ij} \right] - \frac{\sum_{i=1}^n \hat{z}_{k} (x_i) \1 \left[ h(x_i) \neq y_{ij} \right]}{(1+\epsilon)} \right) \\
    &= \frac{1}{\sum_{i=1}^n z_{ik}} \left( \sum_{i=1}^n z_{ik} \1 \left[ h(x_i) \neq y_{ij} \right] - \sum_{i=1}^n \hat{z}_{k} (x_i) \1 \left[ h(x_i) \neq y_{ij} \right] + \frac{\epsilon}{1+\epsilon} \sum_{i=1}^n \hat{z}_{k} (x_i) \1 \left[ h(x_i) \neq y_{ij} \right] \right) \\
    &\le \frac{1}{\sum_{i=1}^n z_{ik}} \left( \epsilon + \epsilon nM\right) \\
    &= \frac{1+nM}{\sum_{i=1}^n z_{ik}} \cdot \epsilon
\end{align*}
The first inequality follows from Equation~\eqref{eq:eq3}. The second inequality follows from Equation~\eqref{eq:eq1}, and the fact that $\hat{z}_{k} (x_i) \le M$. The same bound holds if the second term is bigger than the first term. So we have that,
\[
\left\vert \frac{\sum_{i=1}^n z_{ik} \1 \left[ h(x_i) \neq y_{ij} \right]}{\sum_{i=1}^n z_{ik}} - \frac{\sum_{i=1}^n \hat{z}_{k} (x_i) \1 \left[ h(x_i) \neq y_{ij} \right]}{\sum_{i=1}^n \hat{z}_{k} (x_i)} \right\vert \le \frac{1+nM}{\sum_{i=1}^n z_{ik}} \cdot \epsilon
\]
Therefore, in order to produce an $\alpha$-proxy, we need to find a $\hat z$ satisfying constraints \eqref{eq:eq1} and \eqref{eq:eq2} up to slack:
\[
\epsilon = \frac{\alpha \sum_{i=1}^n z_{ik}}{1+nM}
\]
\end{proof}

\subsection{Proof of Theorem~\ref{thm:ftplProxy}}

\begin{lemma}
\label{lemma:learner_regret} 
Let $T$ be the time horizon for Algorithm~\ref{alg:ftpl}. Let $P_k^1 \times Q_k^1,...,P_k^t \times Q_k^t$ be the sequence of distributions maintained by the Auditor's FTPL algorithm and $\hat{z}_k^1,...,\hat{z}_k^T$ be the sequence of plays by the Learner. Then
\[
\sum_{t=1}^{T} \mathbb{E}_{\lambda \sim P_k^t \times Q_k^t} [L(\hat{z}_k^t, \lambda)] - \min_{\lambda} \sum_{i=1}^{T} L(\hat{z}_k^t,\lambda) \leq 2\left(n^{3/2}CM + C_0\frac{nM}{\sum_{i=1}^n z_{ik}}\right)\sqrt{T}
\]
\end{lemma}

\begin{proof} 
We appeal to Theorem~\ref{thm:ftpl} to bound the regret of the Auditor, and set the parameter $d$ in the theorem to $n$ (the dimension of our linear program). To do so, we examine the maximum absolute values over the coordinates of the two loss vectors. For any $\hat{z}(x)_k \in \Gs$, the absolute value of the $i$-th coordinate of $LC_j(\hat{z}(x)_k)$ is bounded by:

\[
||\lambda||\left(z_{ik} - \hat{z}_k(x_i)\right) \left(\1 \left[y_{ij}=0\right] - \1\left[y_{ij}=1\right]\right) \leq CM
\]

while the absolute value of $\lambda_0 (\frac{\sum_{i=1}^n \hat{z}_{k} (x_i)}{\sum_{i=1}^n z_{ik}} - 1)$ is bounded by $C_0\frac{nM}{\sum_{i=1}^n z_{ik}}$

Choosing $\eta =\frac{1}{CM}\sqrt{\frac{1}{nT}}$ and $\eta' = \frac{1}{C_0\frac{nM}{\sum_{i=1}^n z_{ik}}}\sqrt{\frac{1}{T}}$ causes the Auditor's regret to be bounded by $2(n^{3/2}CM + C_0\frac{nM}{\sum_{i=1}^n z_{ik}})\sqrt{T}$.
\end{proof}

\begin{lemma}
\label{lemma:uc}
Fix any $\xi,\delta \in (0,1)$. Let $\lambda^1,...,\lambda^W$ be $W$ i.i.d draws from $P_k \times Q_k$, and $P_k \times \hat{Q}_k$ be the empirical distribution over the realized sample. Then with probability at least $1-\frac{\delta}{K}$ over the random draws of $\lambda$'s, the following holds:
\[
\max_{\hat{z}(x)_k \in \Gs} |\mathbb{E}_{\lambda \sim P_k^t \times \hat{Q}_k^t} L(\hat{z}(x)_k, \lambda) - \mathbb{E}_{\lambda \sim P_k^t \times Q_k^t} L(\hat{z}(x)_k, \lambda)| \leq \xi
\]

as long as $W \geq \frac{n^2C^2M^2 \log{\frac{K}{2\delta}}}{2\xi^2}$.

\end{lemma}

\begin{proof}

\begin{align}
&\mathbb{E}_{\lambda \sim P_k^t \times \hat{Q}_k^t} L(\hat{z}(x)_k,\lambda) - \mathbb{E}_{\lambda \sim D_k} L(\hat{z}_k,\lambda) \\
&= \mathbb{E}_{\lambda \sim P_k^t \times \hat{Q}_k^t} \left[ \lambda_{h,j}\sum_{i=1}^n \left(\hat{z}_{k} (x_i) - z_{ik}\right) \1 \left[ h(x_i) \neq y_{ij} \right]\right] + \mathbb{E}_{\lambda \sim P_k^t \times Q_k^t}   \left[\lambda_{h,j}\sum_{i=1}^n \left(\hat{z}_{k} (x_i) - z_{ik} \1 \left[ h(x_i) \neq y_{ij} \right]\right)\right]\\
&= \sum_{i=1}^n \left(\hat{z}_{k} (x_i) - z_{ik}\right)\left(\mathbb{E}_{\lambda \sim P_k^t \times \hat{Q}_k^t} \left[ \lambda_{h,j}  \1 \left[ h(x_i) \neq y_{ij}\right] \right] - \mathbb{E}_{\lambda \sim P_k^t \times Q_k^t} \lambda_{h,j} \1 \left[ h(x_i) \neq y_{ij}\right]\right) \\
&\leq  nCM \left(\frac{1}{W} \sum_{\lambda' \in \hat{Q}_k^t}  \1 \left[ h(x_i) \neq y_{ij}\right] - \mathbb{E}_{\lambda' \sim Q_k^t} \1 \left[ h(x_i) \neq y_{ij}\right] \right) \\
\end{align}

We apply the additive Chernoff-Hoeffding Bound (Theorem~\ref{thm:add-chernoff}) to solve for $W$:

\begin{align}
&\mathbb{P} \left[ | nCM \left(\frac{1}{W} \sum_{\lambda' \in \hat{Q}_k^t} \sum_{h \in \Hs, j \in [m]} \1 \left[\lambda_{h,j} \neq 0\right] \1 \left[ h(x_i) \neq y_{ij}\right] - \mathbb{E}_{\lambda' \sim Q_k^t} \sum_{h \in \Hs, j \in [m]} \1 \left[\lambda_{h,j} \neq 0\right] \1 \left[ h(x_i) \neq y_{ij}\right] \right)|\geq \xi \right] \\
&\leq \mathbb{P} \left[ |nCM \left( \frac{1}{W} \sum_{\lambda' \in \hat{Q}_k^t}  \sum_{h \in \Hs, j \in [m]} \1 \left[\lambda_{h,j} \neq 0\right] \1 \left[ h(x_i) \neq y_{ij}\right] - \mathbb{E}_{\lambda' \sim Q_k^t} \sum_{h \in \Hs, j \in [m]} \1 \left[\lambda_{h,j} \neq 0\right] \1 \left[ h(x_i) \neq y_{ij}\right] \right)|\geq \xi \right] \\
&= \mathbb{P} \left[ |\frac{1}{W} \sum_{\lambda' \in \hat{Q}_k^t} \sum_{h \in \Hs, j \in [m]} \1 \left[\lambda_{h,j} \neq 0\right]  \1 \left[ h(x_i) \neq y_{ij}\right] - \mathbb{E}_{\lambda' \times Q_k^t} \sum_{h \in \Hs, j \in [m]} \1 \left[\lambda_{h,j} \neq 0\right] \1 \left[ h(x_i) \neq y_{ij}\right]|\geq \frac{\xi}{nCM} \right] \\
&\leq 2 \exp{\left(-2 \left(\frac{\xi}{nCM}\right)^2 W\right)}
\end{align}

If we want this probability to be no more than $\delta$ for \textit{any} $k$, we apply a union bound and see that we need $W\geq \frac{n^2C^2M^2\log{\frac{\delta}{2}}}{-2\xi^2}$
\end{proof}

\begin{lemma}
\label{lemma:auditor_regret}
Let $T$ be the time horizon for Algorithm~\ref{alg:ftpl}. Let $D_k^1,...,P_k^t \times Q_k^t$ be the sequence of distributions maintained by the Auditor's FTPL algorithm. For each $P_k^t \times Q_k^t$, let $P_k^t \times \hat{Q}_k^t$ be the empirical distribution over $W$ i.i.d draws from $P_k^t \times Q_k^t$. Let $\hat{z}_k^1,...,\hat{z}_k^T$ be the Learner's best responses against $P_k^1 \times \hat{Q}_k^1,...,P_k^t \times \hat{Q}_k^t$. Then, with probability $1-\delta$,

\[\max_{\hat{z}_k \in \Gs} \sum_{t=1}^{T} \mathbb{E}_{\lambda \sim P_k^t \times Q_k^t} [L(\hat{z}_k, \lambda)] - \sum_{t=1}^{T} \mathbb{E}_{\lambda \sim P_k^t \times Q_k^t} [L(\hat{z}_k^t, \lambda)] \leq T \sqrt{\frac{n^2C^2M^2 \log{\frac{T}{2\delta}}}{W}}
\]
\end{lemma}

\begin{proof}
Let $\gamma_{L,k}^t$ be defined as:

\[
\gamma_{L,k}^t = \max_{\hat{z}_k \in \Gs} |\mathbb{E}_{\lambda \sim P_k^t \times \hat{Q}_k^t} L(\hat{z}_k, \lambda) - \mathbb{E}_{\sim P_k^t \times Q_k^t} L(\hat{z}_k, \lambda)|
\]

From Lemma~\ref{lemma:uc} and applying a union bound across $T$ steps, we have that with probability at least $1-\delta$, for all $t \in [T]$ and $k \in [K]$:
\[
\gamma_{L,k}^t \leq \sqrt{\frac{n^2C^2M^2 \log{\frac{TK}{2\delta}}}{W}}
\]
\end{proof}

\begin{thm}[$\alpha$-Proxy for $m$ labeling functions taken from $\F$]
Fix any $\alpha$, and $\delta$. Suppose $\Hs$ has finite VC dimension, and $\Gs$ has finite pseudo-dimension. Suppose $\Delta (\Gs)^K$ contains a $0$-proxy. Then given access to oracles $CSC(\Hs)$ and $ERM(\Gs)$, we have that with probability at least $1-\delta$, Algorithm~\ref{alg:ftpl} returns a distribution $\hat{p} \in \Delta (\Gs)^K$ such that $\hat{z} (x) \triangleq \E_{g \sim \hat{p}} \left[ g(x) \right] = \frac{1}{T} \sum_{t=1}^T \hat{z}_t (x)$ is an $\alpha$-proxy. 
\end{thm}

\begin{proof}
We have seen that the Auditor's average regret for the sequence $P_k^1 \times Q_k^1,...,P_k^T \times Q_k^t$ is bounded by:
\[
\gamma_{A,k} = \frac{1}{T} \sum_{t=1}^{T} \mathbb{E}_{\lambda \sim P_k^t \times Q_k^t} [L(\hat{z}_k^t, \lambda)] - \min_{h \in H} \sum_{i=1}^{T} L(\hat{z}_k^t,\lambda) \leq \frac{ 2\left(n^{3/2}CM + C_0\frac{nM}{\sum_{i=1}^n z_{ik}}\right)}{\sqrt{T}}
\]

The Learner's average regret, with probability $1-\delta$ is bounded by
\[
\gamma_{L,j} \leq \sqrt{\frac{n^2C^2M^2\log{\frac{TK}{2\delta}}}{W}}
\]

By Theorem~\ref{thm:noregret}, we know that the average play $(\bar{P}_k \times \bar{\hat{Q}}_k,\bar{\hat{z}}_k)$ forms an $(\gamma_{A,k} + \gamma_{L,k})$-approximate equilibrium. Then $\gamma_{A,k} + \gamma_{L,k} \leq \epsilon$ if we choose $W \geq \frac{n^2C^2M^2 \log(\frac{TK}{2\delta})}{\epsilon^2}$ 
and $T \geq \sqrt{ \frac{2\left(n^{3/2}CM + C_0\frac{nM}{\sum_{i=1}^n z_{ik}}\right)}{\epsilon}}$.

Now we will show that an approximate solution to the game corresponds to an approximate solution to the Program~\ref{eqn:program1}. 

Because we assume that there is a $0$-proxy in $\Gs$, we may consider some $\hat{z}_k^*$ that is a feasible solution to Program~(\ref{eqn:program1}), and $\hat{\lambda}$ is an $\epsilon$-approximate minimax solution to the Lagrangian game specified in Equation~\ref{eqn:game}. 

Next, consider a proxy $\hat{z}_k$ that is an $\epsilon$-approximate solution to the Lagrangian game but is not a feasible solution for the constrained optimization problem~\ref{eqn:program1}. This must be because the $\lambda_0$ constraint is violated, a $\lambda_{h,j}$ constraint is violated, or both. Let $\xi$ be the maximum magnitude of the violated constraints, and let $\lambda$ be such that the dual variable for the violated constraint is set to $C_0$ and $C$ respectively. By the definition of an $\epsilon$-approximate minimax solution, we know that 
\[
L(\hat{z}_k,\hat{\lambda}) \geq L(\hat{z}_k, \lambda) \geq \frac{1}{n}\sum_{i=1}^{n} \left( z_{ik} - \hat{z}_k (x_i) \right)^2 + (C_0 + C)\xi- \epsilon
\]

Then,

\[
\frac{1}{n}\sum_{i=1}^{n} \left( z_{ik} - \hat{z}_k (x_i) \right)^2 + (C_0 + C)\xi \leq L(\hat{z}_k,\hat{\lambda}) + \epsilon \leq L(\hat{z}^*,\hat{\lambda}) + 2\epsilon \leq \frac{1}{n}\sum_{i=1}^{n} \left( z_{ik} - \hat{z}^*_k (x_i) \right)^2 + 2 \epsilon
\]

Because 
\[
\frac{1}{n}\sum_{i=1}^{n} \left( z_{ik} - \hat{z}_k (x_i) \right)^2  \leq M^2
\]

$(C_0 + C) \xi\geq M^2 + 2\epsilon$. Therefore, the maximum constraint violation is no more than $\frac{M^2 + 2\epsilon}{C_0 + C}$. Setting $C_0 + C = \frac{M^2 + 2\epsilon}{\epsilon}$, we guarantee that $\hat{z}(x)_k$ does not violate any constraint by more than $\epsilon$.

The last step is to transition from an $\epsilon$-approximate solution to the linear program to an $\alpha$-proxy. Plugging in $\epsilon = \frac{\alpha \sum_{i=1}^n z_{ik}}{1+nM}$, we see that choosing $W \geq \frac{(1+nM)^2 n^2C^2M^2 \log(\frac{TK}{2\delta})}{ \left(\alpha \sum_{i=1}^n z_{ik} \right)^2}$, $T \geq \sqrt{ \frac{2(1+nM)n^{5/4}\left(CM + C_0\frac{nM}{\sum_{i=1}^n z_{ik}}\right)}{\alpha \sum_{i=1}^n z_{ik}}}$, and $C=C_0 \geq \frac{M^2(1+nM)}{2\alpha \sum_{i=1}^n z_{ik}} + 1$, Algorithm~\ref{alg:ftpl} produces an $\alpha$-proxy for the sensitive attribute $z$.

\end{proof}

\section{Missing Material Section~\ref{sec:functionclass}}

\begin{algorithm}
\SetAlgoLined
\KwIn{ Data set $\{x_i,y_{ij},z_i\}_{i=1}^{n} \forall j \in [m]$, target proxy parameter $\alpha$, target confidence parameter $\delta$, upper bound $M$ on proxy values, groups $k \in [K]$}
Set dual variable upper bounds: $C=C_0 = (M^2(1+nM)/2\alpha \sum_{i=1}^n z_{ik}) + 1$\;
Set iteration count: $T = \left\lceil \sqrt{2(1+nM)\left(n^{3/2}CM + C_0\frac{nM}{\sum_{i=1}^n z_{ik}}\right)/ \alpha\sum_{i=1}^n z_{ik}} \; \right\rceil$\;
Set sample count: $W = \left \lceil (1+nM)^2 n^2C^2M^2 \log(\frac{TK}{2\delta})/\left(\alpha \sum_{i=1}^n z_{ik} \right)^2 \; \right \rceil$\;
Set learning rates of FTPL: $\eta =\frac{1}{CM} \sqrt{\frac{1}{nT}}\;$, $\;\eta' = \frac{\sum_{i=1}^n z_{ik}}{C_0nM} \sqrt{\frac{1}{T}}$\;
\For{$k=1$ \textbf{to} $K$}{
Initialize $\hat{z}_k^0=\bar{0}$\;
\For{$t=1$ \textbf{to} $T$}{
Sample from the Auditor's FTPL distribution\;
\For{$w=1$ \textbf{to} $W$}{
Draw a random vector $\xi^w$ uniformly at random from $[0,1]^{n}$\;
Use the oracle $CSC(\Hs \oplus \F)$ to compute: \;
$g^{w,t} = \argmin_{g \in \Hs \oplus \F} - \sum_{t' < t}|\langle LC(\hat{z}_k^{t'}),g \rangle| + \frac{1}{\eta} \langle\xi^w,g\rangle$\;
Find sign of $\lambda^{w,t}_{g^{w,t}}$ : $q^{w,t} = 2\1 \left[ \langle LC(\hat{z}_k^{t}),g^{w,t}\rangle > 0 \right] - 1$\;
Let $\lambda'^{w,t}$ be defined as $\lambda^{w,t}_g= q^{w,t} \times C \1 \left[g = g^{w,t}\right]$\;
}
Let $\hat{Q}_k^t$ be the empirical distribution over $\lambda'^{w,t}$\;
Set distribution over $\lambda_0^t$ : $P_k^t = C_0\left(2Bern(p^t) - 1\right)$ where $p_k^t = \min(1,-\eta'(\frac{\sum_{i=1}^n \hat{z}_{k} (x_i)}{\sum_{i=1}^n z_{ik}} - 1)\1\left[\frac{\sum_{i=1}^n \hat{z}_{k} (x_i)}{\sum_{i=1}^n z_{ik}} -1 < 0 \right])$\;
The Learner best responds: $\hat{z}_k^t = \argmin_{\hat{z}_k \in \Gs} \mathbb{E}_{\lambda \sim P_k^t \times \hat{Q}_k^t} L(\hat{z}_k, \lambda)$ by calling $ERM(\Gs)$.
}
}
\KwOut{$\hat{p}$ = uniform distribution over $\{ \hat{z}^1$,...,$\hat{z}^T \}$}
\caption{Learning a Proxy for an Entire Function Class}
\label{alg:functionclass}
\end{algorithm}

\section{Generalization Theorems}

\subsection{Probability and Learning Theory Tools}\label{subsec:tools}
We first provide necessary tools and backgrounds we will need to prove our generalization theorems, starting with the definitions of VC dimension, pseudo dimension, and fat shattering dimension. All of these tools and definitions are taken from standard literature on learning theory (see for e.g. \cite{haussler}).

\begin{definition}[VC dimension]
Let $\Hs \subseteq \{ h: \X \to \{0,1\}\}$ be a class of binary functions. For any $S = \{ x_1, \ldots, x_n \} \subseteq \X$, define $\Hs (S) = \{ (h(x_1), \ldots, h(x_n)): h \in \Hs \}$. We say $\Hs$ shatters $S$, if $ \Hs(S) = \{ 0,1 \}^n$, i.e., if $\Hs (S)$ contains all possible labelings of the points in $S$. The Vapnik-Chervonenkis (VC) dimension of $\Hs$ is the cardinality of the largest set of points in $\X$ that can be shattered by $\Hs$. In other words,
\[
VCdim(\Hs) = \max \{n: \exists S \in \X^n \text{ such that $S$ is shattered by $\Hs$} \}
\]
If $\Hs$ shatters arbitrarily large sets of points in $\X$, then $VCdim(\Hs) = \infty$.
\end{definition}
We have that VC dimension of any hypothesis class is bounded by the log size of the class.
\begin{fact}\label{fact:logboundVC}
If $| \Hs | < \infty$, then $VCdim (\Hs) \le \log (| \Hs |)$.
\end{fact}
\begin{definition}[Pseudo dimension]
For any $r \in \R$, define $sign(r) = \1 \left[ r > 0\right]$, for any $r = (r_1, \ldots, r_n) \in \R^n$, define $sign(r) = (sign(r_1), \ldots, sign(r_n))$, for any $T \subseteq \R^n$, define $sign(T) = \{ sign(r) : r \in T\}$, for any $T \subseteq \R^n$ and $r' \in \R^n$, define $T+r' = \{ r+r': r \in T \}$. Let $\F \subseteq \{ f: \X \to \R \}$ be a class of real-valued functions. For any $S = \{ x_1, \ldots, x_n \} \subseteq \X$, define $\F (S) = \{ (f(x_1), \ldots, f(x_n)): f \in \F \}$. We say $\F$ shatters $S$ if there exists $r \in \R^n$, such that $sign(\F(S) + r) = \{0,1\}^n$. The pseudo dimension of $\F$ is the cardinality of the largest set of points in $\X$ that can be shattered by $\F$. In other words,
\[
Pdim (\F) = \max \{n: \exists S \in \X^n \text{ such that $S$ is shattered by $\F$} \}
\]
If $\F$ shatters arbitrarily large sets of points in $\X$, then $Pdim(\F) = \infty$.
\end{definition}

Pseudo dimension generalizes the notion of VC dimension:

\begin{fact}\label{fact:pdimvc1}
If $\Hs \subseteq \{ h: \X \to \{0,1\} \}$, then $Pdim (\Hs) = VCdim (\Hs)$.
\end{fact}

\begin{definition}[Fat shatteing dimension]
Let $\F \subseteq \{ f: \X \to \R \}$ be a class of real-valued functions. Let $\gamma > 0$. We say $\F$ $\gamma$-shatters a set of points $S = \{ x_1, \ldots, x_n\} \subseteq \X$ if there exists $r \in \R^n$ such that for all $b \in \{ -1 , +1 \}^n$, there exists $f \in \F$ satisfying: $b_i (f(x_i) + r_i ) > \gamma$ for all $i \in [n]$. The fat shattering dimension of $\F$ at scale $\gamma$ is the cardinality of the largest set of points in $\X$ that can be $\gamma$-shattered by $\F$. In other words,
\[
fat_\gamma (\F) = \max \{n: \exists S \in \X^n \text{ such that $S$ is $\gamma$-shattered by $\F$} \}
\]
If $\F$ $\gamma$-shatters arbitrarily large sets of points in $\X$, then $fat_\gamma (\F) = \infty$.
\end{definition}

The fat shattering dimension is always less than (or equal to) the pseudo-dimension.
\begin{fact}\label{fact:fat-pseudo}
For every function class $\F$ and every $\gamma$, $fat_\gamma (\F) \le Pdim (\F)$.
\end{fact}

We now state some useful tools from probability theory:

\begin{thm}[Markov's Inequality]\label{thm:markov}
Let $X$ be a nonnegative random variable. We have that for every $a \ge 0$,
\[
\Pr \left[ X > a\right] \le \frac{\E \left[ X \right]}{a}
\]
\end{thm}

\begin{thm}[Additive Chernoff-Hoeffding]\label{thm:add-chernoff}
Suppose $\Ps$ is a distribution over $[0,1]$. Fix any $\epsilon$. We have that for every $\delta \ge 0$, with probability at least $1-\delta$ over the draw of $S \sim \Ps^n$,
\[
\left\vert \E_{x \sim \Ps} \left[ x \right] - \E_{x \sim S} \left[ x \right] \right\vert \le \epsilon
\]
provided that
\[
n \ge \frac{\log \left( 2/\delta \right)}{2 \epsilon^2}
\]
\end{thm}

\begin{thm}[Multiplicative Chernoff-Hoeffding]\label{thm:mult-chernoff}
Suppose $\Ps$ is a distribution over $\{ 0,1 \}$, and let $\mu = \E_{x \sim \Ps} \left[ x \right]$. Fix any $\epsilon$. We have that for every $\delta \ge 0$, with probability at least $1-\delta$ over the draw of $S \sim \Ps^n$,
\[
\left( 1 - \epsilon \right) \mu \le \E_{x \sim S} \left[ x \right] \le \left( 1 + \epsilon \right) \mu
\]
provided that
\[
n \ge \frac{2 \log \left( 2 / \delta \right)}{\mu \epsilon^2}
\]
\end{thm}

Here we state standard uniform convergence bounds for function classes of bounded (VC, pseudo, fat shattering) dimension:

\begin{thm}[Bounded VC dimension $\Longrightarrow$ Generalization]\label{thm:vc-gen}
Let $\Ps$ be a distribution over some domain $\X$. Suppose $\Hs \subseteq \{ h: \X \to \{ 0, 1 \} \}$ is a hypothesis class with VC dimension $VCdim (\Hs) = d$. Fix any $\epsilon \in [0,1]$. Let $c: \X \to \{ 0,1\}$ be an arbitrary function. We have that for every $\delta \ge 0$, with probability at least $1-\delta$ over the draw of $S \sim \Ps^n$,
\[
\sup_{h \in \Hs} \left\vert \E_{x \sim \Ps} \left[ \1 \left[ h(x) \neq c(x) \right] \right] - \E_{x \sim S} \left[ h(x) \neq c(x) \right] \right\vert \le \epsilon
\]
provided that, for some universal constant $c'$,
\[
n \ge \frac{ c' \left( d \log \left(n\right) + \log \left( 1 / \delta \right) \right) }{\epsilon^2}
\]
\end{thm}

\begin{thm}[Bounded Pseudo-Dimension $\Longrightarrow$ Generalization]\label{thm:pseudo-gen}
Let $\Ps$ be a distribution over some domain $\X$. Suppose $\F \subseteq \{ f: \X \to [0, M] \}$ is a function class with pseudo-dimension $Pdim (\F) = d$. Fix any $\epsilon \in [0,M]$. We have that for every $\delta \ge 0$, with probability at least $1-\delta$ over the draw of $S \sim \Ps^n$,
\[
\sup_{f \in \F} \left\vert \E_{x \sim \Ps} \left[ f(x) \right] - \E_{x \sim S} \left[ f(x) \right] \right\vert \le \epsilon
\]
provided that
\[
n \ge \frac{64 M^2 \left( 2 d \log \left( 16 e M / \epsilon \right) + \log \left( 8 / \delta \right) \right) }{\epsilon^2}
\]
\end{thm}

\begin{thm}[Bounded Fat Shattering Dimension $\Longrightarrow$ Generalization]\label{thm:fat-gen}
Let $\Ps$ be a distribution over some domain $\X$. Fix any $\epsilon \in [0,M]$. Suppose $\F \subseteq \{ f: \X \to [0, M] \}$ is a function class with fat shattering dimension $d$ of scale $\epsilon / 8$: $fat_{\epsilon/8} (\F) = d$. We have that for every $\delta \ge 0$, with probability at least $1-\delta$ over the draw of $S \sim \Ps^n$,
\[
\sup_{f \in \F} \left\vert \E_{x \sim \Ps} \left[ f(x) \right] - \E_{x \sim S} \left[ f(x) \right] \right\vert \le \epsilon
\]
provided that
\[
n \ge \frac{32 M^2 \left( d \log \left( 16 n / \epsilon d \right) \log \left( 8 / \epsilon \right) + \log \left( 4 / \delta \right) \right) }{\epsilon^2}
\]
\end{thm}

Here we state some known results on the complexity of composition of two function classes:
\begin{lemma}[XOR of two VC classes \cite{composition}]\label{lem:xor}
Let $\Hs \subseteq \{ h: \X \to \{ 0, 1 \} \}$ be a class with VC dimension $d_\Hs$, and $\F \subseteq \{ f: \X \to \{ 0, 1 \} \}$ be a class with VC dimension $d_\F$. Let $\Hs \oplus \F = \{ (h \oplus f): h \in \Hs, f \in \F\}$ where $(h\oplus f) (x) = h(x) \oplus f(x) = \1 [h(x) \neq f(x) ]$ (XOR). Let $d$ denote the VC dimension of $\Hs \oplus \F$. We have that
$
d \le 10 \max \left\{ d_\Hs, d_\F \right\}
$.
\end{lemma}

We combine Fact~\ref{fact:fat-pseudo} and Theorem 6.4 of \cite{composition}, to get a bound on the fat shattering dimension of the product of two function classes.
\begin{lemma}[Product of two Pdim Classes]\label{lem:product}
Suppose $\F_1 \subseteq \{ f: \X \to [0, M] \} $ and $\F_2 \subseteq \{ f: \X \to [0, M] \}$ are two function classes with pseudo-dimension $d_1$ and $d_2$, respectively. Define the product class $\F = \{ f_{f_1, f_2}: f_1 \in \F_1, f_2 \in \F_2\}$ where $f_{f_1, f_2}: \X \to \R$ is defined by $f_{f_1, f_2} (x) = f_1(x) f_2(x)$. We have that for every $\gamma$,
$
fat_\gamma (\F) = O \left( d_1 + d_2 \right)
$.
\end{lemma}

Finally, we give the definitions of the covering sets and covering numbers for a function class:

\begin{definition}[$\epsilon$-cover and $\epsilon$-covering number]
Let $(A, d)$ be a metric space. A set $C \subseteq A$ is said to be an $\epsilon$-cover for $W \subseteq A$ with respect to the metric $d$, if for every $w \in W$ there exists $c \in C$ such that $d(w,c) \le \epsilon$. We have that the $\epsilon$-covering number of $W$ with respect to $d$ is
\[
\mathcal{N} \left( \epsilon, W, d \right) = \min \left\{ \vert C \vert : C \text{ is an $\epsilon$-cover for } W \text{ w.r.t. } d \right\}
\]
\end{definition}

\begin{definition}[covering number of a function class]
Let $\F \subseteq \{ f: \X \to \R \}$ be a class of real-valued functions. We have that
\[
\mathcal{N} \left( \epsilon, \F, n \right) \triangleq \max_{S \in \X^n} \mathcal{N} \left( \epsilon, \F (S), d_1 \right)
\]
where for $S=\{x_1, \ldots, x_n \}$, $\F(S) = \{ (f(x_1), \ldots, f(x_n)): f \in \F \} \subseteq \R^n$, and that $d_1$ is the following metric over $\R^n$:
\[
\forall r, r' \in \R^n: \ d_1 \left(r , r' \right) = \frac{1}{n} \sum_{i=1}^n \vert r_i - r'_i \vert
\]
\end{definition}

\begin{lemma}[Bounded Pseudo-Dimension $\Longrightarrow$ Bounded Covering Number]\label{lem:N1}
Let $\F \subseteq \{ f: \X \to [0,M] \}$ be a class of real-valued functions such that $Pdim (\F) = d$. We have that for every $\epsilon \in [0,M]$, and every $n$,
\[
\mathcal{N} \left( \epsilon, \F, n \right) = O \left( \left( \frac{1}{\epsilon} \right)^d \right)
\]
\end{lemma}

\subsection{Proofs of Generalization Theorems}
We are now ready to prove our generalization theorems.
\begin{thm}[Generalization over $\Ps$]
Fix any $\epsilon$ and $\delta$. Fix a distribution $\Ps$ over $\X \times \Z$. Suppose $Pdim (\Gs) = d_\Gs$, $VCdim(\F) = d_\F$, and $VCdim(\Hs) = d_\Hs$. We have that with probability at least $1-\delta$ over $S \sim \Ps^n$, every $\hat z$ that is an $\alpha$-proxy with respect to the data set $S$ is also an $(\alpha + \epsilon)$-proxy with respect to the underlying distribution $\Ps$, provided that
\[
n = \tilde{\Omega} \left( \frac{M^2 \left( d_\Gs + \max \left\{ d_\Hs, d_\F \right\}  + \log \left( K / \delta \right) \right)}{\mu \mu_\Gs^2 \left( \min \left\{ \mu, \mu_\Gs \right\} \right)^2 \epsilon^2} \right)
\]
where
\[
\mu = \min_{1 \le k \le K} \E_{(x,z) \sim \Ps} \left[ z_k \right], \quad \mu_\Gs = \inf_{\hat z \in \Gs} \left\{  \min_{1 \le k \le K} \E_{(x,z) \sim \Ps} \left[ \hat{z}_k (x) \right] \right\}
\]
\end{thm}

\begin{proof}
We first provide uniform convergence bounds for every four expectations appearing in the definition of a proxy (see Definition~\ref{def:proxy}). We use $\oplus$ throughout to denote the XOR function: for $a,b \in \{0,1\}$, $a \oplus b = \1 \left[ a \neq b \right]$. Fix any $\epsilon \ge 0$ and any $\delta \in [0,1]$. First, we have by an application of Multiplicative Chernoff-Hoeffding bound (Theorem~\ref{thm:mult-chernoff}) that, with probability at least $1-\delta$ over the draw of $S \sim \Ps^n$, for every $k \in [K]$,
\[
\left( 1 - \epsilon \right) \E_{(x,z) \sim \Ps} \left[ z_k \right] \le \E_{(x,z) \sim S} \left[ z_k \right] \le \left( 1 + \epsilon \right) \E_{(x,z) \sim \Ps} \left[ z_k \right]
\]
as long as
\[
n = \Omega \left( \frac{\log \left( K / \delta \right)}{ \mu \epsilon^2 }\right)
\]
Second, we have by standard uniform convergence bounds for function classes of bounded pseudo-dimension (Theorem~\ref{thm:pseudo-gen}) that, with probability at least $1-\delta$ over the draw of $S \sim \Ps^n$, for every $k \in [K]$ and every $\hat{z}_k \in \Gs$,
\[
\left( 1 - \epsilon \right) \E_{(x,z) \sim \Ps} \left[ \hat{z}_k (x) \right] \le \E_{(x,z) \sim S} \left[ \hat{z}_k (x) \right] \le \left( 1 + \epsilon \right) \E_{(x,z) \sim \Ps} \left[ \hat{z}_k (x) \right]
\]
so long as
\[
n = \Omega \left( \frac{M^2 \left( d_\Gs \log \left( M / \epsilon \right) + \log \left( K / \delta \right) \right)}{ \mu_\Gs^2 \epsilon^2 }\right)
\]

Third, we want to find a uniform convergence bound for $\E_{(x,z) \sim S} \left[ z_k \1 \left[ h(x) \neq f(x,z) \right] \right]$. Fix any $k$. For any $h \in \Hs$ and $f \in \F$, define the XOR function $h \oplus f: \X \times \Z \to \{0,1\}$ as $(h \oplus f) (x,z) = h(x) \oplus f(x,z)$, and let $g_k : \X \times \Z \to \{0,1\}$ be defined as $g_k (x,z) = z_k$. Note that we can write
\begin{equation}\label{eq:re-write0}
z_k \1 \left[ h(x) \neq f(x,z) \right] = z_k \1 \left[ h(x) \oplus f(x,z) = 1 \right] = \1 \left[ g_k (x,z) \cdot (h \oplus f) (x,z) = 1 \right]
\end{equation}
Define $\Hs \oplus \F = \{ h \oplus f: h \in \Hs, f \in \F \}$ and note that Lemma~\ref{lem:xor} implies
\begin{equation}\label{eq:vcxor}
VCdim(\Hs \oplus \F) = O(\max \left\{ d_\Hs, d_\F \right\})
\end{equation}
Let $(\Hs \oplus \F)_k = \{  g_k \cdot (h \oplus f): h \in \Hs, f \in \F \}$ where $g_k \cdot (h \oplus f) (x,z) = g_k (x,z) \cdot (h \oplus f) (x,z)$. Note that
\begin{equation}\label{eq:vc-xor}
VCdim ((\Hs \oplus \F)_k) = O(VCdim(\Hs \oplus \F)) = O(\max \left\{ d_\Hs, d_\F \right\})
\end{equation}
We therefore have using the equality established in Equation~\eqref{eq:re-write0} that
\begin{align*}
&\sup_{k \in [K], h \in \Hs, f \in \F} \left\vert \E_{(x,z) \sim \Ps} \left[ z_k \1 \left[ h(x) \neq f(x,z) \right] \right] - \E_{(x,z) \sim S} \left[ z_k \1 \left[ h(x) \neq f(x,z) \right] \right] \right\vert \\
&= \max_{1 \le k \le K} \left\{ \sup_{h' \in (\Hs \oplus \F)_k} \left\vert \E_{(x,z) \sim \Ps} \left[ \1 \left[ h'(x,z) = 1 \right] \right] - \E_{(x,z) \sim S} \left[ \1 \left[ h'(x,z) = 1 \right] \right] \right\vert \right\}
\end{align*}
But using uniform convergence bounds for VC classes (apply Theorem~\ref{thm:vc-gen} with $c(x) = 1$), we have with probability at least $1-\delta$ over the draw of $S \sim \Ps^n$ that
\[
\max_{1 \le k \le K} \left\{ \sup_{h' \in (\Hs \oplus \F)_k} \left\vert \E_{(x,z) \sim \Ps} \left[ \1 \left[ h'(x,z) = 1 \right] \right] - \E_{(x,z) \sim S} \left[ \1 \left[ h'(x,z) = 1 \right] \right] \right\vert \right\} \le \epsilon
\]
so long as
\[
n = \Omega \left( \frac{ \max \left\{ d_\Hs, d_\F \right\} \log \left( n \right) + \log \left( K / \delta \right) }{\epsilon^2} \right)
\]

where we use Equation~\eqref{eq:vc-xor}. Finally, we want to find a uniform convergence bound for $\E_{(x,z) \sim S} \left[ \hat{z}_k (x) \1 \left[ h(x) \neq f(x) \right] \right]$. Fix any $k$. Note that we can write
\begin{equation}\label{eq:re-write1}
\hat{z}_k (x) \1 \left[ h(x) \neq f(x,z) \right] = \hat{z}_k (x) \1 \left[ (h \oplus f) (x,z) = 1 \right] = \hat{z}_k (x) \cdot (h \oplus f) (x,z)
\end{equation}
Now define a function class $\Gs'_k = \{ \hat{z}_k \cdot (h \oplus f): \hat{z}_k \in \Gs, h \in \Hs, f \in \F \}$ where $\hat{z}_k \cdot (h \oplus f) : \X \times \Z \to \R$ is defined as $\hat{z}_k \cdot (h \oplus f) (x,z) =  \hat{z}_k (x) \cdot (h \oplus f) (x,z)$. Using Lemma~\ref{lem:product}, we have for every $\gamma$
\[
fat_\gamma (\Gs'_k) = O \left( Pdim (\Gs) + Pdim (\Hs \oplus \F)\right)
\]
But $Pdim (\Gs) = d_\Gs$, and using Equation~\eqref{eq:vcxor} and Fact~\ref{fact:pdimvc1}, $Pdim (\Hs \oplus \F) = VCdim (\Hs \oplus \F) = O(\max \left\{ d_\Hs, d_\F \right\})$. Consequently, we have that for every $\gamma$, 
\begin{equation}\label{eq:fatdim}
fat_\gamma (\Gs'_k) = O \left( d_\Gs + \max \left\{ d_\Hs, d_\F \right\} \right)
\end{equation}
We therefore have using Equation~\eqref{eq:re-write1} that
\begin{align*}
&\sup_{k, h, f, \hat{z}} \left\vert \E_{(x,z) \sim \Ps} \left[ \hat{z}_k (x) \1 \left[ h(x) \neq f(x) \right] \right] - \E_{(x,z) \sim S} \left[ \hat{z}_k (x) \1 \left[ h(x) \neq f(x) \right] \right] \right\vert \\
& = \max_{1 \le k \le K} \left\{ \sup_{g \in \Gs'_k} \left\vert \E_{(x,z) \sim \Ps} \left[ g (x,z) \right] - \E_{(x,z) \sim S} \left[ g(x,z) \right] \right\vert \right\}
\end{align*}
But using uniform convergence bounds for function classes with bounded fat shattering dimension (Theorem~\ref{thm:fat-gen}), we have with probability at least $1-\delta$ over the draw of $S \sim \Ps^n$ that
\[
\max_{1 \le k \le K} \left\{ \sup_{g \in \Gs'_k} \left\vert \E_{(x,z) \sim \Ps} \left[ g (x,z) \right] - \E_{(x,z) \sim S} \left[ g(x,z) \right] \right\vert \right\} \le \epsilon
\]
so long as
\[
n = \Omega \left( \frac{M^2 \left( \left( d_\Gs + \max \left\{ d_\Hs, d_\F \right\} \right) \log \left( n / \epsilon \right) \log \left( 1 / \epsilon \right) + \log \left( K / \delta \right) \right)}{\epsilon^2} \right)
\]

where we use Equation~\eqref{eq:fatdim}. Therefore, combining all four uniform convergence bounds, we have that so long as
\begin{equation}\label{eq:samplecomplexity0}
n = \tilde{\Omega} \left( \frac{M^2 \left( d_\Gs + \max \left\{ d_\Hs, d_\F \right\}  + \log \left( K / \delta \right) \right)}{\mu \mu_\Gs^2 \epsilon^2} \right)
\end{equation}
we have with probability $1-\delta$ over the draw of $S \sim \Ps^n$, for all groups $k \in [K]$, all classifiers $h \in \Hs$, all learning tasks $f \in \F$, and all proxies $\hat{z}_k \in \Gs$, that the following inequalities simultaneously hold.
\begin{equation}\label{eq:b1}
 \left( 1 - \epsilon \right) \E_{(x,z) \sim \Ps} \left[ z_k\right] \le \E_{(x,z) \sim S} \left[ z_k\right] \le \left( 1 + \epsilon \right) \E_{(x,z) \sim \Ps} \left[ z_k\right]
\end{equation}
\begin{equation}\label{eq:b2}
 \left( 1 - \epsilon \right) \E_{(x,z) \sim \Ps} \left[ \hat{z}_k (x) \right] \le \E_{(x,z) \sim S} \left[ \hat{z}_k (x) \right] \le \left( 1 + \epsilon \right) \E_{(x,z) \sim \Ps} \left[ \hat{z}_k (x) \right]
\end{equation}
\begin{equation}\label{eq:b3}
\left\vert \E_{(x,z) \sim \Ps} \left[ z_k \1 \left[ h(x) \neq f(x,z) \right] \right] - \E_{(x,z) \sim S} \left[ z_k \1 \left[ h(x) \neq f(x,z) \right] \right] \right\vert \le \epsilon
\end{equation}
\begin{equation}\label{eq:b4}
\left\vert \E_{(x,z) \sim \Ps} \left[ \hat{z}_k (x) \1 \left[ h(x) \neq f(x,z) \right] \right] - \E_{(x,z) \sim S} \left[ \hat{z}_k (x) \1 \left[ h(x) \neq f(x,z) \right] \right] \right\vert \le \epsilon
\end{equation}
We note that in this proof, the same sample complexity bound holds when we take our uniform convergence over the simplex $\Delta (\Gs)$ because any $\sup$ over $\Delta (\Gs)$ in this proof can be upper bounded by a $\sup$ over $\Gs$ due to linearity of expectations. In particular, for any function $g$,
\begin{align*}
& \sup_{p \in \Delta (\Gs)} \left\vert \E_{\hat{z}_k \sim p} \left[ \E_{(x,z) \sim S} \left[ g(\hat{z}_k ; x,z ) \right] - \E_{(x,z) \sim \Ps} \left[ g(\hat{z}_k ; x,z ) \right] \right]\right\vert \\
&\le \sup_{p \in \Delta (\Gs)} \E_{\hat{z}_k \sim p} \left[ \left\vert   \E_{(x,z) \sim S} \left[ g(\hat{z}_k ; x,z ) \right] - \E_{(x,z) \sim \Ps} \left[ g(\hat{z}_k ; x,z ) \right] \right\vert \right] \\
&= \sup_{\hat{z}_k \in \Gs} \left\vert   \E_{(x,z) \sim S} \left[ g(\hat{z}_k ; x,z ) \right] - \E_{(x,z) \sim \Ps} \left[ g(\hat{z}_k ; x,z ) \right] \right\vert
\end{align*}
This observation is important because our algorithm outputs an object in $\Delta (\Gs)^K$ and so our uniform convergence bound must be taken over $\Delta (\Gs)^K$. But given this observation, without any loss, we work with deterministic proxies in $\Gs^K$ in the rest of the proof. In particular, suppose $\hat z \in \Gs^K$ is an $\alpha$-proxy with respect to $S \sim \Ps^n$. In other words we have for all $f \in \F$ and all $h \in \Hs$ and all $k \in [K]$,
\begin{equation}\label{eq:proxy-sample}
\left\vert \frac{\E_{(x,z) \sim S} \left[ z_k \1 \left[ h(x) \neq f(x,z) \right] \right]}{\E_{(x,z) \sim S} \left[  z_k \right]} - \frac{\E_{(x,z) \sim S} \left[ \hat{z}_k (x) \1 \left[ h(x) \neq f(x,z) \right] \right]}{\E_{(x,z) \sim S} \left[  \hat{z}_k (x) \right]} \right\vert \le \alpha
\end{equation}
where $(x,z) \sim S$ means a sample drawn uniformly at random from $S$. We want to use the uniform convergence bounds found above, along with Equation~\eqref{eq:proxy-sample}, to argue that $\hat z$ is a proxy with respect to the underlying distribution $\Ps$, with small degradation in its approximation parameter $\alpha$. I.e., we want to bound the following, for all $f,h,k$.
\[
\left\vert \frac{\E_{(x,z) \sim \Ps} \left[ z_k \1 \left[ h(x) \neq f(x,z) \right] \right]}{\E_{(x,z) \sim \Ps} \left[  z_k \right]} - \frac{\E_{(x,z) \sim \Ps} \left[ \hat{z}_k (x) \1 \left[ h(x) \neq f(x,z) \right] \right]}{\E_{(x,z) \sim \Ps} \left[  \hat{z}_k (x) \right]} \right\vert
\]
Fix any $f,h,k$. Suppose the first term is greater than the second term. Similar derivations apply if the second term is greater than the first one. We have that
\begin{align*}
    &\frac{\E_{(x,z) \sim \Ps} \left[ z_k \1 \left[ h(x) \neq f(x,z) \right] \right]}{\E_{(x,z) \sim \Ps} \left[  z_k \right]} - \frac{\E_{(x,z) \sim \Ps} \left[ \hat{z}_k (x) \1 \left[ h(x) \neq f(x,z) \right] \right]}{\E_{(x,z) \sim \Ps} \left[  \hat{z}_k (x) \right]} \\
    &\le \frac{\left( 1+\epsilon \right) \E_{(x,z) \sim \Ps} \left[ z_k \1 \left[ h(x) \neq f(x,z) \right] \right]}{\E_{(x,z) \sim S} \left[  z_k \right]} - \frac{ \left( 1-\epsilon \right) \E_{(x,z) \sim \Ps} \left[ \hat{z}_k (x) \1 \left[ h(x) \neq f(x,z) \right] \right]}{\E_{(x,z) \sim S} \left[  \hat{z}_k (x) \right]} \\
    &= \left( 1-\epsilon \right) \left( \frac{ \E_{(x,z) \sim \Ps} \left[ z_k \1 \left[ h(x) \neq f(x,z) \right] \right]}{\E_{(x,z) \sim S} \left[  z_k \right]} - \frac{ \E_{(x,z) \sim \Ps} \left[ \hat{z}_k (x) \1 \left[ h(x) \neq f(x,z) \right] \right]}{\E_{(x,z) \sim S} \left[  \hat{z}_k (x) \right]} \right) \\
    &+ 2 \epsilon \frac{ \E_{(x,z) \sim \Ps} \left[ z_k \1 \left[ h(x) \neq f(x,z) \right] \right]}{\E_{(x,z) \sim S} \left[  z_k \right]} \\
    &\le \frac{ \E_{(x,z) \sim S} \left[ z_k \1 \left[ h(x) \neq f(x,z) \right] \right]}{\E_{(x,z) \sim S} \left[  z_k \right]} - \frac{ \E_{(x,z) \sim S} \left[ \hat{z}_k (x) \1 \left[ h(x) \neq f(x,z) \right] \right]}{\E_{(x,z) \sim S} \left[  \hat{z}_k (x) \right]} \\
    &+\frac{ \E_{(x,z) \sim \Ps} \left[ z_k \1 \left[ h(x) \neq f(x,z) \right] \right] - \E_{(x,z) \sim S} \left[ z_k \1 \left[ h(x) \neq f(x,z) \right] \right]}{\E_{(x,z) \sim S} \left[  z_k \right]} \\
    &+ \frac{ \E_{(x,z) \sim \Ps} \left[ \hat{z}_k (x) \1 \left[ h(x) \neq f(x,z) \right] \right] - \E_{(x,z) \sim S} \left[ \hat{z}_k (x) \1 \left[ h(x) \neq f(x,z) \right] \right]}{\E_{(x,z) \sim S} \left[  \hat{z}_k (x) \right]} \\
    &+ 2 \epsilon \frac{ \E_{(x,z) \sim \Ps} \left[ z_k \1 \left[ h(x) \neq f(x,z) \right] \right]}{\E_{(x,z) \sim S} \left[  z_k \right]} \\
    &\le \alpha + \epsilon \left( \frac{3}{\E_{(x,z) \sim S} \left[  z_k \right]} + \frac{1}{\E_{(x,z) \sim S} \left[  \hat{z}_k (x) \right]} \right) \\
    &\le \alpha + \epsilon \left( \frac{3}{\left( 1 - \epsilon \right) \E_{(x,z) \sim \Ps} \left[  z_k \right]} + \frac{1}{\left(1-\epsilon\right)\E_{(x,z) \sim \Ps} \left[  \hat{z}_k (x) \right]} \right) \\
    &\le \alpha + \frac{12 \epsilon}{\min \left\{ \mu, \mu_\Gs \right\} }
\end{align*}
where the first inequality follows from Equations~\eqref{eq:b1} and \eqref{eq:b2}, and the third follows from Equations~\eqref{eq:b3}, \eqref{eq:b4}, and \eqref{eq:proxy-sample}. The fourth inequality is another application of Equations~\eqref{eq:b1} and \eqref{eq:b2}, and the last one follows from the definition of $\mu$ and $\mu_\Gs$. Now, by replacing $\epsilon$ with $\epsilon \cdot \min \left\{ \mu, \mu_\Gs \right\} / 12$ in the sample complexity bound of Equation~\eqref{eq:samplecomplexity0}, we have that if
\[
n = \tilde{\Omega} \left( \frac{M^2 \left( d_\Gs + \max \left\{ d_\Hs, d_\F \right\}  + \log \left( K / \delta \right) \right)}{\mu \mu_\Gs^2 \left( \min \left\{ \mu, \mu_\Gs \right\} \right)^2 \epsilon^2} \right)
\]
Then $\hat z$ is an $(\alpha + \epsilon)$-proxy with respect to the underlying distribution $\Ps$.
\end{proof}

\begin{thm}[Generalization over $\Ps$ and $\Qs$]
Fix any $\epsilon$, $\delta$, and $\beta$. Fix a distribution $\Ps$ over $\X \times \Z$, and a distribution $\Qs$ over $\F$. Suppose $Pdim (\Gs) = d_\Gs$, and $VCdim(\Hs) = d_\Hs$. We have that with probability at least $1-\delta$ over $S \sim \Ps^n$ and $F \sim \Qs^m$, every $\hat z$ that is an $(\alpha, 0)$-proxy with respect to the data set $(S, F)$ is also a $((\alpha + \epsilon) / \beta, \beta )$-proxy with respect to the underlying distributions $(\Ps, \Qs)$, provided that
\[
n = \tilde{\Omega} \left( \frac{M^2 \left( d_\Gs + \max \left\{ d_\Hs, \log \left( m \right) \right\}  + \log \left( K / \delta \right) \right)}{\mu \mu_\Gs^2 \left( \min \left\{ \mu, \mu_\Gs \right\} \right)^2 \epsilon^2} \right)
\]
\[
 m = \tilde{\Omega} \left( \frac{M^2 \left( K d_\Gs \log \left( \left\vert \text{supp} (\Ps) \right\vert \right) + \log \left( 1 / \delta \right) \right) }{ \mu_\Gs^2 \left( \min \left\{ \mu, \mu_\Gs \right\} \right)^2 \epsilon^4} \right)
\]
where $\text{supp} (\Ps)$ is the support of $\Ps$, and that
\[
\mu = \min_{1 \le k \le K} \E_{(x,z) \sim \Ps} \left[ z_k \right], \quad \mu_\Gs = \inf_{\hat z \in \Gs} \left\{  \min_{1 \le k \le K} \E_{(x,z) \sim \Ps} \left[ \hat{z}_k (x) \right] \right\}
\]
\end{thm}
\begin{proof}
We lift our in sample guarantees to distributional guarantees in two steps. First, Theorem~\ref{thm:gen-x} implies, for every set of functions $F \in \F^m$, with probability at least $1-\delta/2$ over $S \sim \Ps^n$, every $\hat{z}$ that is $(\alpha, 0)$-proxy with respect to $(S,F)$ is also $(\alpha+\epsilon/2, 0)$-proxy with respect to $(\Ps, F)$, where we use the fact (Fact~\ref{fact:logboundVC}) that the VC dimension of $F$ is at most $\log m$: $d_\F \le \log m$ (in the sample complexity for $n$). Second, we can apply Lemma~\ref{lem:gen-f} to conclude that, with probability at least $1-\delta/2$ over $F \sim \Qs^m$, every $\hat{z}$ that is an $(\alpha+\epsilon/2, 0)$-proxy with respect to $(\Ps, F)$, is also a $(\alpha + \epsilon / \beta, \beta )$-proxy with respect to the underlying distributions $(\Ps, \Qs)$.
\end{proof}

\begin{lemma}[Generalization over $\Qs$]\label{lem:gen-f}
Fix any $\epsilon$, $\delta$ and $\beta$. Fix a distribution $\Ps$ over $\X \times \Z$ and a distribution $\Qs$ over $\F$. Suppose $Pdim (\Gs) = d_\Gs$. We have that with probability at least $1-\delta$ over $F \sim \Qs^m$, every $\hat z$ that is an $(\alpha, 0)$-proxy with respect to $(\Ps, F)$ is also a $( (\alpha + \epsilon)/ \beta, \beta )$-proxy with respect to $(\Ps, \Qs)$, provided that
\[
 m = \tilde{\Omega} \left( \frac{M^2 \left( K d_\Gs \log \left( \left\vert \text{supp} (\Ps) \right\vert \right) + \log \left( 1 / \delta \right) \right) }{ \mu_\Gs^2 \left( \min \left\{ \mu, \mu_\Gs \right\} \right)^2 \epsilon^4} \right)
\]
where $\text{supp} (\Ps)$ is the support of $\Ps$.
\end{lemma}

\begin{proof}[Proof of Lemma~\ref{lem:gen-f}]
Define
\[ g \left( f ; \hat{z}, h, k\right) \triangleq 
\left\vert \frac{\E_{(x,z) \sim \Ps} \left[ z_k \1 \left[ h(x) \neq f(x,z) \right] \right]}{\E_{(x,z) \sim \Ps} \left[  z_k \right]} - \frac{\E_{(x,z) \sim \Ps} \left[ \hat{z}_k (x) \1 \left[ h(x) \neq f(x,z) \right] \right]}{\E_{(x,z) \sim \Ps} \left[  \hat{z}_k (x) \right]} \right\vert
\]
which is the quantity that appears in the definition of a proxy. Suppose for $F \sim \Qs^m$, we have that $\hat z$ is an $(\alpha, 0)$-proxy with respect to $(\Ps, F)$. In other words we have that $\hat z$ satisfies: for every $f \in F$, every $h \in \Hs$, and every $k \in [K]$,
\[
g \left( f ; \hat{z}, h, k\right) \le \alpha
\]
For any $\alpha'$, we have that
\begin{align}\label{eq:something1}
\begin{split}
    &\Pr_{f \sim \Qs} \left[ \exists h, k : g \left( f ; \hat{z}, h, k\right) > \alpha' \right] \\
    &\le \Pr_{f \sim \Qs} \left[ \sup_{h,k} \left\{ g \left( f ; \hat{z}, h, k\right) \right\} > \alpha' \right] \\
    &\le \frac{\E_{f \sim \Qs} \left[ \sup_{h,k} \left\{ g \left( f ; \hat{z}, h, k\right) \right\} \right]}{\alpha'} \\
    &\le \frac{ \alpha + \left\vert \E_{f \sim \Qs} \left[ \sup_{h,k} \left\{ g \left( f ; \hat{z}, h, k\right) \right\} \right] - \E_{f \sim F} \left[ \sup_{h,k} \left\{ g \left( f ; \hat{z}, h, k\right) \right\} \right] \right\vert}{\alpha'} \\
    &\le \frac{ \alpha + \sup_{\hat{z}} \left\vert \E_{f \sim \Qs} \left[ \sup_{h,k} \left\{ g \left( f ; \hat{z}, h, k\right) \right\} \right] - \E_{f \sim F} \left[ \sup_{h,k} \left\{ g \left( f ; \hat{z}, h, k\right) \right\} \right] \right\vert}{\alpha'}
\end{split}
\end{align}

where the second inequality is an application of Markov's inequality (Theorem~\ref{thm:markov}), and the third follows because $\E_{f \sim F} \left[ \sup_{h,k} \left\{ g \left( f ; \hat{z}, h, k\right) \right\} \right] \le \alpha$. We first bound the following
\[
\sup_{\hat{z} \in \Gs^K} \left\vert \E_{f \sim \Qs} \left[ \sup_{h,k} \left\{ g \left( f ; \hat{z}, h, k\right) \right\} \right] - \E_{f \sim F} \left[ \sup_{h,k} \left\{ g \left( f ; \hat{z}, h, k\right) \right\} \right] \right\vert
\]
which is a uniform convergence over the class of functions $\Gs^K$, and then consider uniform convergence over the simplex $\Delta (\Gs)^K$ which is what we want. Note that $\Gs$ can potentially have infinitely many functions, but it is known that when the pseudo dimension of $\Gs$ is \emph{finite}, the $\epsilon$-cover (defined in the previous subsection) of $\Gs$ is \emph{finite}, and hence, up to an $O(\epsilon)$ error, we can take our uniform convergence over the $\epsilon$-cover of $\Gs$ which will enable us to apply a union bound over this finite class of functions.
 In particular, if $S = \text{supp} (\Ps)$ is the entire data points in the support of $\Ps$, then Lemma~\ref{lem:N1} implies that for every $\epsilon'$,
\[
\mathcal{N} \left( \frac{\epsilon'}{\left\vert \text{supp} (\Ps) \right\vert}, \Gs (S), d_1 \right) = O \left( \left( \frac{\left\vert \text{supp} (\Ps) \right\vert}{\epsilon'} \right)^{d_\Gs} \right)
\]
implying that there exists some $C \subseteq \Gs$, such that the following holds: for every ${\hat z}_k \in \Gs$, there exists $\tilde{z}_k \in C$, such that
\[
d_1 (\hat{z}_k (S), \tilde{z}_k (S) ) \le \frac{\epsilon'}{\left\vert \text{supp} (\Ps) \right\vert} \Longrightarrow \forall x \in \text{supp} (\Ps): \ \left\vert \hat{z}_k (x) - \tilde{z}_k (x) \right\vert \le \epsilon'
\]
and furthermore, we have that,
\[
\left\vert C \right\vert = O \left( \left( \frac{\left\vert \text{supp} (\Ps) \right\vert}{\epsilon'} \right)^{d_\Gs} \right)
\]

Now given $\hat{z} \in \Gs^K$, and $\epsilon \in [0,4]$, let $\tilde{z} \in C^K$ be such that for all $k \in [K]$,
\[
\forall x \in \text{supp} (\Ps): \ \left\vert \hat{z}_k (x) - \tilde{z}_k (x) \right\vert \le \frac{\epsilon \mu_G}{8} := \epsilon'
\]
We have that
\begin{align*}
&\frac{\E_{(x,z) \sim \Ps} \left[ \tilde{z}_k (x) \1 \left[ h(x) \neq f(x,z) \right] \right]}{\E_{(x,z) \sim \Ps} \left[  \tilde{z}_k (x) \right]} - \frac{\E_{(x,z) \sim \Ps} \left[ \hat{z}_k (x) \1 \left[ h(x) \neq f(x,z) \right] \right]}{\E_{(x,z) \sim \Ps} \left[  \hat{z}_k (x) \right]} \\
&\le \frac{\E_{(x,z) \sim \Ps} \left[ \hat{z}_k (x) \1 \left[ h(x) \neq f(x,z) \right] \right] + \epsilon'}{\E_{(x,z) \sim \Ps} \left[  \hat{z}_k (x) \right] - \epsilon'} - \frac{\E_{(x,z) \sim \Ps} \left[ \hat{z}_k (x) \1 \left[ h(x) \neq f(x,z) \right] \right]}{\E_{(x,z) \sim \Ps} \left[  \hat{z}_k (x) \right]} \\
&\le \frac{\epsilon'}{\E_{(x,z) \sim \Ps} \left[  \hat{z}_k (x) \right] - \epsilon \mu_G} \\
&\le \frac{\epsilon'}{\mu_\Gs - \epsilon'} \\
&\le \frac{\epsilon}{4}
\end{align*}
We can similarly show
\[
\frac{\E_{(x,z) \sim \Ps} \left[ \hat{z}_k (x) \1 \left[ h(x) \neq f(x,z) \right] \right]}{\E_{(x,z) \sim \Ps} \left[  \hat{z}_k (x) \right]} - \frac{\E_{(x,z) \sim \Ps} \left[ \tilde{z}_k (x) \1 \left[ h(x) \neq f(x,z) \right] \right]}{\E_{(x,z) \sim \Ps} \left[  \tilde{z}_k (x) \right]} \le \frac{\epsilon}{4}
\]
which implies
\[
\left\vert \frac{\E_{(x,z) \sim \Ps} \left[ \tilde{z}_k (x) \1 \left[ h(x) \neq f(x,z) \right] \right]}{\E_{(x,z) \sim \Ps} \left[  \tilde{z}_k (x) \right]} - \frac{\E_{(x,z) \sim \Ps} \left[ \hat{z}_k (x) \1 \left[ h(x) \neq f(x,z) \right] \right]}{\E_{(x,z) \sim \Ps} \left[  \hat{z}_k (x) \right]} \right\vert \le \frac{\epsilon}{4}
\]
Therefore, we have that
\begin{align*}
&\sup_{\hat{z} \in \Gs^K} \left\vert \E_{f \sim \Qs} \left[ \sup_{h,k} \left\{ g \left( f ; \hat{z}, h, k\right) \right\} \right] - \E_{f \sim F} \left[ \sup_{h,k} \left\{ g \left( f ; \hat{z}, h, k\right) \right\} \right] \right\vert \\
&\le \sup_{\tilde{z} \in C^K} \left\vert \E_{f \sim \Qs} \left[ \sup_{h,k} \left\{ g \left( f ; \tilde{z}, h, k\right) \right\} \right] - \E_{f \sim F} \left[ \sup_{h,k} \left\{ g \left( f ; \tilde{z}, h, k\right) \right\} \right] \right\vert + \frac{\epsilon}{2}
\end{align*}
I.e. we have reduced a uniform convergence over $\Gs^K$ to a uniform convergence over the finite set $C^K$. Now we can apply a Chernoff-Hoeffding bound (Theorem~\ref{thm:add-chernoff}), while union bounding over the finite covering $C^K$, to get that with probability $1-\delta$ over $F \sim \Qs^m$,
\begin{align*}
&\sup_{\tilde{z} \in C^K} \left\vert \E_{f \sim \Qs} \left[ \sup_{h,k} \left\{ g \left( f ; \tilde{z}, h, k\right) \right\} \right] - \E_{f \sim F} \left[ \sup_{h,k} \left\{ g \left( f ; \tilde{z}, h, k\right) \right\} \right] \right\vert \\
&\le \frac{M}{\min \left\{ \mu, \mu_\Gs \right\}} \sqrt{ \frac{K d_\Gs \log \left( \left\vert \text{supp} (\Ps) \right\vert / \epsilon \mu_\Gs \right) + \log \left( 2 / \delta \right)}{2 m} }
\end{align*}
Hence, if
\begin{equation}\label{eq:sampcomp}
m = 8 \left( \frac{M}{\min \left\{ \mu, \mu_\Gs \right\}} \right)^2  \frac{K d_\Gs \log \left( 4 \left\vert \text{supp} (\Ps) \right\vert / \alpha \mu_\Gs \right) + \log \left( 2 / \delta \right)}{ \epsilon^2} 
\end{equation}
we are guaranteed that, with probability $1-\delta$ over $F \sim \Qs^m$, 
\[
\sup_{\hat{z} \in \Gs^K} \left\vert \E_{f \sim \Qs} \left[ \sup_{h,k} \left\{ g \left( f ; \hat{z}, h, k\right) \right\} \right] - \E_{f \sim F} \left[ \sup_{h,k} \left\{ g \left( f ; \hat{z}, h, k\right) \right\} \right] \right\vert \le \epsilon
\]
We now need to lift this uniform convergence bound over $\Gs^K$ to a uniform convergence bound over the simplex $\Delta (\Gs)^K$. We achieve this, for some appropriately chosen $s$, by reducing the uniform convergence over $\Delta (\Gs)^K$ to the uniform convergence over $\Delta_s (\Gs)^K$ where $\Delta_s (\Gs)$ denotes the distributions over $\Gs$ that are $s$-sparse, i.e., their support size is at most $s$. In particular, for any $\epsilon'$, if $s = O(1/\epsilon'^2)$, then for every distribution $p \in \Delta (\Gs)$, an application of Chernoff-Hoeffding's inequality (Theorem~\ref{thm:add-chernoff}) implies that there exists $p_s \in \Delta_s (\Gs)$ (which can be derived by taking the uniform distribution over $s$ samples drawn $i.i.d.$ from $p$) such that for all $x \in \text{supp} (\Ps)$ and all $k$,
\[
\left\vert \E_{\hat{z}_k \sim p} \left[ \hat{z}_k (x) \right] - \E_{\hat{z}_k \sim p_s} \left[ \hat{z}_k (x) \right] \right\vert \le \epsilon'
\]
Given this observation, and taking $\epsilon' = O(\epsilon \mu_\Gs)$ which implies $s = O(1/ (\epsilon \mu_\Gs)^2)$, we can show similar to our previous derivations, that
\begin{align*}
&\sup_{\hat{z} \in \Delta(\Gs)^K} \left\vert \E_{f \sim \Qs} \left[ \sup_{h,k} \left\{ g \left( f ; \hat{z}, h, k\right) \right\} \right] - \E_{f \sim F} \left[ \sup_{h,k} \left\{ g \left( f ; \hat{z}, h, k\right) \right\} \right] \right\vert \\
&\le \sup_{\hat{z} \in \Delta_s (\Gs)^K} \left\vert \E_{f \sim \Qs} \left[ \sup_{h,k} \left\{ g \left( f ; \hat{z}, h, k\right) \right\} \right] - \E_{f \sim F} \left[ \sup_{h,k} \left\{ g \left( f ; \hat{z}, h, k\right) \right\} \right] \right\vert + O (\epsilon) \\
&\le \sup_{\hat{z} \in \Delta_s (C)^K} \left\vert \E_{f \sim \Qs} \left[ \sup_{h,k} \left\{ g \left( f ; \hat{z}, h, k\right) \right\} \right] - \E_{f \sim F} \left[ \sup_{h,k} \left\{ g \left( f ; \hat{z}, h, k\right) \right\} \right] \right\vert + O (\epsilon)
\end{align*}
So using our uniform convergence bound over $C^K$, and taking the desired union bound over $\Delta_s (C)^K$, which is of size $|C|^{sK}$, we can blow up the sample complexity (Equation~\ref{eq:sampcomp}) by a factor of $s$ and get a uniform convergence bound over $\Delta_s (C)^K$. In other words, as long as
\[
m = \Omega \left( \left( \frac{M}{\min \left\{ \mu, \mu_\Gs \right\}} \right)^2  \frac{K d_\Gs \log \left( 4 \left\vert \text{supp} (\Ps) \right\vert / \alpha \mu_\Gs \right) + \log \left( 2 / \delta \right)}{ \mu_\Gs^2 \cdot \epsilon^4} \right)
\]
we are guaranteed that, with probability $1-\delta$ over $F \sim \Qs^m$, 
\[
\sup_{\hat{z} \in \Delta(\Gs)^K} \left\vert \E_{f \sim \Qs} \left[ \sup_{h,k} \left\{ g \left( f ; \hat{z}, h, k\right) \right\} \right] - \E_{f \sim F} \left[ \sup_{h,k} \left\{ g \left( f ; \hat{z}, h, k\right) \right\} \right] \right\vert \le \epsilon
\]

and consequently, using Equation~\eqref{eq:something1}, we get that with probability $1-\delta$ over $F \sim \Qs^m$, for $\hat{z}$ that is an $(\alpha,0)$-proxy with respect to $(\Ps, F)$,
\[
\Pr_{f \sim \Qs} \left[ \exists h, k : g \left( f ; \hat{z}, h, k\right) > \alpha' \right] \le \frac{\alpha + \epsilon}{\alpha'}
\]
So for $\alpha' = (\alpha + \epsilon)/ \beta$, we get that
\[
\Pr_{f \sim \Qs} \left[ \exists h, k : g \left( f ; \hat{z}, h, k\right) > \frac{ (\alpha + \epsilon)}{\beta} \right] \le \beta
\]
meaning $\hat z$ is a $( (\alpha+\epsilon) / \beta, \beta)$-proxy with respect to $(\Ps, \Qs)$.
\end{proof}

\section{Missing Material Section~\ref{sec:experiments}}
\label{sec:appendix_experiments}
\subsection{Definitions}
\begin{definition}[Weighted Binary Sample Transformation]
\label{def:swt}
We define a Weighted Binary Sample Transformation of a dataset $S$ $(WBST(S,\hat{z}_k))$ as a function that takes in a dataset $S=\{(x_i,y_i)\}_{i=1}^{n} \subset \X \times Y$ and a proxy function $\hat{z}_k \in \Gs$ and produces an augmented dataset $\tilde{S}= \{(\tilde{x}_{ik},\tilde{y}_i,\tilde{z}_{ik})\}_{i=1}^{2n} \in \X \times Y \times \{0,1\}$ equipped with a probability mass vector $\tilde{q}_k=\{\tilde{q}_{ik}\}_{i=1}^{2n}$ such that
\begin{enumerate}
    \item $(\tilde{x}_i,\tilde{y}_i) = (x_i,y_i)$ for $1\leq i \leq n$
    \item $(\tilde{x}_{i+n},\tilde{y}_{i+n}) = (x_i,y_i)$ for $1\leq i \leq n$
    \item $\tilde{z}_{ik} = \1[i>n]$
    \item $\tilde{q}_{ik}=\frac{1-\hat{z}_k(x_i)}{n}$ if $1\leq i \leq n$ and $\frac{\hat{z}_k(x_i)}{n}$ otherwise
\end{enumerate}

We assume that samples $\{(x_i,y_i)\}_{i=1}^{n}$ have uniform mass $q_i = \frac{1}{n}$ in the dataset $S$.

\end{definition}

\begin{clm}
\label{clm:swt}
Consider a dataset $S \cup z \in \X \times \Y \times \Z$, where $z \in \Z = \{0,1\}^K$, and a proxy $\hat{z} \in \Gs$. For any group $z_k$ and any hypothesis $h \in \Hs$, the average error of group $z_k$ estimated according to $\tilde{z}_k$ with dataset $\tilde{S}$ is the same as the average error estimated according to the proxy $\hat{z}_k$ with $S$, i.e.

\begin{align}
&\frac{\E_{(x,y,z) \sim \tilde{q}} \left[(1-\tilde{z}_k) \cdot  \1[h(\tilde{x}) \neq \tilde{y}] \right]}{\E_{(x,y,z) \sim \tilde{q}} \left[1-\tilde{z}_k\right]} = \frac{\E_{(x,y)\sim q} \left[ (1-\hat{z}_k(x))\1[h(x) \neq y)\right]}{\E_{(x,y)\sim q}\left[1-\hat{z}_k(x)\right]} \\
&\frac{\E_{(\tilde{x},\tilde{y},\tilde{z}_k) \sim \tilde{q}} \left[ \tilde{z} \cdot  \1[h(\tilde{x}) \neq \tilde{y}] \right]}{\E_{(\tilde{x},\tilde{y},\tilde{z}) \sim \tilde{q}} \left[\tilde{z}_k \right]} = \frac{\E_{(x,y)\sim q} \left[ \hat{z}_k(x) \1[h(x) \neq y)\right]}{\E_{(x,y)\sim q}\left[\hat{z}_k(x)\right]}
\end{align}

\end{clm}

\begin{proof}
We write the expectations as finite weighted sums over the probability spaces $\tilde{S}$ and $S$ respectively. To move from the first to second line, we observe that $(1-\tilde{z}_{ik})$ is an indicator that only evaluates to 1 on the sample indices $1\leq i \leq n$. This allows us to reduce the sum over $2n$ terms to a sum over $n$ terms. To move from the second to third line, we use the fact that $(\tilde{x}_i,\tilde{y}_i)=(x_i,y_i)$.

\begin{align}
\frac{\E_{(x,y,z) \sim \tilde{q}} \left[(1-\tilde{z}_k) \cdot  \1[h(\tilde{x}) \neq \tilde{y}] \right]}{\E_{(x,y,z) \sim \tilde{q}} \left[1-\tilde{z}_k\right]}
&=\frac{\sum_{i=1}^{2n} \tilde{q}_{ik} \cdot \1[h(\tilde{x}_i) \neq \tilde{y}_i] \cdot (1-\tilde{z}_{ik})}{\sum_{i=1}^{2n} \tilde{q}_{ik} \cdot (1-\tilde{z}_{ik})} \\
&= \frac{\sum_{i=1}^{n} (1-\hat{z}_k(\tilde{x}_i))\1[h(\tilde{x}_i) \neq \tilde{y}_i]}{\sum_{i=1}^{n} (1-\hat{z}_k(\tilde{x}_i))} \\
&= \frac{\sum_{i=1}^{n} (1-\hat{z}_k(x_i))\1[h(x_i) \neq y_i]}{\sum_{i=1}^{n} (1-\hat{z}_k(x_i))} \\
&= \frac{\E_{(x,y)\sim q} \left[ (1-\hat{z}_k(x)\1[h(x) \neq y)\right]}{\E_{(x,y)\sim q}\left[1-\hat{z}_k\right]}
\end{align}

We proceed similarly to show the second equality in the claim. Here we use the fact that $\tilde{z}_{ik}$ is 1 on samples $1+n \leq i \leq 2n$ and 0 everywhere else. We also use the fact that $(\tilde{x}_{i+n},\tilde{y}_{i+n}) = (x_i,y_i)$.

\begin{align}
\frac{\E_{(x,y,z) \sim \tilde{q}} \left[\tilde{z}_k \cdot \1[h(\tilde{x}) \neq \tilde{y}]\right]}{\E_{(x,y,z) \sim \tilde{q}_{ik}} \left[\tilde{z}_k\right]}
&=\frac{\sum_{i=1}^{2n} \tilde{q}_{ik} \cdot \1[h(\tilde{x}_i) \neq \tilde{y}_i] \cdot \tilde{z}_{ik}}{\sum_{i=1}^{2n} \tilde{q}_{ik} \cdot \tilde{z}_{ik}} \\
&= \frac{\sum_{i=1+n}^{2n} \hat{z}_k(\tilde{x}_i) \cdot \1[h(\tilde{x}_i) \neq \tilde{y}_i]}{\sum_{i=n+1}^{2n} \hat{z}_k(\tilde{x}_i)} \\
&= \frac{\sum_{i=1}^{n} \hat{z}_k(x_i) \cdot \1[h(x_i) \neq y_i]}{\sum_{i=n+1}^{2n} \hat{z}_k(x_i)} \\
&= \frac{\E_{(x,y)\sim q} \left[\hat{z}_k(x) \cdot \1[h(x) \neq y)\right]}{\E_{(x,y)\sim q}\left[\hat{z}_k(x)\right]}
\end{align}

\end{proof}

\begin{definition}[Paired Regression Classifier]
\label{def:PRC}
The paired regression classifier operates as follows: We form two weight vectors, $z^0$ and $z^1$, where $z^k_i$ corresponds to the penalty assigned to sample $i$ in the event that it is labeled $k$. For the correct labeling of $x_i$, the penalty is $0$. For the incorrect labeling, the penalty is the current sample weight of the point, $w_i$. We fit two linear regression models $h^0$ and $h^1$ to predict $z^0$ and $z^1$, respectively, on all samples. Then, given a new point $x$, we calculate $h^0(x)$ and $h^1(x)$ and output $h(x) = \argmin_{k\in\{0,1\}} h^k(x)$.
\end{definition} 

\begin{definition}[(Informal) Reductions algorithm for error parity]
\label{def:reductions_error_parity}

\text{} 

\noindent Input: 
\begin{enumerate}
    \item An arbitrary dataset with binary sensitive features and (optional) sample weights 
    \item An approximate cost sensitive classification oracle.\footnote{Because classification algorithms are generally intractable, the use of heuristics, such as the paired regression classifier of\emily{Add citation} (Section~\ref{sec:prc}) is necessary.}
    \item A relaxation parameter $\gamma \in [0, 1]$ 
\end{enumerate}

\noindent Output:

A randomized ensemble of classification model--each of which is produced by the CSC "oracle"--that, in expectation, minimize population error subject to (approximately) enforcing the constraint that the difference in error between any pair of sensitive groups is at most $\gamma$. When $\gamma = 0$, the algorithm produces a model that minimizes population error while (approximately) enforcing exact error parity. 
\end{definition}

\subsection{Implementation Details and Hyperparameters}
\label{sec:appendix_implementation}

\subsubsection{Implementation details}

We implemented our algorithm in PyTorch with a custom loss function to solve our constrained optimization. The two non-standard elements of our implementation are the specific definition of our loss function, and the use of the auditor in the training loop to produce the most constraint violating model $h \in \Hs$ with respect to the current weights of the proxy. This implementation is based on the algorithm derived for linear $\Gs$ in Section~\ref{sec:linear}, but simplified to work easily with Pytorch's auto differentiation. In particular, we do not carefully set the values for the number of rounds $T$, the dual variable upper bound $C$ or the learning rate $\eta$. Instead, optimization is completely delegated to the optimizer (in our case Adam) with our custom loss function. 

Our loss function takes the following form:
$$
\text{Proxy Loss}(z, \hat{z}^t, y,  h^t, \alpha) = \alpha \cdot \sum_{i=1}^n {\left(z_i - \hat{z}_i\right)^2} + \left | \frac{\sum_{i=1}^n{\hat{z_i}}}{\sum_i{z_i}} - 1 \right | + \left | \sum_{i=1}^n \left( z_i - \hat{z}_i \right) \cdot \mathds{1}[h^t(x_i) \neq y_i] \right|
$$
\noindent where $\alpha$ is hyperparameter coefficient, $y$ are the true labels, $h^t$ is the ``most violating model'' chosen by the auditor in round $t$, $z$ are the true sensitive features, and $\hat{z}^t$ are the proxy values for the sensitive features in round $t$. 

The first term in the loss function is  MSE($z$, $\hat{z}$). This corresponds to the objective function of our constrained optimization and--despite not being strictly necessary to the theory-- we found was a useful heuristic in guiding the weights of the proxy. We also apply a scaling parameter $\alpha < 1$ to ensure that we satisfy the constraints before focusing on our objective. The second term corresponds to a penalty denoting the degree of violation of the constraint that the arithmetic means of $z$ and $\hat{z}$ are equal. The third term denotes the degree of violation of the constraint that enforces that the absolute value of the sum of differences of $z_i$ and $\hat{z}_i$ \textit{in the regions where the auditor's model $h$ makes an error} are equal. We note that our loss function is dynamic in that after each step of gradient descent, the auditor chooses a new model $h$ which may make errors on a completely different set of points than in the previous round of gradient descent. 

The training loop works as follows. For $t \in [1..T]$:

\begin{enumerate}

    \item Use the current weights of the model to compute $\hat{z}^t$
    
    \item Let the vector $c = \left(z - \hat{z}\right)\cdot \left(1 - 2 \cdot y\right)$
    \item Train $h = \Hs(x, c)$ to be the regression function that predicts these costs as a function of $x$
    
    \item Let $h_{+}$ and $h_{-}$ be classification models $h_{+} = h(x) > 0$ and $h_{-} = h(x) <= 0$
    
    \item Let $err_{+} = h_{+}(x) != y$ and  $err_{-} = h_{-}(x) != y$ 
    
    \item Let $v_{+} = \sum_i err_{+}$ and $v_{-} = -\sum_i err_{-}$ 
    
    \item If $v_{+} > v_{-}$, let $h^t = h_{+}$. Otherwise let $h^t = h_{-}$ 
    
    \item Update the weights of the proxy ($\hat{z}$) with gradient descent according to the loss function above to create $\hat{z}^{t+1}$
    
\end{enumerate}

In step (1), the vector $c$ denotes the ``costs'' of an error on a positive prediction when the costs of a negative prediction have been normalized to 0. (See Equation~\ref{eqn:costs} and the accompanying explanation for a derivation. The concept of costs is also discussed in Section~\ref{sec:general} in the main body.) In step (4), note that to maximize the absolute sum of differences between $z$ and $\hat{z}$ on the error region, the auditor will either want to select only positive differences or only negative differences if possible, and will prefer higher values to lower ones. Thresholding this regression model is a heuristic way of solving the CSC problem that finds that model that makes errors with the highest such sum. In steps (6) and (7), $v_{+}$ and $v_{-}$ denote the degree of violation of each of the models $h_{+}$ and $h_{-}$ of the constraint that the auditor is trying to find a violating model for, so we select the model corresponding to the highest constraint violation. 

\subsubsection{Hyperparameters}

We stuck to a standard set of hyperparameters for all experiments to keep results consistent. To select these hyperparameters, we increased the number of rounds until we reached convergence on a few tasks and then fixed them for all remaining experiments. We avoided hyperparameter tuning on each task to ensure that our results were an accurate representation of our algorithm's performance.  

\begin{itemize}
    \item For our proxy algorithm we used $\alpha = 0.1$ and $T = 300$ with a learning rate of 0.01 and the Adam optimizer built into PyTorch (the remaining settings of the Adam optimizer were left at their defaults)
    
    \item  For our downstream learner, which was the reductions model for error parity, we used $T = 500, a=5, b=0.5$ where $a$ and $b$ define the learning rate $\eta$ at round $t$. In particular, $\eta^t = a \cdot t^{-b}$. The default parameters are $a = 1$ and $b=0.5$, we simply increased $a$ by a constant factor to speed up training. 
    
\end{itemize}

\subsection{More dataset info}
\label{sec:appendix_datasets}

Each of the tasks we performed come directly from the pre-defined tasks specified in \cite{DBLP:journals/corr/abs-2108-04884}, and specifically as implemented in the folktables Python package which is available on GitHub at \url{https://github.com/zykls/folktables}. Since these tasks each used some subset of the data included in the entire ACS dataset, we will describe the changes we made to particular \textit{features} which, combined with the definition of these tasks, fully specify the experiment.

\begin{enumerate}

    \item We removed the following features from all tasks: OCCP, POBP, ST, PUMA, POWPUMA, RELP. Our reason for doing this was that the first five features on this list were categorical and contained dozens or even hundreds of distinct categories. Since we applied a one-hot encoding, including even one of these features would result in a dimensionality increase potentially many times greater than the dimensionality resulting from all other features. The RELP feature was excluded because--in addition to containing many categorical option--it was unclear exactly what this value represented.
    
    \item We applied binning to the following features:  SCHL, ESP, JWTR. Feature values in each bin were replaced with the bin's index, leaving these categorical features with fewer distinct categories. In parentheses, we specify the encoding values in the original dataset of all entries that fit into a particular bin.
    
    \begin{itemize}
        \item  The SCHL feature represents the amount of schooling and originally had 24 distinct options including every unique grade level. We simplified this into the following categories:
        \begin{itemize}
            \item Didn't finish high school (0-15)
            \item Finished high school or equivalent (16-19)
            \item Associate's degree (20)
            \item Bachelor's degree (21)
            \item Master's degree (22)
            \item Other professional degree (23)
            \item PhD (24)
        \end{itemize}
        
        \item The ESP feature represented the employment status of one's parents. We created the following categories by treating the gender of the parents as irrelevant:
        \begin{itemize}
            \item N/A (0)
            \item Living with two parents, both working (1)
            \item Living with two parents, one working (2 and 3)
            \item Living with two parents, neither (4)
            \item Living with one parent, working (5 and 7)
            \item Living with one parent, not working (6 and 8)
        \end{itemize}
        
        \item The JWTR feature represented one's means of transportation to work. We created the following categories:
        \begin{itemize}
            \item Personal vehicle (1, 8)
            \item Bus, streetcar, or trolley bus (2, 3)
            \item Subway, elevated, or railroad (4, 5)
            \item Taxicab (7)
            \item Bicycle (9)
            \item Walked (10)
            \item Worked at home (11)
            \item Other (including Ferry) (6, 12)
        \end{itemize}
    \end{itemize}

\end{enumerate}

\subsection{More plots}
\label{sec:appendix_plots}

In this section, we include plots and analysis for the remaining tasks that were not covered in the main body.

\subsubsection{ACS-Employment-Race}

In Fig.~\ref{fig:acs_employment_race} we observe that the even without fairness constraints, the downstream learner on the true sensitive features achieves an error disparity of only ${\sim}$0.003 in-sample. Since we try discrete values of gamma in intervals of 0.005, this means that the tradeoff curve on the true sensitive features contains only two points. The $\Hs$-proxy is incredibly successful on this task and achieves nearly identical performance to the true sensitive features. In fact, the least disparate model trained on the $\Hs$-proxy achieves a slightly lower in-sample error disparity than that trained on the true sensitive features \textit{measured with respect to the true sensitive features}. Since the models trained on the true labels should be able to optimize for this quantity exactly, these results indicate that any differences in the downstream models that result from being trained on the proxy rather than the true sensitive features are less significant than the unavoidable approximation error of the downstream learner. Thus, the $\Hs$-proxy's performance on this task is as close to perfect as we could hope for, since a perfect proxy exactly emulates the true sensitive features. On the other hand, the MSE proxy and baseline proxy both perform poorly--exhibiting tradeoff curves that increase error dramatically for negligible decreases in disparity. This further highlights the success of the $\Hs$-proxy on this task. Generalization on this dataset is quite good (note the scale of the y-axis): out-of-sample the models achieve similar errors and error disparities that are only $\sim~$0.004 higher than in-sample. 

\begin{figure*}[h]
\centering
\includegraphics[width=0.7\textwidth]{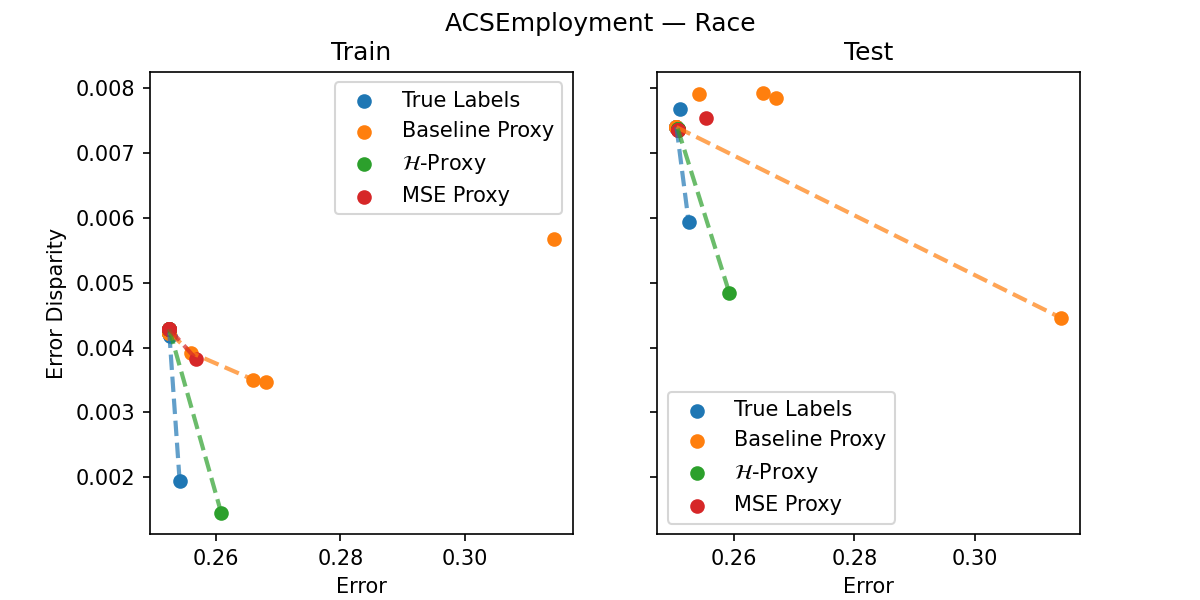}
\captionsetup[subfigure]{labelformat=empty}
\caption{Plots for ACSEmployment task with race as sensitive feature}
\label{fig:acs_employment_race}
\end{figure*}

\subsubsection{ACS-Employment-Age}

In Fig.~\ref{fig:acs_employment_age} we note that none of the three proxies \textit{nor} the true sensitive features admit models that yield sensible tradeoff curves on the downstream task. Poor performance on the true sensitive features indicates that the downstream learning algorithm may have failed to converge, rather than  indicating a failure of the proxy. In fact, the odd behavior of the models trained on the true labels is similar to the odd behavior of the downstream models trained on the proxy. Despite the less-than-optimal behavior, generalization is quite good: the train and test plots look nearly identical.

\begin{figure*}[h]
\centering
\includegraphics[width=0.7\textwidth]{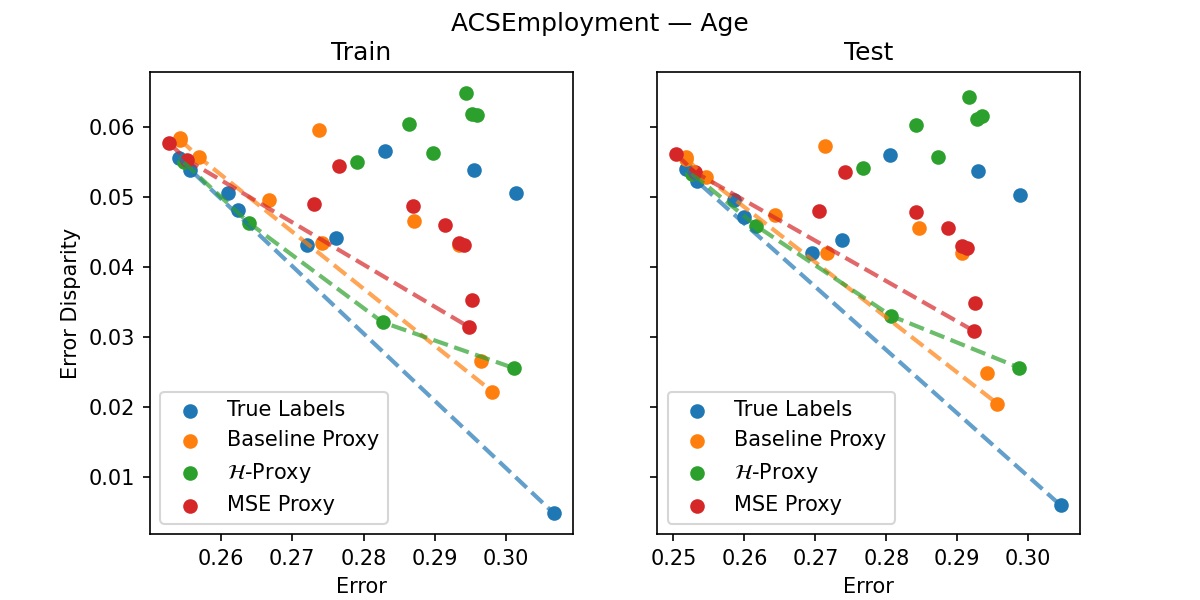}
\captionsetup[subfigure]{labelformat=empty}
\caption{Plots for ACSEmployment dataset with age as sensitive feature}
\label{fig:acs_employment_age}
\end{figure*}

\subsubsection{ACS-Employment-Sex}

In Fig.~\ref{fig:acs_employment_sex} we see that all three proxies fail to exhibit a sensible tradeoff curve on the downstream task, while the models trained on the true sensitive features do. This could be due to lack of convergence of the proxy algorithm, approximation error, or non-existence of a multi-accurate proxy in $\Gs$. Despite this, we note that the models all demonstrate remarkable generalization in terms of both error and error disparity, and the poor performance of the proxies was detected in-sample.

\begin{figure*}[h]
\centering
\includegraphics[width=0.7\textwidth]{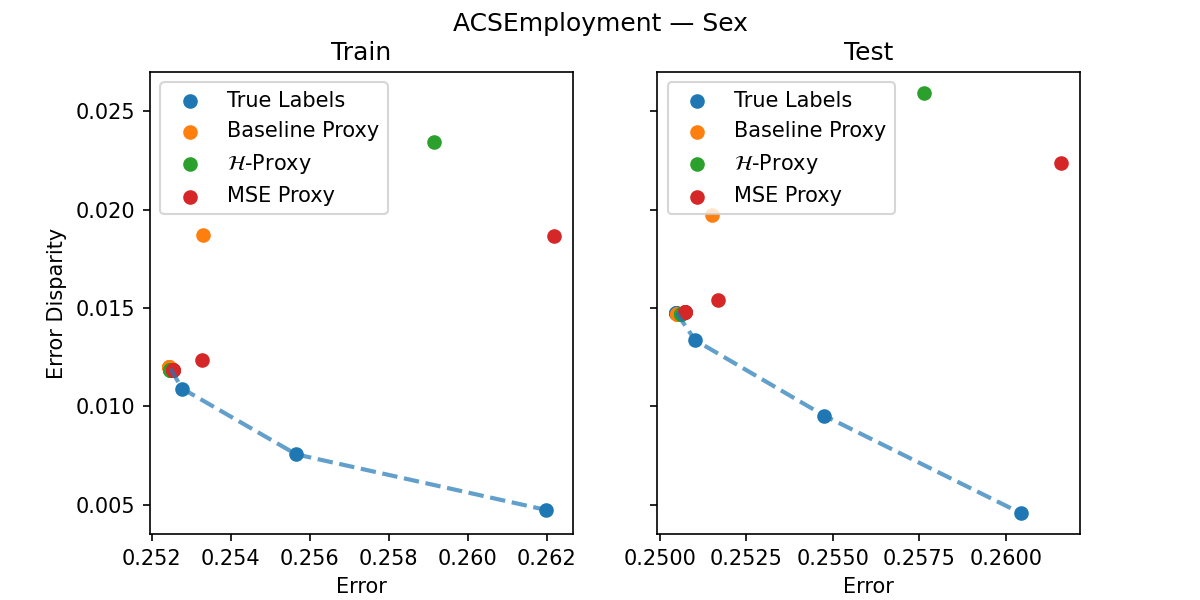}
\captionsetup[subfigure]{labelformat=empty}
\caption{Plots for ACSEmployment dataset with sex as sensitive feature}
\label{fig:acs_employment_sex}
\end{figure*}

\subsubsection{ACS-IncomePovertyRatio-Race}

In Fig.~\ref{fig:acs_incomepovertyratio_race} we observe that the MSE and $\Hs$-proxy achieve near optimal performance with tradeoff curves that are overlaid on the curve corresponding to the true sensitive features. Both of these proxies outperform the baseline, which results in models that have error up to 0.02 greater for the same levels of disparity. Generalization performance is quite good for all models. The shape and scale of the curves are similar in- and out-of-sample.

\begin{figure*}[h]
\centering
\includegraphics[width=0.7\textwidth]{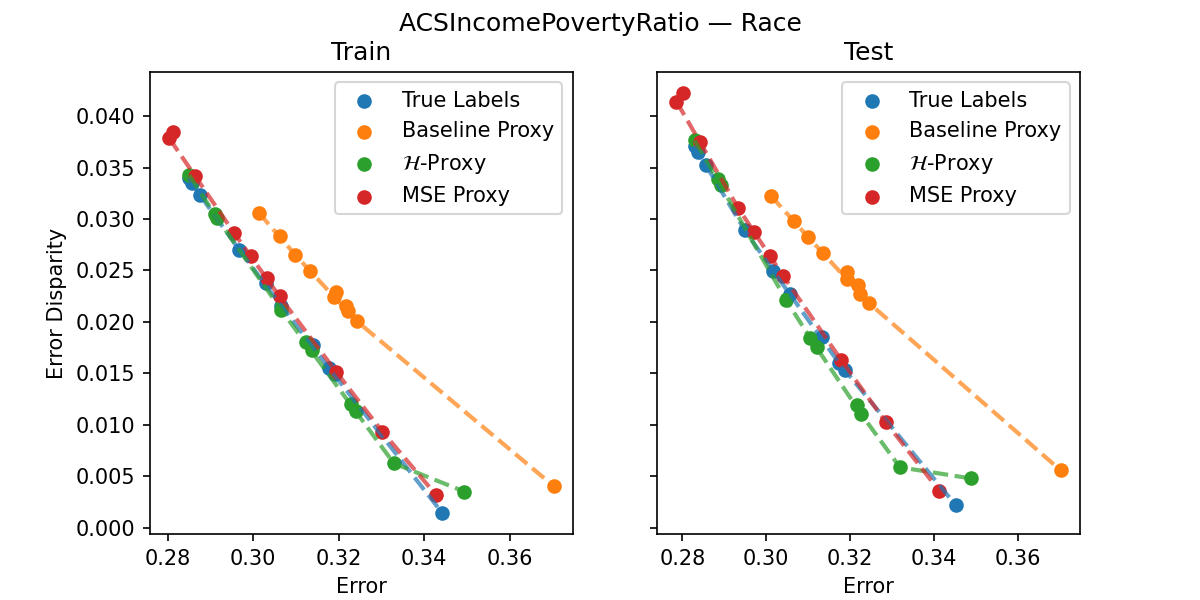}
\captionsetup[subfigure]{labelformat=empty}
\caption{Plots for ACSIncomePovertyRatio task with race as sensitive feature}
\label{fig:acs_incomepovertyratio_race}
\end{figure*}

\subsubsection{ACS-IncomePovertyRatio-Age}

In Fig.~\ref{fig:acs_incomepovertyratio_age} we observe near optimal performance of the $\Hs-$proxy, with a tradeoff curve that is overlaid on that of the true sensitive features. The MSE and baseline proxies achieve similar performance for the more relaxed constraints but are unable to achieve error disparity below 0.005 even though the sensitive features and $H-$proxy are both capable of inducing models that can achieve near 0 error disparity. Generalization performance on this dataset is excellent for all models trained. 

\begin{figure*}[h]
\centering
\includegraphics[width=0.7\textwidth]{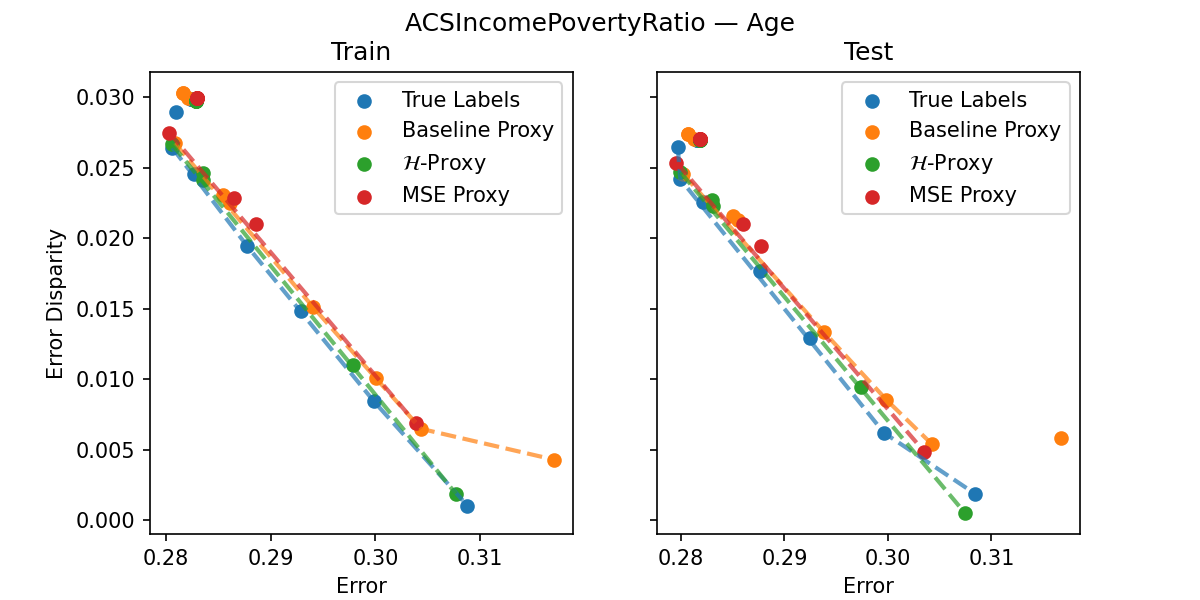}
\captionsetup[subfigure]{labelformat=empty}
\caption{Plots for ACSIncomePovertyRatio dataset with age as sensitive feature}
\label{fig:acs_incomepovertyratio_age}
\end{figure*}

\subsubsection{ACS-IncomePovertyRatio-Sex}

In Fig.~\ref{fig:acs_incomepovertyratio_sex} the $\Hs$-proxy induces models with slightly worse performance than the true sensitive features. The cost of using the pareto-optimal models induced by the $\Hs$-proxy rather than those of the true sensitive features is an error disparity less than 0.002. However, we note that in the worst case, one would accidentally use the non-pareto model resulting from the $\Hs$-proxy, which would result in an error ${\sim}$0.01 greater for the same level of error disparity. Given the small scale of the of these differences in absolute terms, it is possible that they can be explained by approximation error of the downstream learner rather than an explicit failure of the proxy algorithm.

\begin{figure*}[h]
\centering
\includegraphics[width=0.7\textwidth]{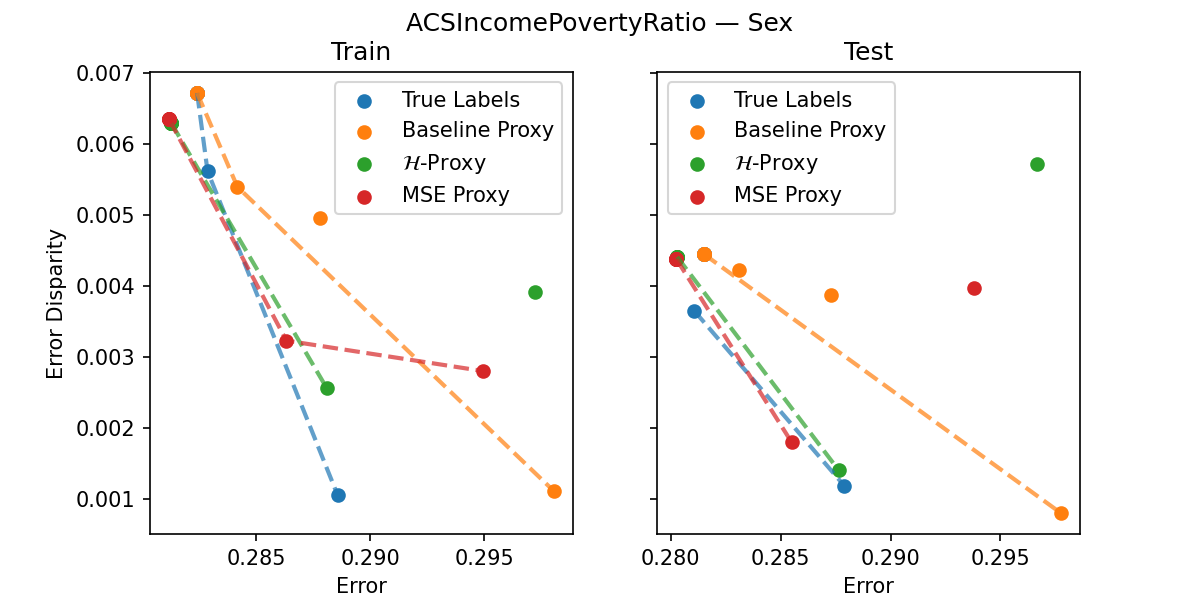}
\captionsetup[subfigure]{labelformat=empty}
\caption{Plots for ACSIncomePovertyRatio dataset with sex as sensitive feature}
\label{fig:acs_incomepovertyratio_sex}
\end{figure*}

\subsubsection{ACS-Mobility-Race}

In Fig.~\ref{fig:acs_mobility_race} we observe that the three proxies exhibit similar performance to each other and the true labels for the relaxed portion of the tradeoff curves, but none of them are able to achieve an error disparity lower than 0.008 while the least disparate model on the true features achieves an error disparity near 0. However, we also note that just looking at models trained on the true sensitive features, the second most disparate model achieves an error disparity greater than 0.11, even though it has a gamma value of 0.005. Therefore, there is some approximation error of the downstream learner that, independent of the effects of the proxy, can increase the error disparity by at least 0.006 compared to the intended constraint. This source of error explains most of the discrepancy between the results on the proxies and those on the true labels. Therefore, these results are not a strong indication that the proxy ``failed" as a substitute for the true sensitive features, though it may have been imperfect.

\begin{figure*}[h]
\centering
\includegraphics[width=0.7\textwidth]{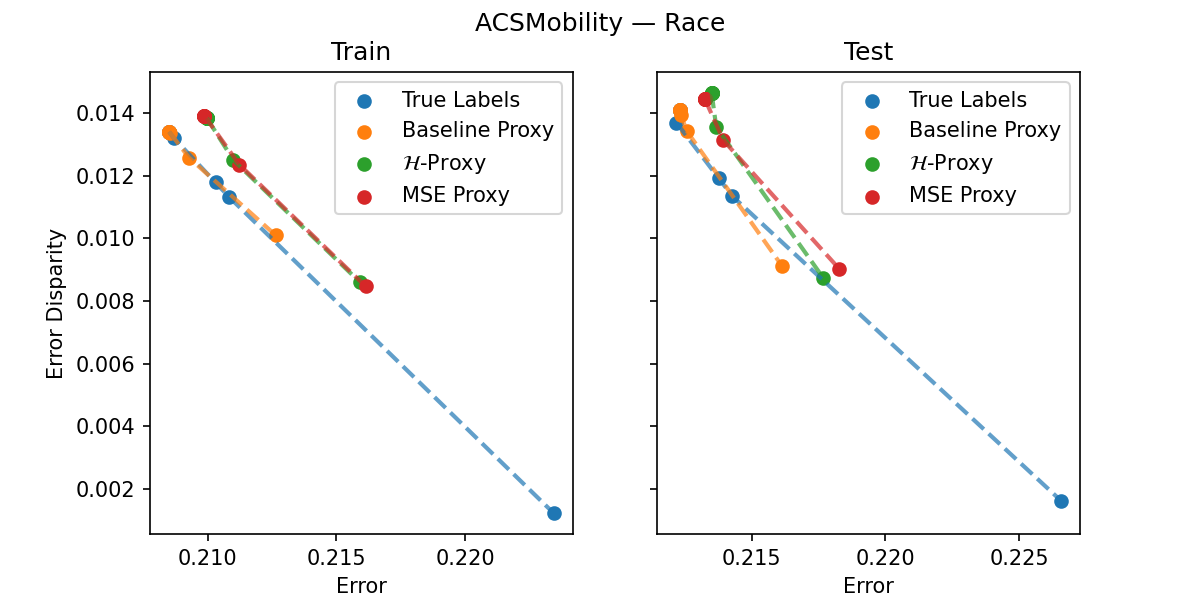}
\captionsetup[subfigure]{labelformat=empty}
\caption{Plots for ACSMobility task with race as sensitive feature}
\label{fig:acs_mobility_race}
\end{figure*}

\subsubsection{ACS-Mobility-Sex}

In Fig.~\ref{fig:acs_mobility_sex} we investigate the performance on the ACSMobility task, using sex as the sensitive feature. At first glance, these results appear to indicate a failure of the proxies. None of the models demonstrate clear tradeoffs between error and error disparity. However, looking at the models trained on the true sensitive features, we notice that there is exactly one point on the tradeoff curve, with error ${\sim}$0.209 and error disparity ${\sim}$0.0005. This means that the population error minimizing model, by luck, achieves an error disparity that is essentially 0, meaning there was not much room for improvement for the original model, let alone the proxies. We also note that the apparent poor-performance of the proxies is visually amplified by the unusually small scale of the plots. The least disparate model for the $\Hs$-proxy achieves an error disparity ${\sim}$0.0015 and error ${\sim}$0.213. Compared to the model trained on the true features, the error disparity is greater by only 0.001 and the error by only 0.004. From our previous experiments, we recall that these differences are well within the reasonable approximation error of the downstream learner alone, and thus do not indicate that the proxy failed as a substitute for the true sensitive features on this task. 

\begin{figure*}[h]
\centering
\includegraphics[width=0.7\textwidth]{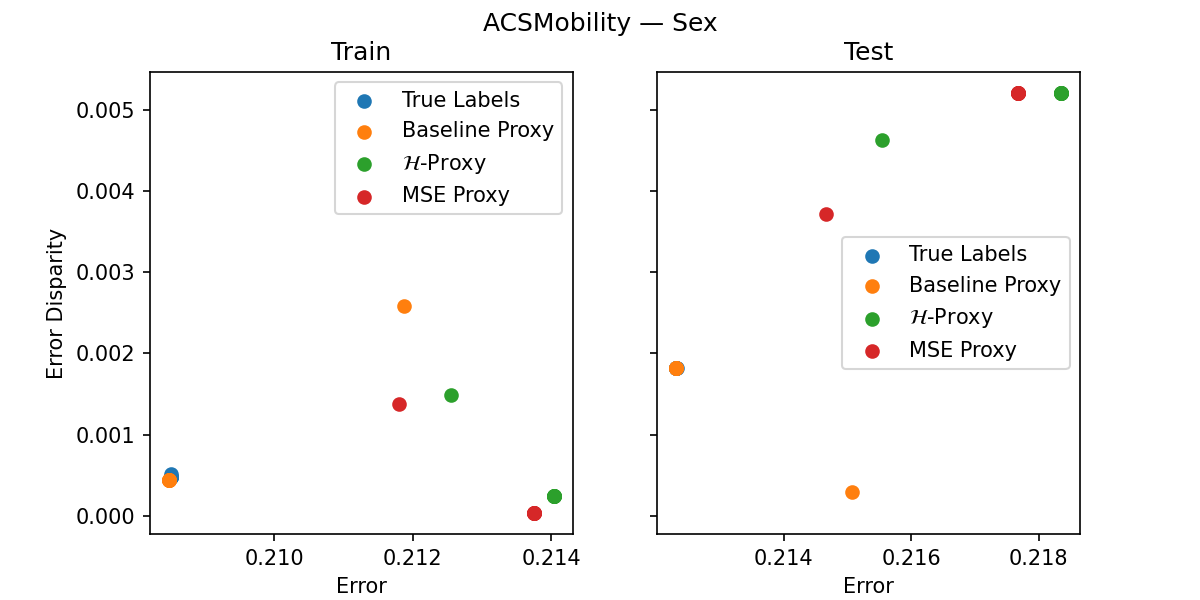}
\captionsetup[subfigure]{labelformat=empty}
\caption{Plots for ACSMobility dataset with sex as sensitive feature}
\label{fig:acs_mobility_sex}
\end{figure*}

\subsubsection{ACS-PublicCoverage-Race}

In Fig.~\ref{fig:acs_publiccoverage_race} we note more evidence of success for the $\Hs$-proxy. The tradeoff curve it induces is nearly identical to that induced by the true sensitive features, although with slightly greater error disparity. The $\Hs$-proxy also outperforms the baseline and MSE proxies, which achieve similar performance for the relaxed portion of the curve, but are unable to achieve error disparity lower than 0.002 whereas the $\Hs$-proxy achieves a minimum error disparity $<0.001$. The baseline proxy also appears to induce models with slightly more error for the same levels of unfairness compared to the other three curves. All models exhibit excellent generalization performance and out-of-sample behavior is nearly identical to in-sample. 
\begin{figure*}[h]
\centering
\includegraphics[width=0.7\textwidth]{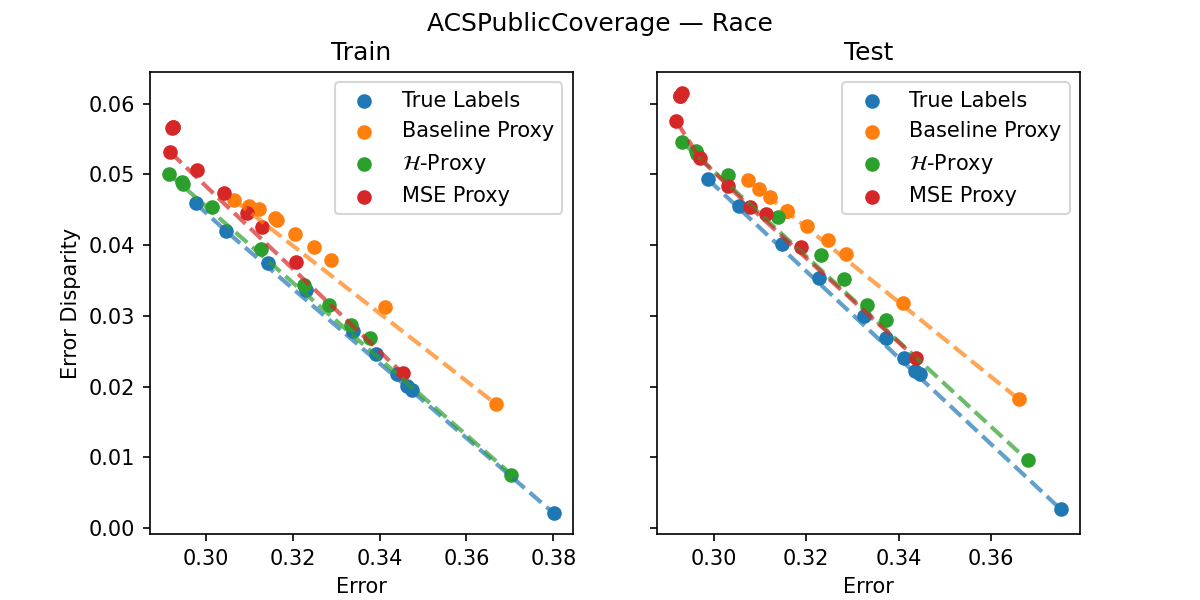}
\captionsetup[subfigure]{labelformat=empty}
\caption{Plots for ACSPublicCoverage task with race as sensitive feature}
\label{fig:acs_publiccoverage_race}
\end{figure*}

\subsubsection{ACS-PublicCoverage-Age}

In Fig.~\ref{fig:acs_publiccoverage_age} we observe that the unconstrained downstream learner achieves an error disparity that is essentially 0. Because of this, the tradeoff curve on the true sensitive features consists of a single point. The proxies induce models that don't indicate a clear tradeoff, but, when considering scale, do not indicate a significant failure. All proxies have tradeoff curves consisting of just two points, meaning that the estimated error disparity of the unweighted model (i.e. the error disparity with respect to the proxies) was at most 0.005. This difference is quite small in absolute terms and--while indicating that the proxy was not a perfect substitute for the true sensitive features--likely does not constitute a ``failure'' of the proxy. Out-of-sample, we note that all proxies actually end up with downstream models that achieve near 0 error disparity and slightly less error than that of the true labels, but this may be the result of random noise.

\begin{figure*}[h]
\centering
\includegraphics[width=0.7\textwidth]{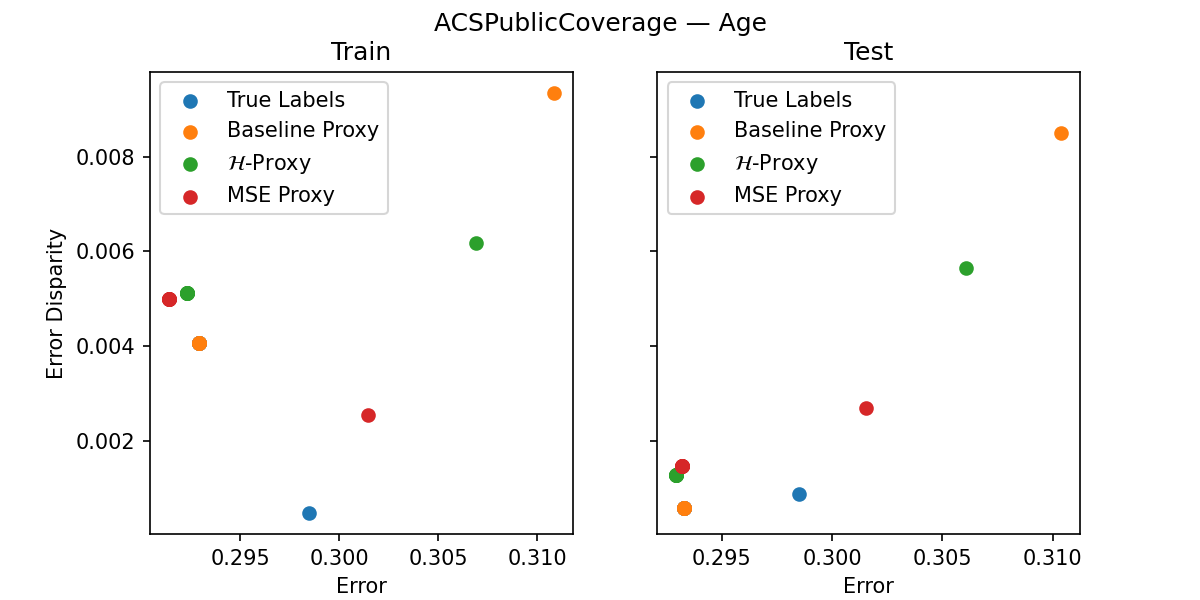}
\captionsetup[subfigure]{labelformat=empty}
\caption{Plots for ACSPublicCoverage dataset with age as sensitive feature}
\label{fig:acs_publiccoverage_age}
\end{figure*}

\subsubsection{ACS-PublicCoverage-Sex}

In Fig.~\ref{fig:acs_publiccoverage_sex}  we observe that the $\Hs$-proxy induces models that are unable to achieve much better performance than the unweighted models, achieving a minimum error disparity of only 0.006 while the true sensitive features are able to induce models with error disparity as low as 0.001, albeit with slightly higher error. Unlike the MSE and baseline proxies, the $H$-proxy does not seem to admit models that increase both error and disparity, which is a desirable property. This may indicate that the multi-accuracy constraints succeeded in reigning in the proxy such that the downstream learner recognized that it could not make improvements despite being unable to meet the strictest constraints, but it is also possible that this apparent ``good" behavior was simply by luck.

\begin{figure*}[h]
\centering
\includegraphics[width=0.7\textwidth]{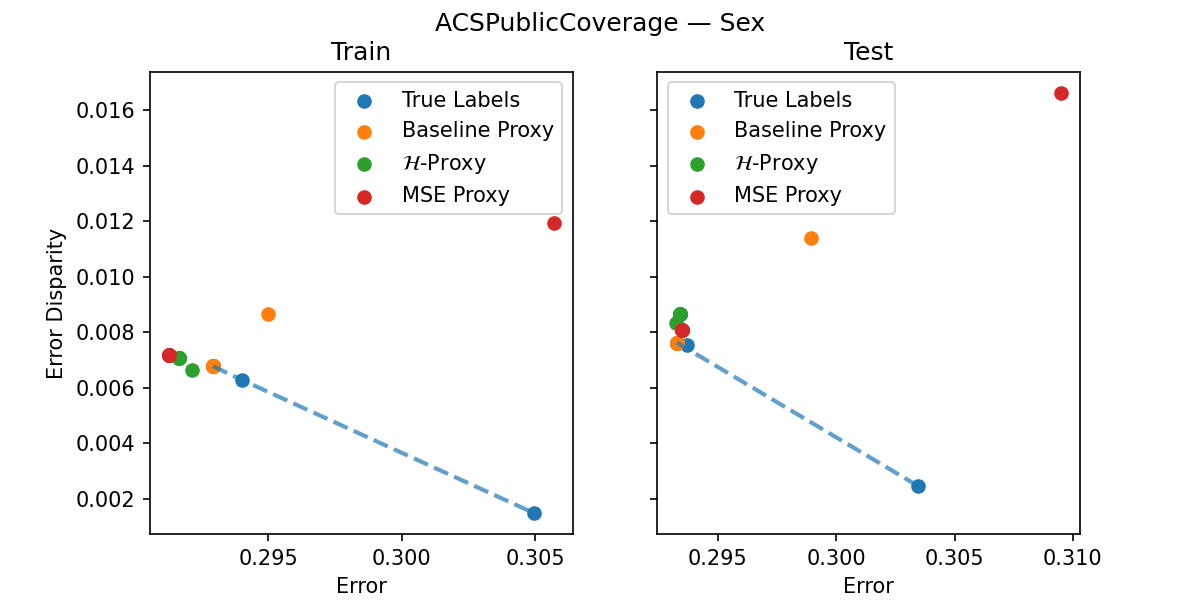}
\captionsetup[subfigure]{labelformat=empty}
\caption{Plots for ACSPublicCoverage dataset with sex as sensitive feature}
\label{fig:acs_publiccoverage_sex}
\end{figure*}

\subsubsection{ACS-TravelTime-Race}

In Fig.~\ref{fig:acs_traveltime_race} we observe that  both the $\Hs$-proxy (as well as the MSE proxy) seem to have at least partially failed on this task, perhaps due to non-convergence or a lack of a suitable multi-accurate proxy in $\Gs$. There are a sufficient number of points for both proxies that illustrate no clear tradeoff and result in simultaneously increasing error disparity and population error. The baseline proxy admits a more sensible looking tradeoff, although the minimum disparity it can reach is only 0.001 and, for the same levels of disparity as the true labels, admits models that are up to 0.015 less accurate. Fortunately, we note that all un-intended behavior could be detected in-sample and, in fact, that the proxy's performance seemed to \textit{improve} out-of-sample. Despite this improvement, we would recommend discarding proxies that exhibit unexpected behavior in-sample.

\begin{figure*}[h]
\centering
\includegraphics[width=0.7\textwidth]{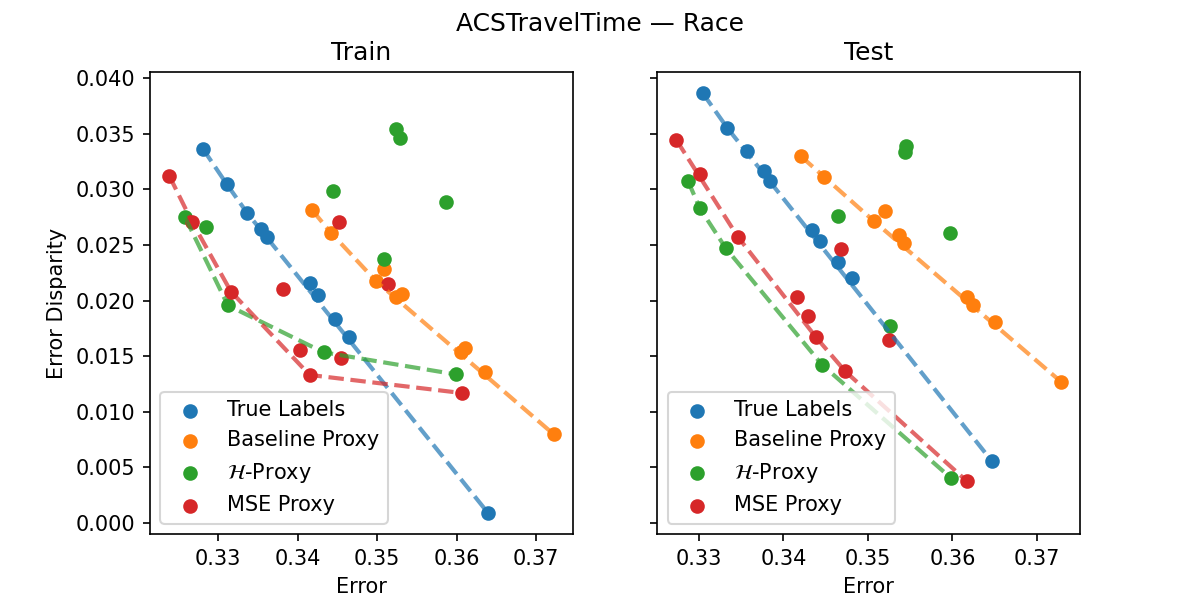}
\captionsetup[subfigure]{labelformat=empty}
\caption{Plots for ACSTravelTime task with race as sensitive feature}
\label{fig:acs_traveltime_race}
\end{figure*}

\subsubsection{ACS-TravelTime-Age}

In Fig.~\ref{fig:acs_traveltime_age} we note that all proxies seem to achieve similar performance to the true labels, although that it appears the downstream learner may have failed to converge in all cases. The evidence for this is that the smallest error disparity for models trained on the true labels is 0.005, rather than near 0. There is some non-convex behavior of the tradeoff curves induced by the proxies (and the true features), but all models are within 0.005 error disparity of forming a sensible tradeoff. Results are generally consistent in- and out-of-sample.

\begin{figure*}[h]
\centering
\includegraphics[width=0.7\textwidth]{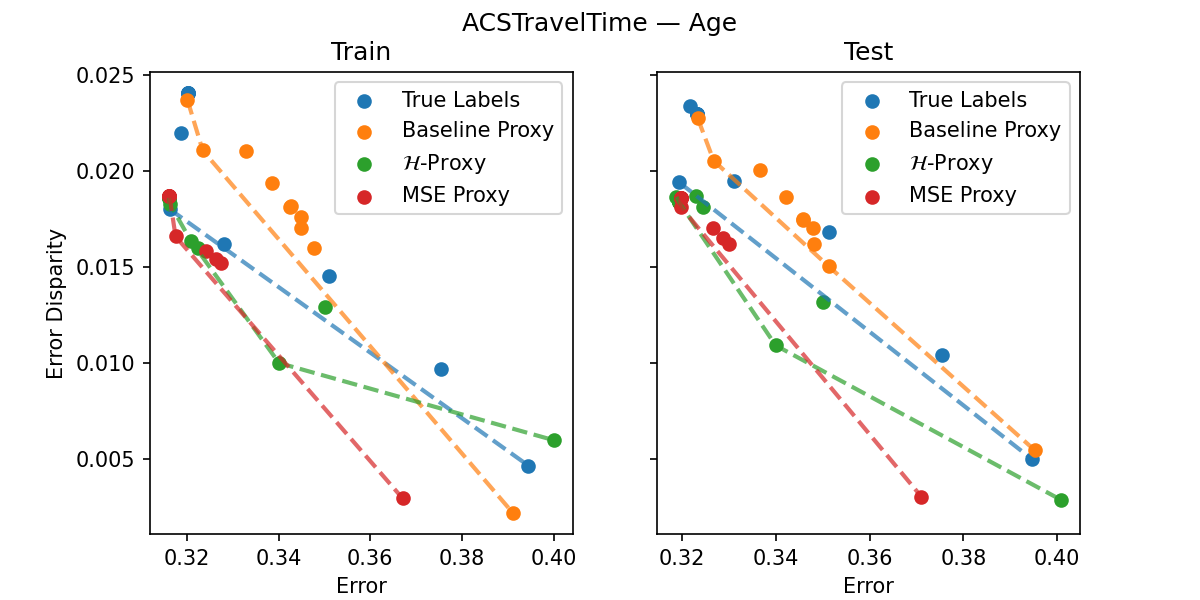}
\captionsetup[subfigure]{labelformat=empty}
\caption{Plots for ACSTravelTime dataset with age as sensitive feature}
\label{fig:acs_traveltime_age}
\end{figure*}

\subsubsection{ACS-TravelTime-Sex}

In Fig.~\ref{fig:acs_traveltime_sex} on this task, all three proxies failed to induce downstream models capable of achieving optimal error disparity. In particular, the baseline, $\Hs-$ and MSE proxies achieved minimum error disparity, 0.018, 0.016, and 0.012, respectively, while the true labels were able to induce models with disparity near 0. Failure of the $\Hs$-proxy may indicate that there did not exist a suitable proxy $g \in \Gs$ that satisified multi-accuracy constraints.

\begin{figure*}[h]
\centering
\includegraphics[width=0.7\textwidth]{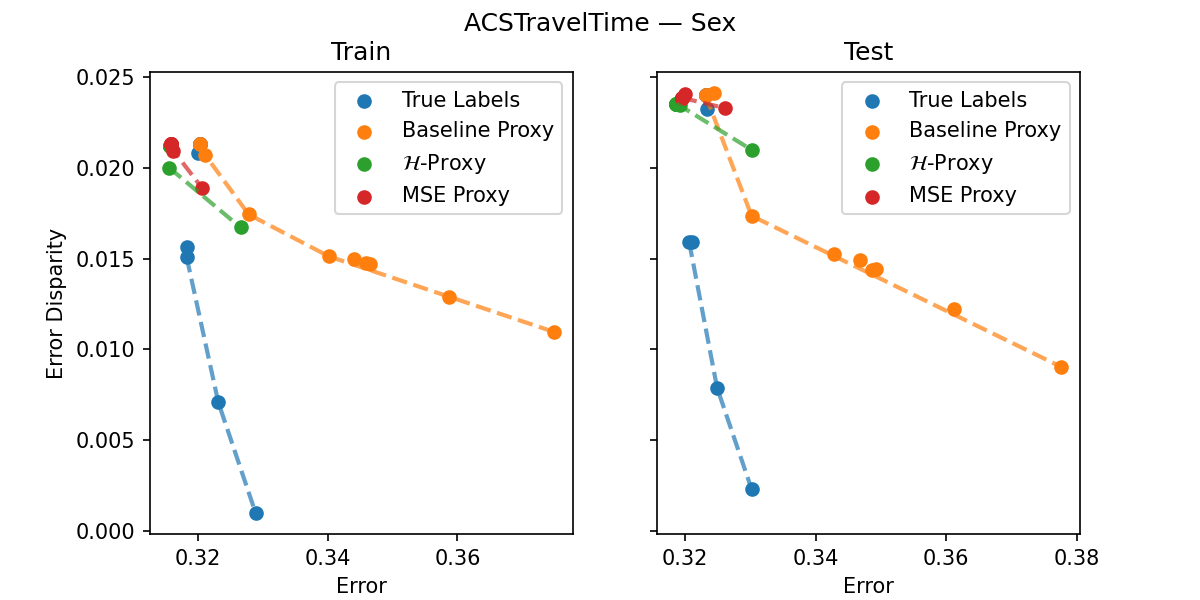}
\captionsetup[subfigure]{labelformat=empty}
\caption{Plots for ACSTravelTime dataset with sex as sensitive feature}
\label{fig:acs_traveltime_sex}
\end{figure*}

\subsection{Detecting proxy failures} 
\label{sec:addressing_failures}

Despite occasional failures, our experiments indicate that the $\Hs$-proxy can safely be used in practice on the condition that it is tested before deployment. In particular, we advise that anyone using our algorithm to train a proxy do the following: After training the proxy, train two downstream models on the sample. One using the proxy in place of the sensitive attributes, and the other using the sensitive attributes directly. If the performance of these two models is similar, our results indicate that the proxy can be deployed for use on out-of-sample instances on the same distribution and maintain the expected fairness guarantees on the relevant fairness task. Otherwise, the proxy should be discarded.

\end{document}